\documentclass[11pt,reqno]{amsart}
\usepackage{graphicx,enumerate,amssymb,bbold}
\usepackage[notrig]{physics}
\usepackage[labelfont=small]{subcaption}
\usepackage[ruled,vlined]{algorithm2e}
\usepackage[left=3.0cm,right=3.0cm,top=2.5cm,bottom=2.5cm,includeheadfoot]{geometry}

\usepackage[colorlinks=true, pdfstartview=FitV, linkcolor=blue,
            citecolor=blue, urlcolor=blue]{hyperref}

\usepackage[capitalise,noabbrev,nameinlink]{cleveref}

\usepackage{multirow}

\numberwithin{equation}{section}
\numberwithin{figure}{section}
\usepackage{booktabs}
\usepackage{mathtools}
\usepackage[hang,flushmargin]{footmisc} 
\usepackage{tabularx}
\usepackage[table]{xcolor}
\usepackage{tikz} 
\usepackage{pgfplots}
\pgfplotsset{compat=newest} 
\usetikzlibrary{calc, chains, fit, positioning, shapes, patterns}
\usepackage{gensymb}
\usepackage{adjustbox}
\usepackage{pifont}

\theoremstyle{plain}
\newtheorem{theorem}{Theorem}[section]
\newtheorem{lemma}[theorem]{Lemma}
\newtheorem{proposition}[theorem]{Proposition}
\newtheorem{corollary}[theorem]{Corollary}

\theoremstyle{definition}
\newtheorem{definition}{Definition}[section]

\newtheorem{remark}{Remark}[section]

\newcommand{\bitem}{\begin{itemize}}
\newcommand{\eitem}{\end{itemize}}
\newcommand{\mc}[1]{\mathcal{#1}}
\newcommand{\mb}[1]{\mathbb{#1}}

\newcommand{\N}{\mathbb{N}}
\newcommand{\R}{\mathbb{R}}

\newcommand{\EE}{\mathbb{E}}

\newcommand{\bpm}{\begin{pmatrix}}
\newcommand{\epm}{\end{pmatrix}}
\newcommand{\bsm}{\left(\begin{smallmatrix}}
\newcommand{\esm}{\end{smallmatrix}\right)}
\newcommand{\T}{\top}

\newcommand{\la}{\langle}
\newcommand{\ra}{\rangle}

\newcommand{\mrm}[1]{\mathrm{#1}}

\newcommand{\veps}{\varepsilon}

\newcommand{\dotcup}{\mathbin{\dot{\cup}}}

\newcommand{\eins}{\mathbb{1}}

\DeclareMathOperator{\Diag}{Diag}
\DeclareMathOperator{\diag}{diag}

\DeclareMathOperator{\ggrad}{grad}

\DeclareMathOperator{\Exp}{Exp}

\newcommand{\cmark}{\ding{51}}%
\newcommand{\xmark}{\ding{55}}%

\definecolor{colorM}{RGB}{255,140,0}
\definecolor{colorL}{RGB}{0,140,200}
\definecolor{colorR}{RGB}{140,0,200}

\newcommand{\PlusMarker}{\protect\tikz{\protect\draw[line width=0.3ex, x=1ex, y=1ex] (0.5,0) -- (0.5,1)(0,0.5) -- (1,0.5);}}

\newcommand{\SelfAssignmentGraph}{
\begin{tikzpicture}[x=1mm, y=1mm,
    node distance = 4 and 21,
      start chain = going below,
         V/.style = {circle, draw,
                     fill=##1,
                     inner sep=0, minimum size=10,
                     node contents={}, thick},
                    ]
\foreach \i in {1,...,5}
{
	\node (n1\i) [V=colorL, on chain, label={[text=colorL]left:$f_{\i}'$}]  {};
	\ifnum\i>1
		\ifnum\i<5
			\node (n2\i) [V=colorM, right= 40 of n1\i] {};
		\fi
	\fi
	\node (n3\i) [V=colorR, right= 80 of n1\i, label={[text=colorR]right:$f_{\i}$}]  {};
}
\node (label) [colorL,fit=(n15) (n11), label=above:\textbf{Labeling}] {};
\node (data) [colorR,fit=(n35) (n31), label=above:\textbf{Data}] {};
\node [anchor=south, above=1.5 of n22] {\textbf{Prototypes}};

\foreach \i in {1,...,5}
{
	\foreach \j in {1,...,5}
	{
		\ifnum\j>1
			\ifnum\j<5
				\draw[-, shorten >=1mm, shorten <= 1mm,thick] (n1\i) edge (n2\j);
			\fi
		\fi
		\ifnum\i>1
			\ifnum\i<5
				\draw[-, shorten >=1mm, shorten <= 1mm,thick] (n2\i) edge (n3\j);
			\fi
		\fi
	}
}
\node[anchor=center] at ($(n15)!0.25!(n34)+(0,-2.5)$) {\small { $W$ } };
\node[anchor=center] at ($(n14)!0.75!(n35)+(0,-2.5)$) {\small { $C(W)^{-1}W^\top$} };
\end{tikzpicture}
}

\definecolor{myblue}{RGB}{0,113,188}

\definecolor{mycolorA}{RGB}{113,188,0}
\definecolor{mycolorB}{RGB}{188,0,113}
\definecolor{mycolorC}{RGB}{188,113,0}

\pgfdeclarepatternformonly{my grid}{\pgfqpoint{-1pt}{-1pt}}{\pgfqpoint{2pt}{2pt}}{\pgfqpoint{2pt}{2pt}}%
{
  \pgfsetlinewidth{0.2pt}
  \pgfpathmoveto{\pgfqpoint{0pt}{0pt}}
  \pgfpathlineto{\pgfqpoint{0pt}{2pt}}
  \pgfpathmoveto{\pgfqpoint{0pt}{0pt}}
  \pgfpathlineto{\pgfqpoint{2pt}{0pt}}
  \pgfusepath{stroke}
}

\definecolor{colorHighlight}{RGB}{80,80,80}

\newcommand{\showExpPatchTransferTrain}[2]{
	\begin{centering}
		\begin{tabular}{c@{\hskip 0.6em}c@{\hskip 0.4em}c@{\hskip 0.4em}c@{\hskip 0.4em}c}

			& \multicolumn{4}{c@{\hskip 0.6em}}{\footnotesize \textbf{Locally Invariant Patch Dictionary Learning using the SAF } \boldmath $(s=0)$} \\
			\cmidrule[0.5pt](l{1.2em}r{1.65em}){2-5}
			\addlinespace
			& \cellcolor{colorHighlight} & \cellcolor{colorHighlight} & \cellcolor{colorHighlight} & \cellcolor{colorHighlight} \\[-0.8em]

			\rotatebox{90}{\, \quad \footnotesize \textbf{images}}
			&
			\cellcolor{colorHighlight}
			\includegraphics[width=#1\textwidth]{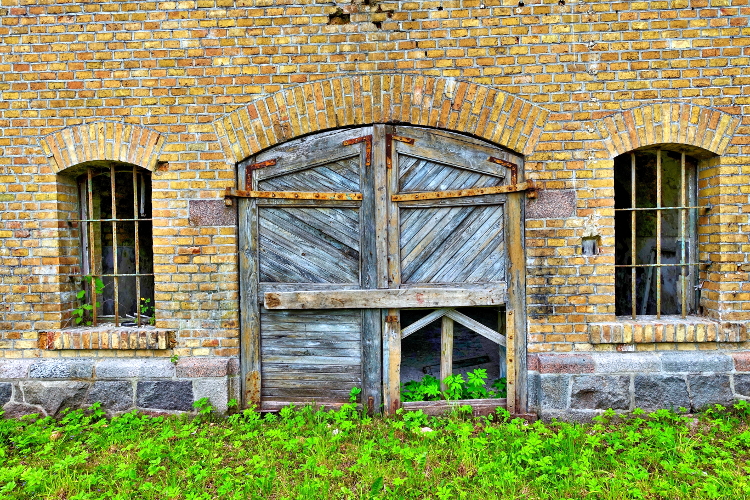}
			&
			\cellcolor{colorHighlight}
			\includegraphics[width=#1\textwidth]{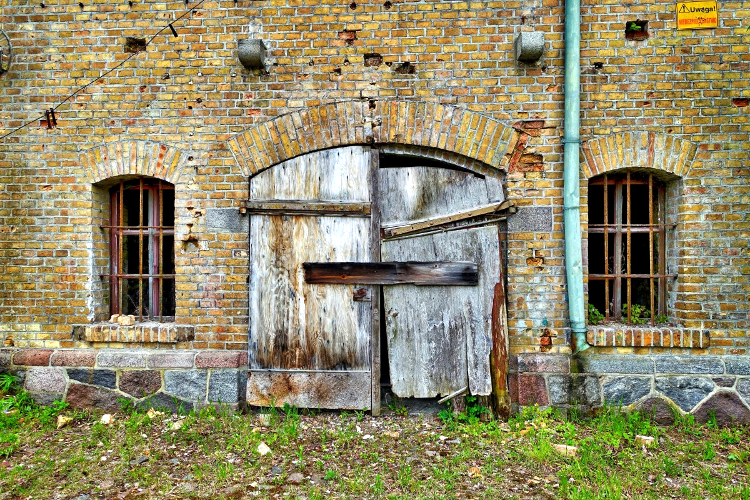}
			&
			\cellcolor{colorHighlight}
			\includegraphics[width=#1\textwidth]{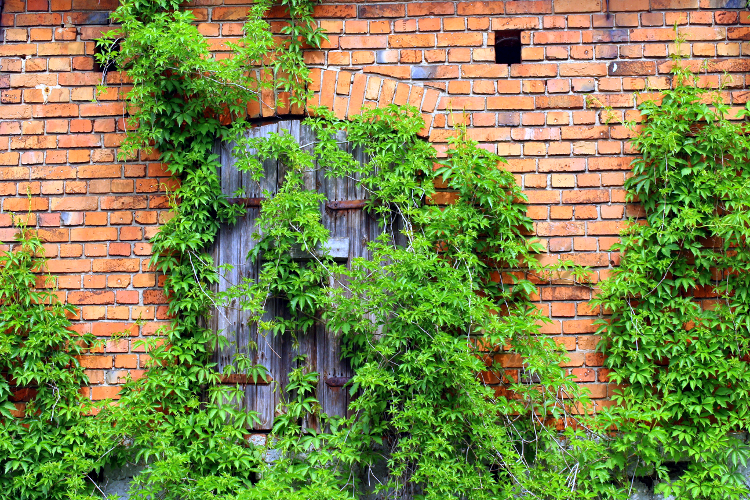}
			&
			\cellcolor{colorHighlight}
			\includegraphics[width=#1\textwidth]{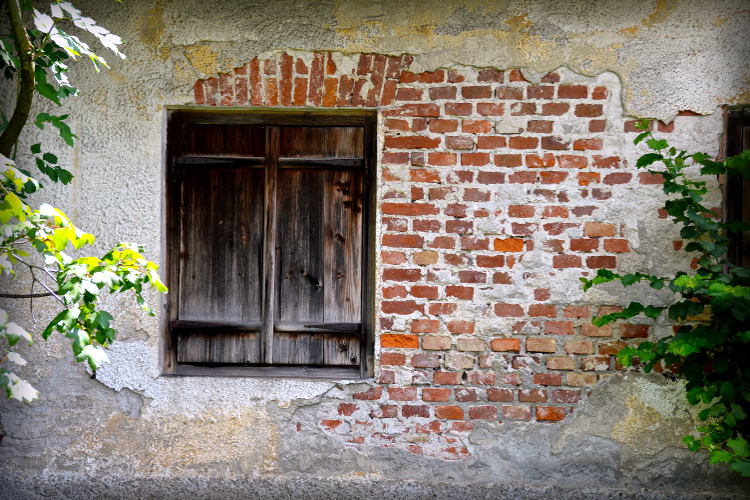}
			\\[-0.8em]
			& & & & \\[0.6em]

			\rotatebox{90}{\ \quad \footnotesize \textbf{partition}}
			&
			\includegraphics[width=#1\textwidth]{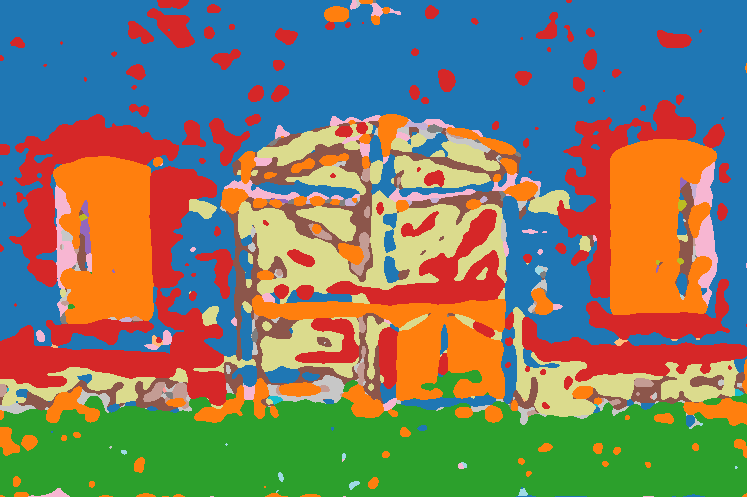}
			&
			\includegraphics[width=#1\textwidth]{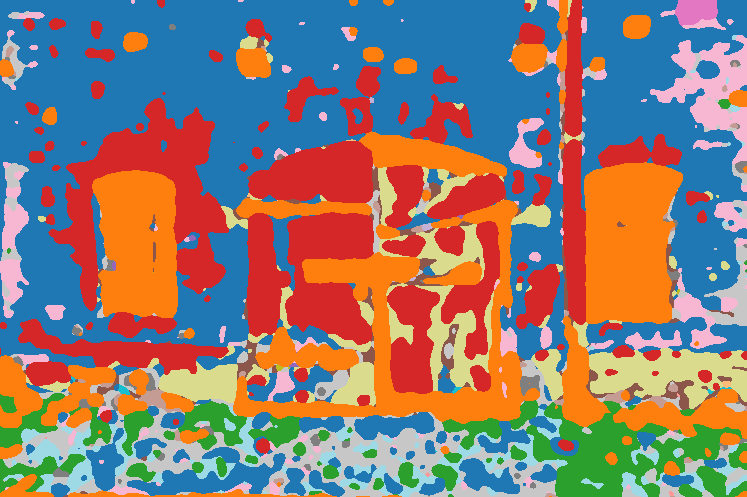}
			&
			\includegraphics[width=#1\textwidth]{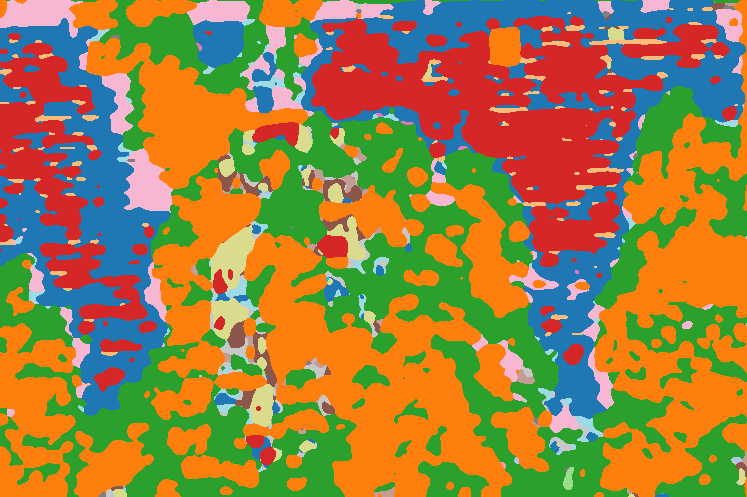}
			&
			\includegraphics[width=#1\textwidth]{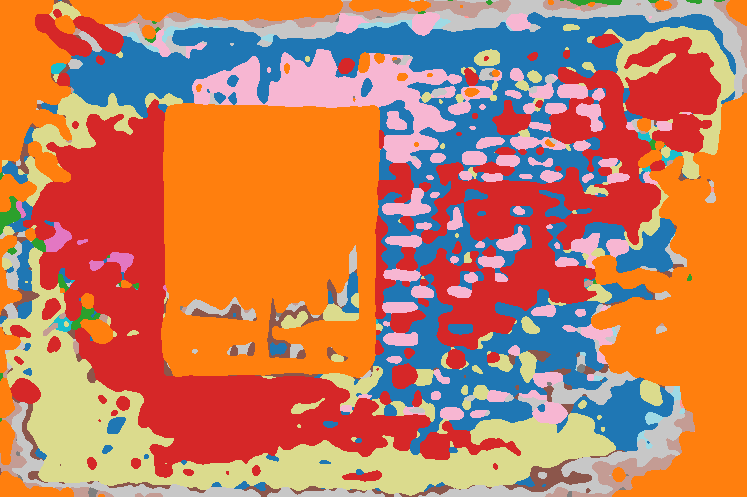}
			\\

			\rotatebox{90}{\quad   \footnotesize \textbf{assignment}}
			&
			\includegraphics[width=#1\textwidth]{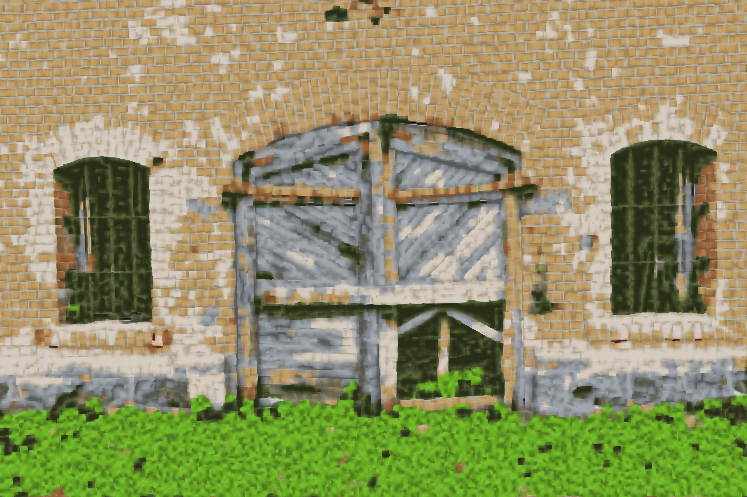}
			&
			\includegraphics[width=#1\textwidth]{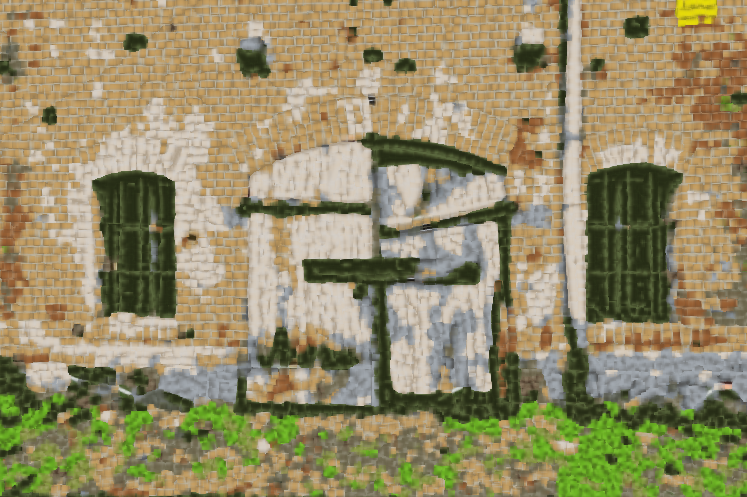}
			&
			\includegraphics[width=#1\textwidth]{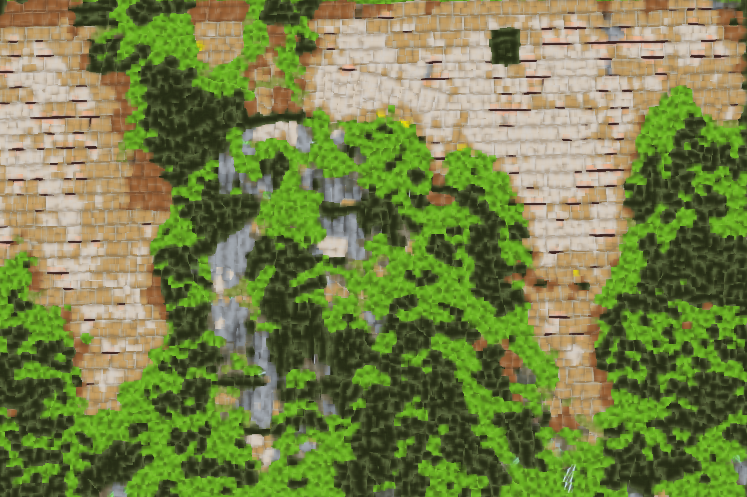}
			&
			\includegraphics[width=#1\textwidth]{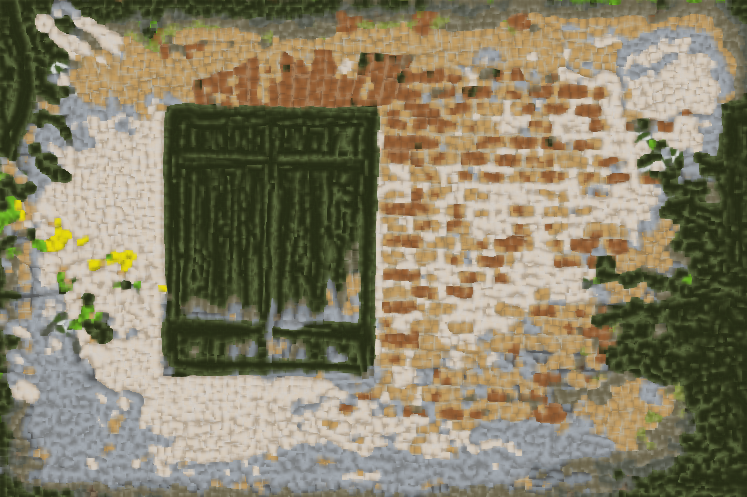}
			\\[0.6em]

			\rotatebox{90}{\quad \footnotesize \textbf{patches} \boldmath $\, \mc{F}_{\ast}$}
			&
			\multicolumn{2}{c@{\hskip 0.6em}}
			{
			  \includegraphics[width=#1\textwidth]{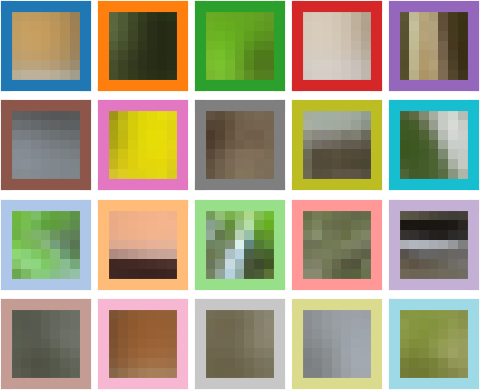}
			}
			&
			\multicolumn{2}{c@{\hskip 0.4em}}
			{
			  \includegraphics[width=#2\textwidth]{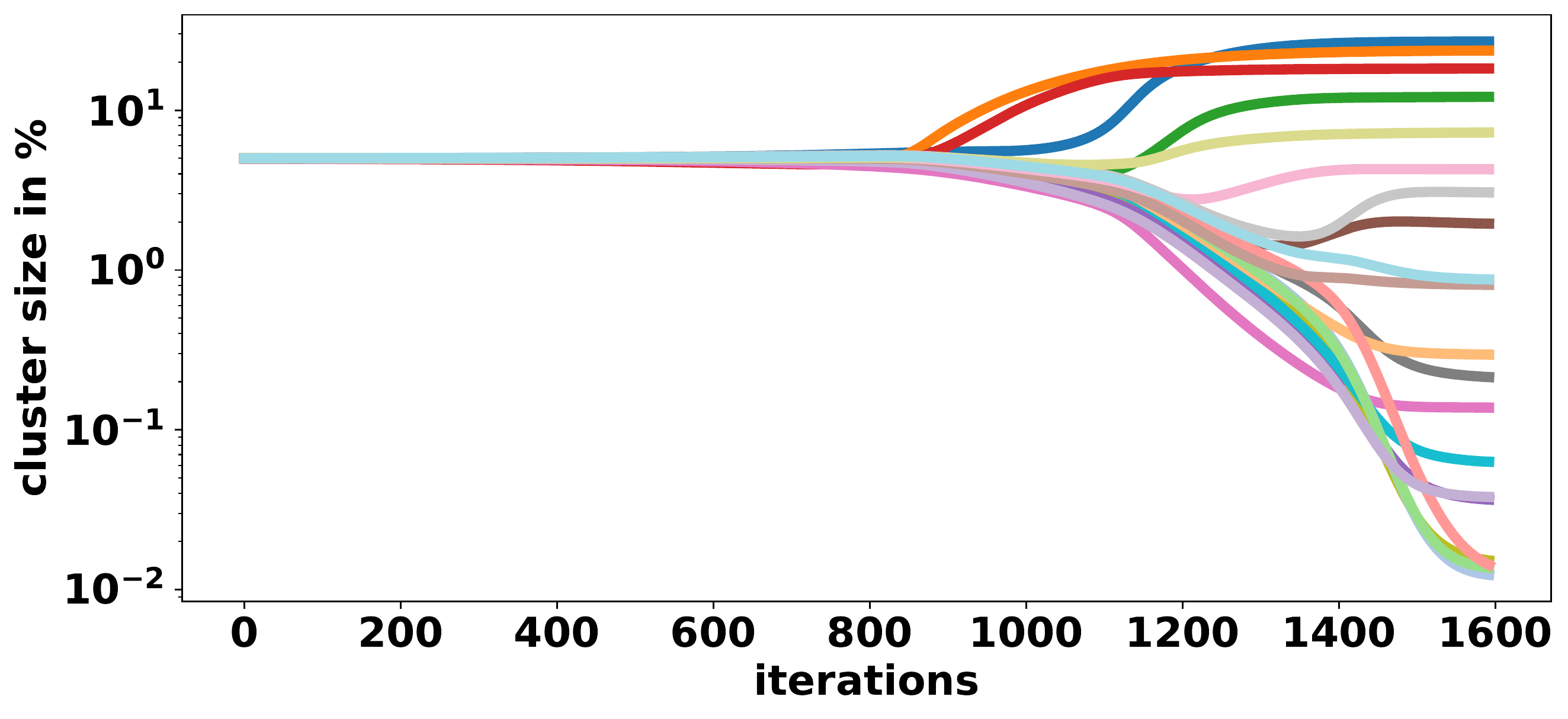}
			}

		\end{tabular}
	\end{centering}
}

\newcommand{\showExpPatchTransferEval}[2]{
	\begin{centering}
		\begin{tabular}{c@{\hskip 0.6em}c@{\hskip 0.4em}c@{\hskip 0.4em}c@{\hskip 0.4em}c}

			& \multicolumn{4}{c@{\hskip 0.6em}}{\footnotesize \textbf{Patch Dictionary Evaluation using the supervised AF} \boldmath } \\
			\cmidrule[0.5pt](l{1.2em}r{1.65em}){2-5}
			& & & & \\[-0.8em]

			& \footnotesize \textbf{(a}) & \footnotesize \textbf{(b}) & \footnotesize \textbf{(c}) & \footnotesize \textbf{(d}) \\[0.4em]

			\rotatebox{90}{\, \quad \footnotesize \textbf{images}}
			&
			\includegraphics[width=#1\textwidth]{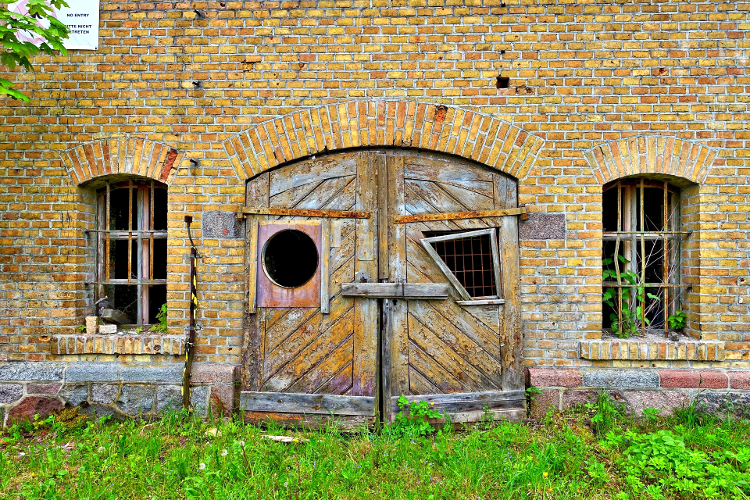}
			&
			\includegraphics[width=#1\textwidth]{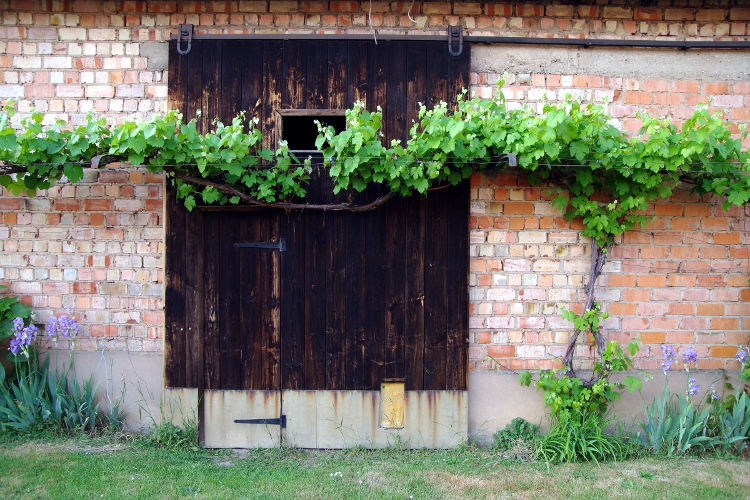}
			&
			\includegraphics[width=#1\textwidth]{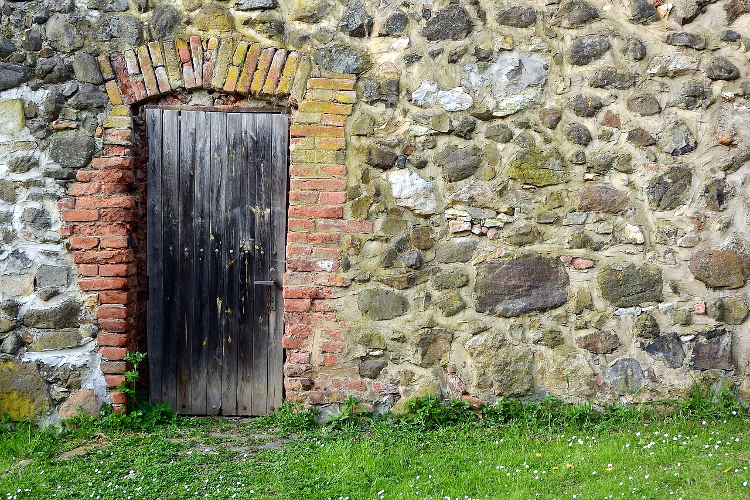}
			&
			\includegraphics[width=#1\textwidth]{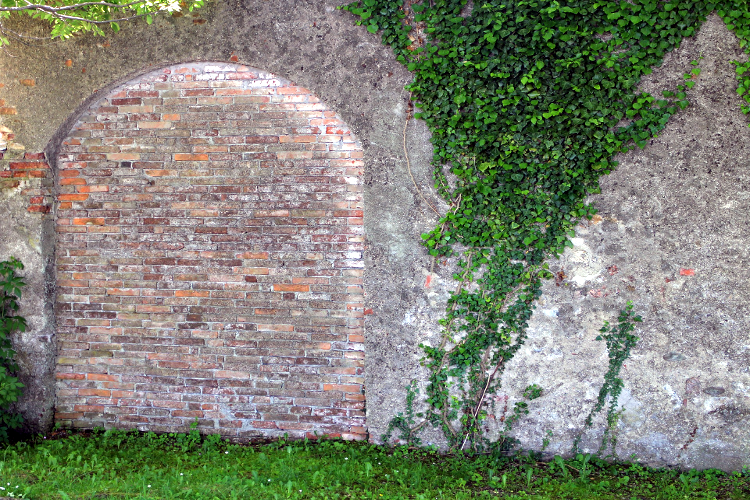}
			\\

			\rotatebox{90}{\ \quad \footnotesize \textbf{partition}}
			&
			\includegraphics[width=#1\textwidth]{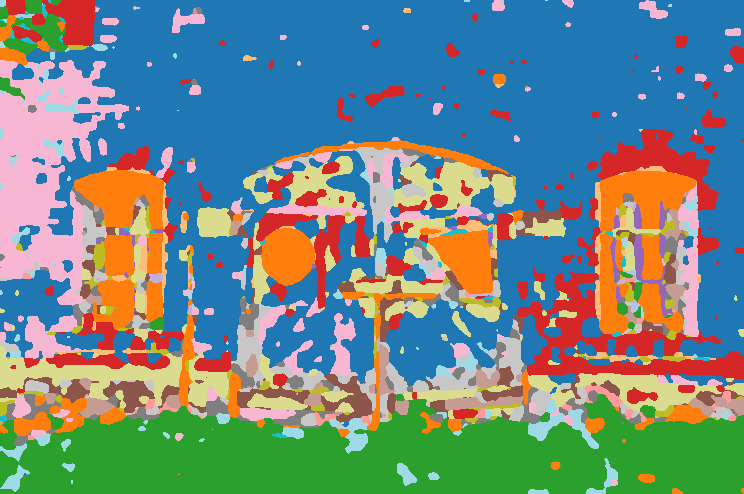}
			&
			\includegraphics[width=#1\textwidth]{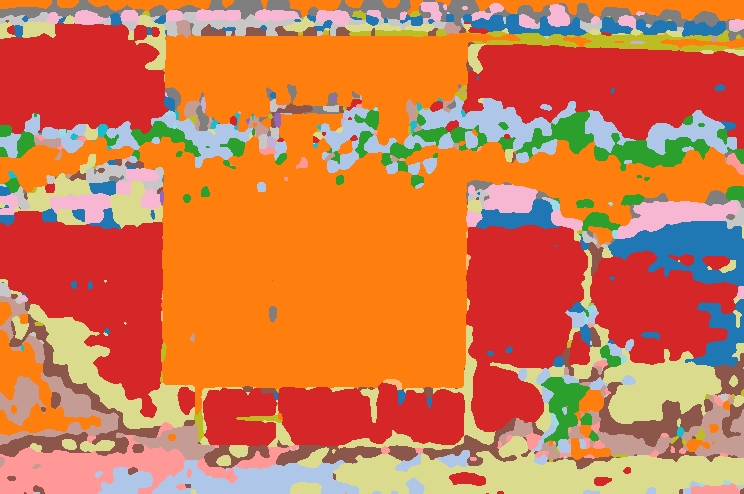}
			&
			\includegraphics[width=#1\textwidth]{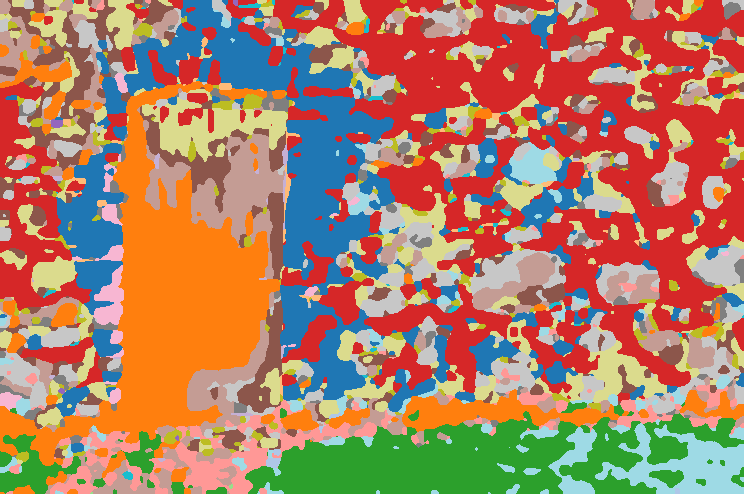}
			&
			\includegraphics[width=#1\textwidth]{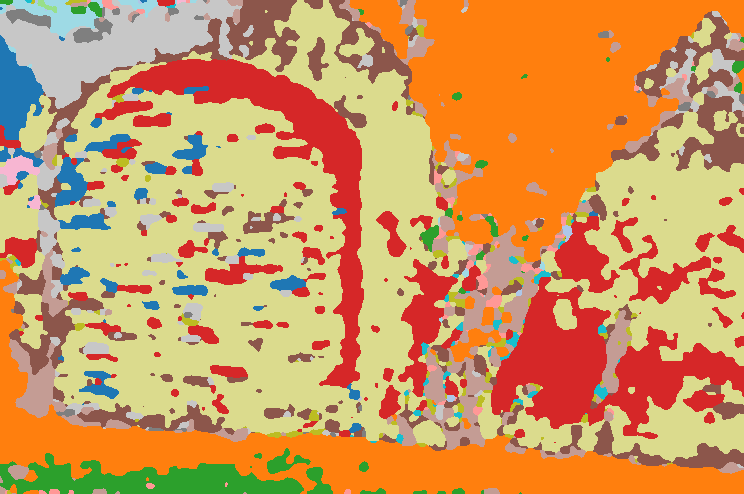}
			\\

			\rotatebox{90}{\quad  \footnotesize \textbf{assignment}}
			&
			\includegraphics[width=#1\textwidth]{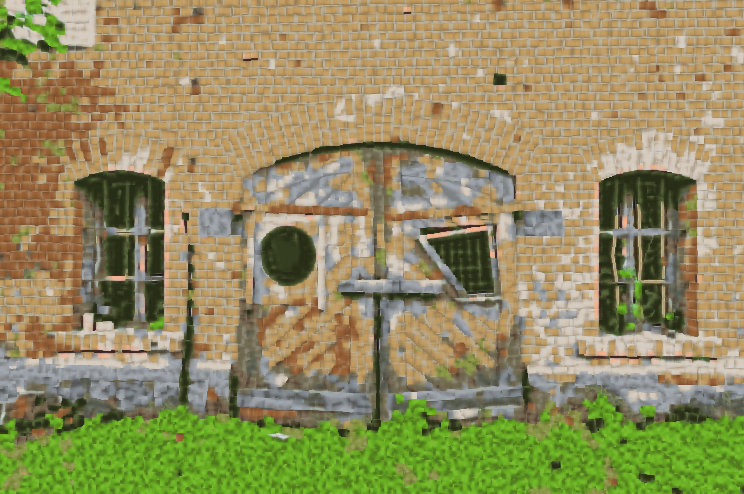}
			&
			\includegraphics[width=#1\textwidth]{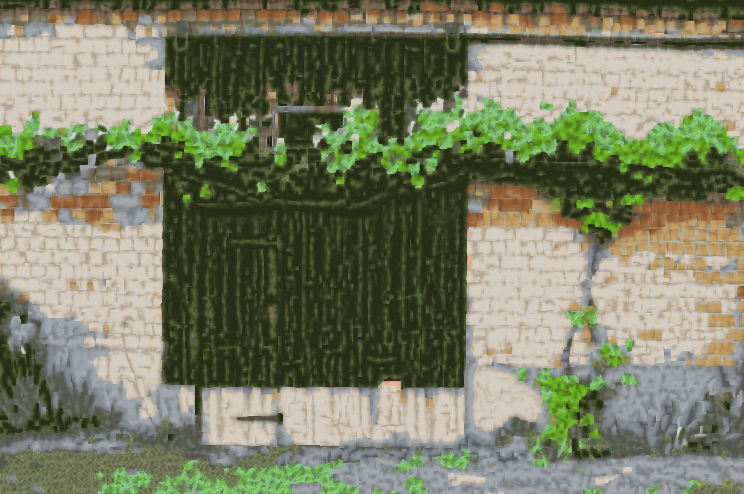}
			&
			\includegraphics[width=#1\textwidth]{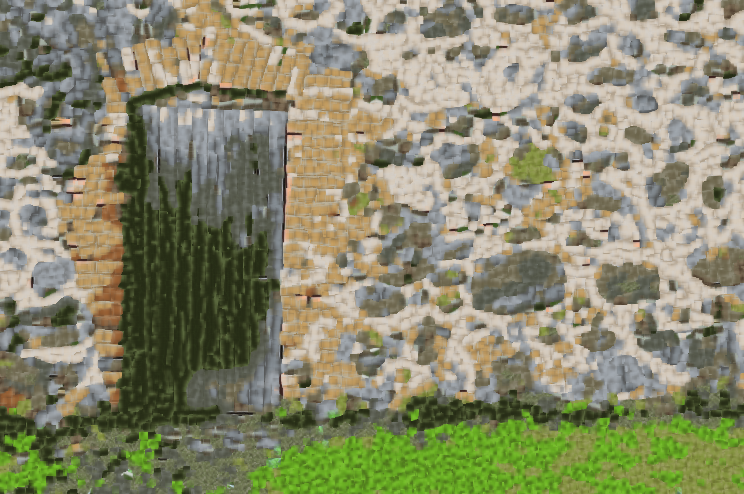}
			&
			\includegraphics[width=#1\textwidth]{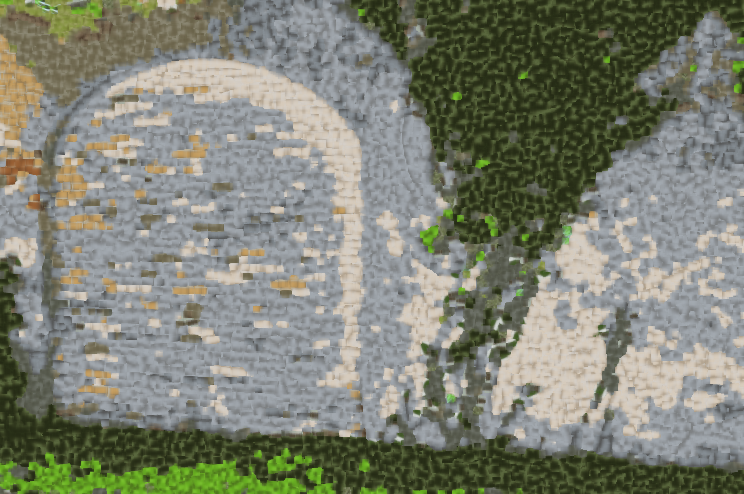}

		\end{tabular}
	\end{centering}
}

\newcommand{\showExpRGBiiia}[1]{
	\begin{centering}
		\begin{tabular}{c@{\hskip 0.6em}c@{\hskip 0.4em}c@{\hskip 0.8em}c@{\hskip 0.4em}c}
			& \multicolumn{2}{c@{\hskip 0.6em}}{\footnotesize \textbf{SAF, } \boldmath $s=0$} & \multicolumn{2}{c@{\hskip 0.8em}}{\footnotesize \textbf{SAF, } \boldmath $s=1$} \\
			\cmidrule(l{1.2em}r{2.0em}){2-3} \cmidrule(l{1.2em}r{1.65em}){4-5}
			& \footnotesize uniform & \footnotesize non-uniform & \footnotesize uniform & \footnotesize non-uniform \\[0.6em]

			\rotatebox{90}{\quad \, \footnotesize \textbf{partition}}
			&
			\includegraphics[width=#1\textwidth]{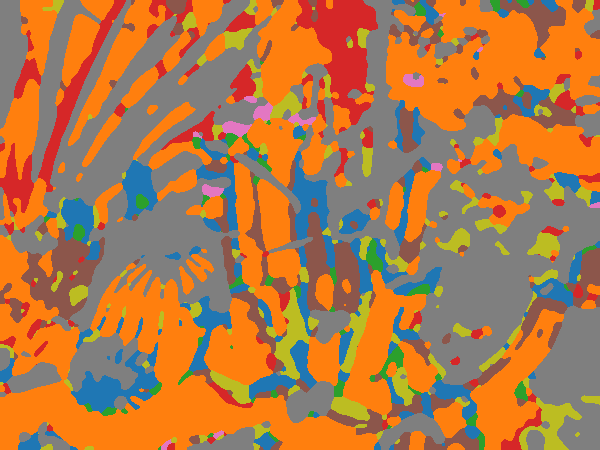}
			&
			\includegraphics[width=#1\textwidth]{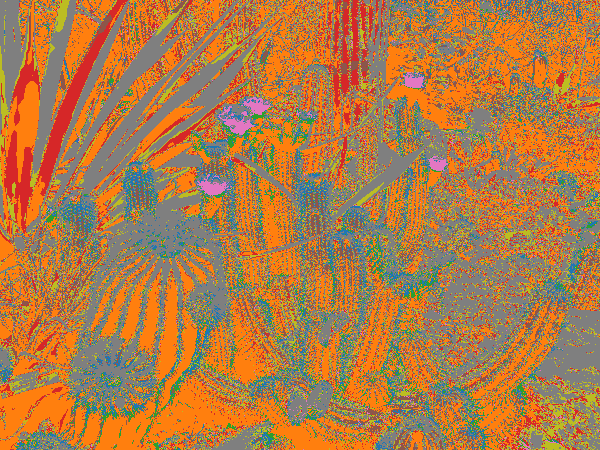}
			&
			\includegraphics[width=#1\textwidth]{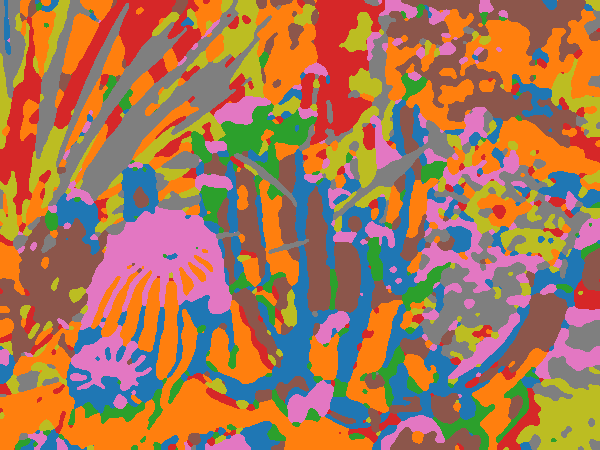}
			&
			\includegraphics[width=#1\textwidth]{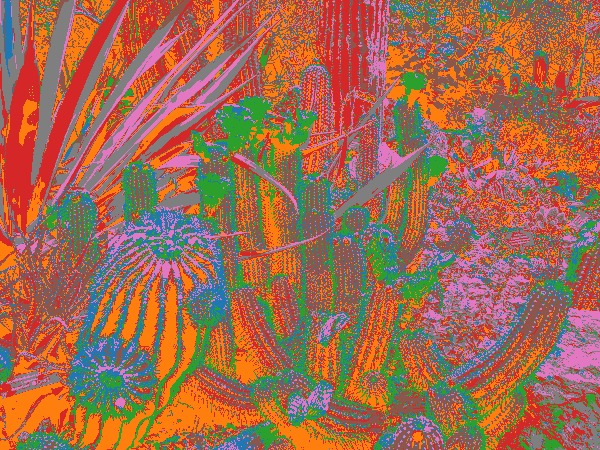}
			\\

			\rotatebox{90}{\quad \footnotesize \textbf{assignment}}
			&
			\includegraphics[width=#1\textwidth]{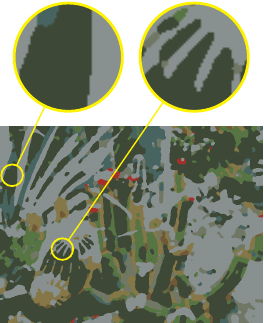}
			&
			\includegraphics[width=#1\textwidth]{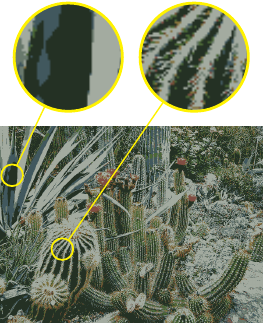}
			&
			\includegraphics[width=#1\textwidth]{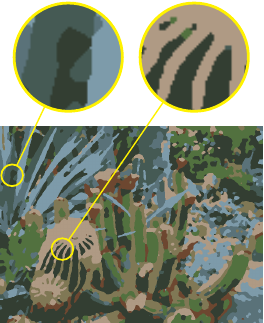}
			&
			\includegraphics[width=#1\textwidth]{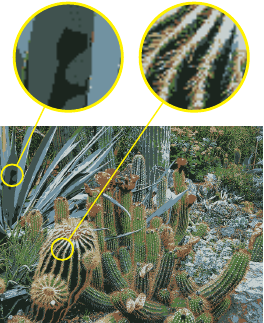}
			\\[-0.18em]

			\rotatebox{90}{\footnotesize \boldmath $\, \mc{F}_{\ast}$}
			&
			\includegraphics[width=#1\textwidth]{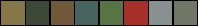}
			&
			\includegraphics[width=#1\textwidth]{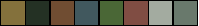}
			&
			\includegraphics[width=#1\textwidth]{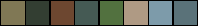}
			&
			\includegraphics[width=#1\textwidth]{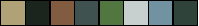}
			\\[1.2em]

			& \multicolumn{2}{c@{\hskip 0.6em}}{\footnotesize \textbf{AF} supervised \cite{Astrom:2017ac}} & \multicolumn{2}{c@{\hskip 0.8em}}{\footnotesize \textbf{Nearest Neighbor}} \\
			\cmidrule(l{1.2em}r{2.0em}){2-3} \cmidrule(l{1.2em}r{1.65em}){4-5}
			& \footnotesize k-center & \footnotesize k-means & \footnotesize k-center & \footnotesize k-means \\[0.6em]

			\rotatebox{90}{\quad \, \footnotesize \textbf{partition}}
			&
			\includegraphics[width=#1\textwidth]{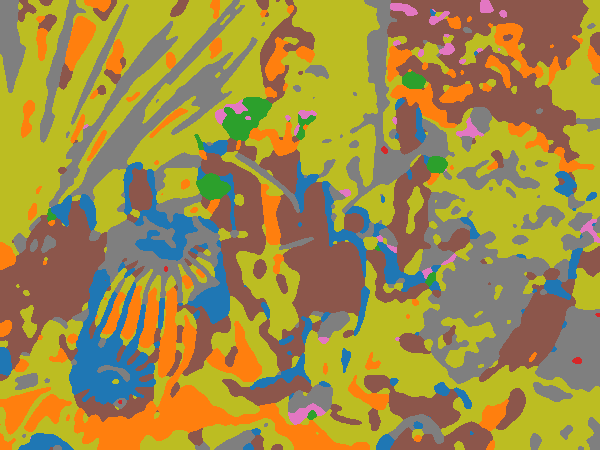}
			&
			\includegraphics[width=#1\textwidth]{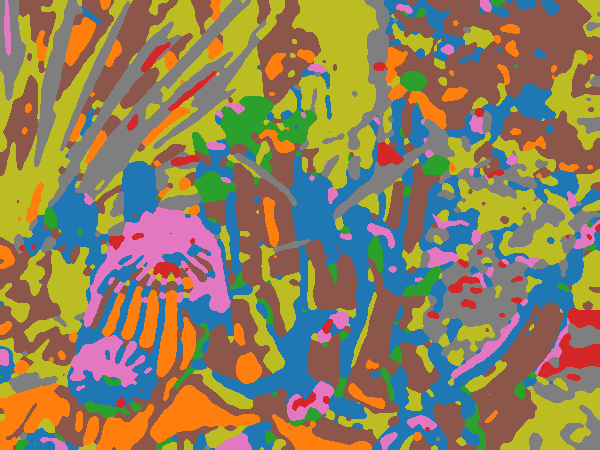}
			&
			\includegraphics[width=#1\textwidth]{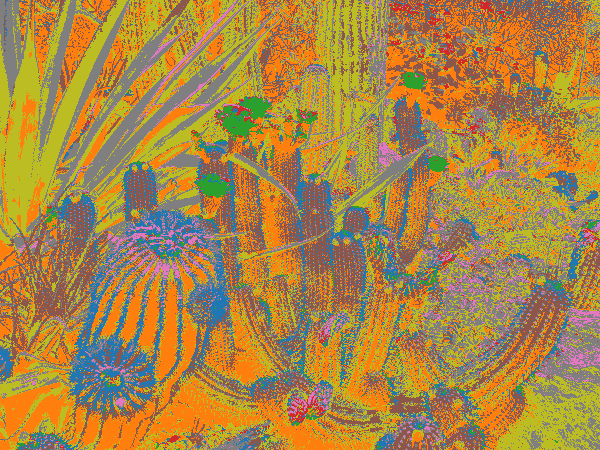}
			&
			\includegraphics[width=#1\textwidth]{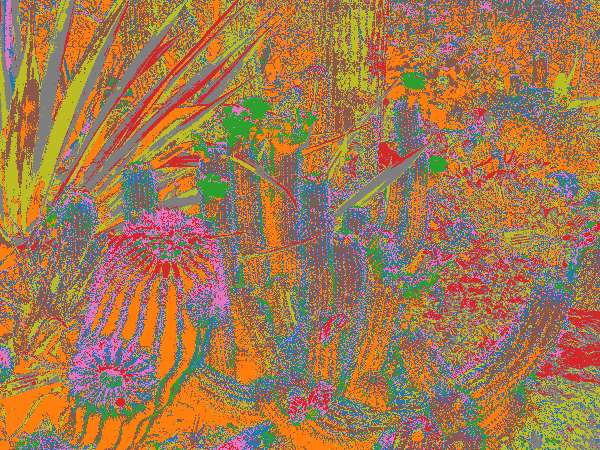}
			\\

			\rotatebox{90}{\quad \footnotesize \textbf{assignment}}
			&
			\includegraphics[width=#1\textwidth]{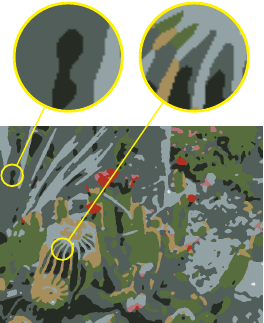}
			&
			\includegraphics[width=#1\textwidth]{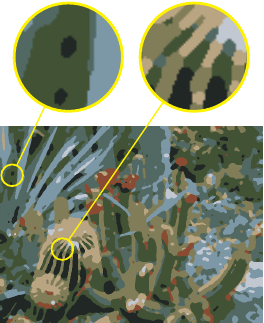}
			&
			\includegraphics[width=#1\textwidth]{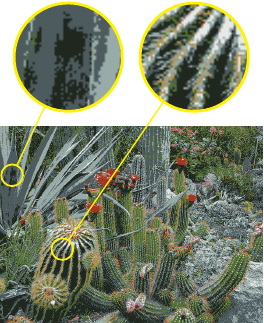}
			&
			\includegraphics[width=#1\textwidth]{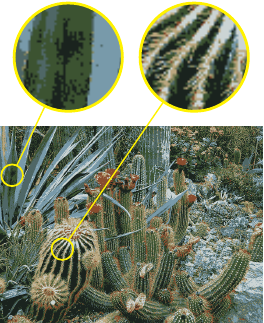}
			\\[-0.18em]

			\rotatebox{90}{\footnotesize \boldmath $\, \mc{F}_{\ast}$}
			&
			\includegraphics[width=#1\textwidth]{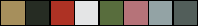}
			&
			\includegraphics[width=#1\textwidth]{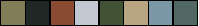}
			&
			\includegraphics[width=#1\textwidth]{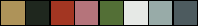}
			&
			\includegraphics[width=#1\textwidth]{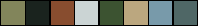}

		\end{tabular}
	\end{centering}
}
\newcommand{\showExpRGBiiib}[1]{
	\begin{centering}
		\begin{tabular}{c@{\hskip 0.6em}c@{\hskip 0.4em}c@{\hskip 0.8em}c@{\hskip 0.4em}c}

			& \multicolumn{2}{c@{\hskip 0.6em}}{\footnotesize \textbf{Spectral Clustering} \cite{Shi2000} } & \multicolumn{2}{c@{\hskip 0.8em}}{\footnotesize \textbf{Fast Partitioning} \cite{Storath2014} } \\
			\cmidrule(l{1.2em}r{2.0em}){2-3} \cmidrule(l{1.2em}r{1.65em}){4-5}
			& \footnotesize $\alpha = 0.12$ & \footnotesize $\alpha = 0.6$ & \footnotesize $\gamma = 0.1$ & \footnotesize $\gamma = 0.3$ \\[0.6em]

			\rotatebox{90}{\quad \, \footnotesize \textbf{partition}}
			&
			\includegraphics[width=#1\textwidth]{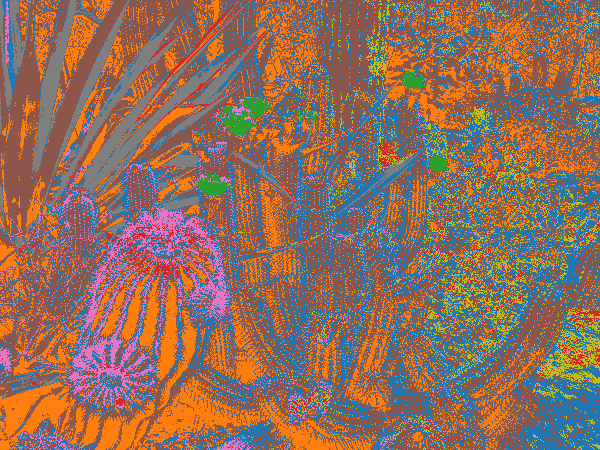}
			&
			\includegraphics[width=#1\textwidth]{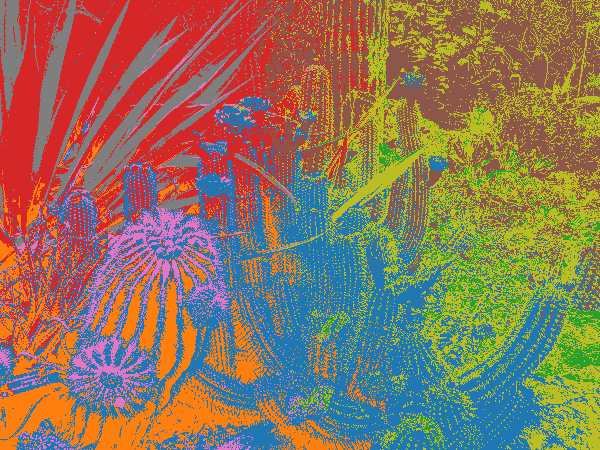}
			&
			\includegraphics[width=#1\textwidth]{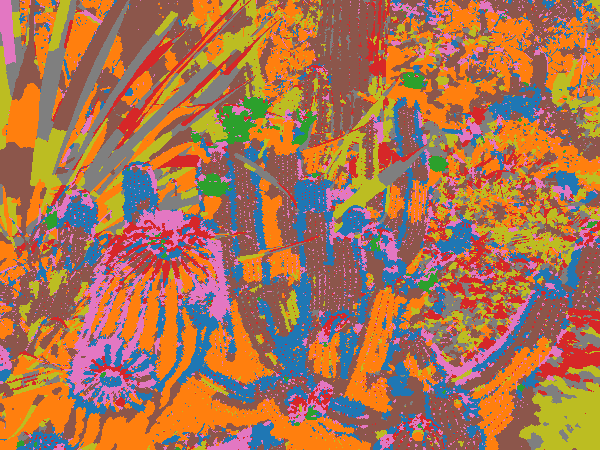}
			&
			\includegraphics[width=#1\textwidth]{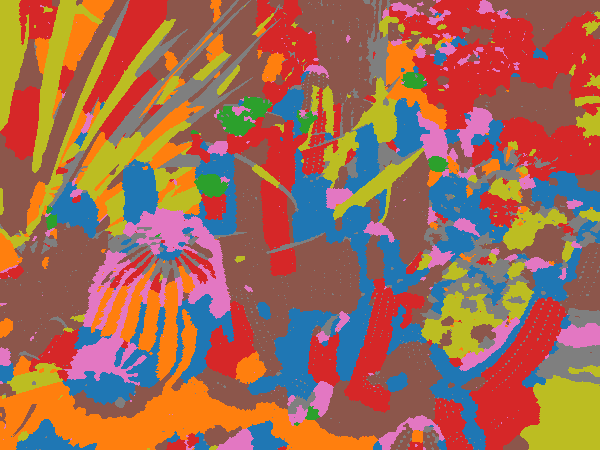}
			\\

			\rotatebox{90}{\quad \footnotesize \textbf{assignment}}
			&
			\includegraphics[width=#1\textwidth]{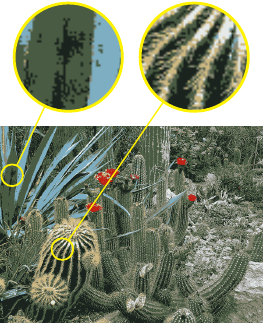}
			&
			\includegraphics[width=#1\textwidth]{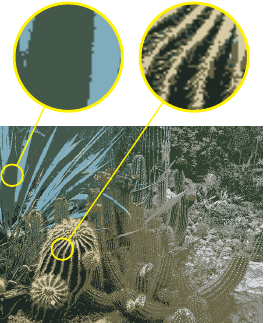}
			&
			\includegraphics[width=#1\textwidth]{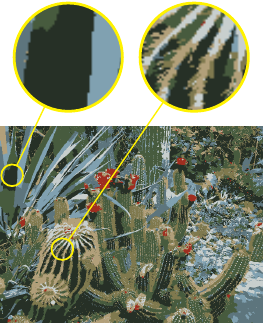}
			&
			\includegraphics[width=#1\textwidth]{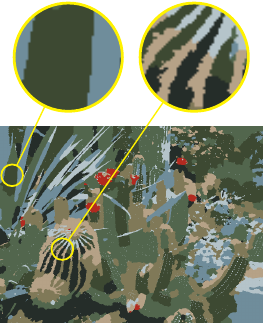}
			\\[-0.18em]

			\rotatebox{90}{\footnotesize \boldmath $\, \mc{F}_{\ast}$}
			&
			\includegraphics[width=#1\textwidth]{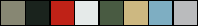}
			&
			\includegraphics[width=#1\textwidth]{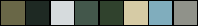}
			&
			\includegraphics[width=#1\textwidth]{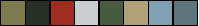}
			&
			\includegraphics[width=#1\textwidth]{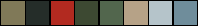}
		\end{tabular}
	\end{centering}
}

\newcommand{\showExpPatchesII}[3]{
	\begin{centering}
		\begin{tabular}{c@{\hskip 0.6em}c@{\hskip 0.4em}c@{\hskip 0.4em}c@{\hskip 0.8em}c@{\hskip 0.4em}c@{\hskip 0.4em}c}
			& \multicolumn{3}{c@{\hskip 0.6em}}{\footnotesize \textbf{SAF, } \boldmath $s=0$} & \multicolumn{3}{c@{\hskip 0.8em}}{\footnotesize \textbf{SAF, } \boldmath $s=1$} \\
			\cmidrule(l{1.4em}r{2.2em}){2-4} \cmidrule(l{1.4em}r{1.85em}){5-7}
			&
			\footnotesize \boldmath $7 \times 7$ & \footnotesize \boldmath $11 \times 11$ & \footnotesize \boldmath $15 \times 15$ & \footnotesize \boldmath $7 \times 7$ & \footnotesize \boldmath $11 \times 11$  & \footnotesize \boldmath $15 \times 15$ \\[0.6em]
			\rotatebox{90}{\quad \, \footnotesize \textbf{partition}}
			&
			\includegraphics[width=#1\textwidth]{{./Figures/Patches/#2/self_assignment/Assignment_False_Colors_self_assignment_K=4_size_neigh=#3_s=0.0_p=7_rotations=True}.png}
			&
			\includegraphics[width=#1\textwidth]{{./Figures/Patches/#2/self_assignment/Assignment_False_Colors_self_assignment_K=4_size_neigh=#3_s=0.0_p=11_rotations=True}.png}
			&
			\includegraphics[width=#1\textwidth]{{./Figures/Patches/#2/self_assignment/Assignment_False_Colors_self_assignment_K=4_size_neigh=#3_s=0.0_p=15_rotations=True}.png}
			&
			\includegraphics[width=#1\textwidth]{{./Figures/Patches/#2/self_assignment/Assignment_False_Colors_self_assignment_K=4_size_neigh=#3_s=1.0_p=7_rotations=True}.png}
			&
			\includegraphics[width=#1\textwidth]{{./Figures/Patches/#2/self_assignment/Assignment_False_Colors_self_assignment_K=4_size_neigh=#3_s=1.0_p=11_rotations=True}.png}
			&
			\includegraphics[width=#1\textwidth]{{./Figures/Patches/#2/self_assignment/Assignment_False_Colors_self_assignment_K=4_size_neigh=#3_s=1.0_p=15_rotations=True}.png}
			\\

			\rotatebox{90}{\quad \footnotesize \textbf{assignment}}
			&
			\includegraphics[width=#1\textwidth]{{./Figures/Patches/#2/self_assignment/Assignment_self_assignment_K=4_size_neigh=#3_s=0.0_p=7_rotations=True}.png}
			&
			\includegraphics[width=#1\textwidth]{{./Figures/Patches/#2/self_assignment/Assignment_self_assignment_K=4_size_neigh=#3_s=0.0_p=11_rotations=True}.png}
			&
			\includegraphics[width=#1\textwidth]{{./Figures/Patches/#2/self_assignment/Assignment_self_assignment_K=4_size_neigh=#3_s=0.0_p=15_rotations=True}.png}
			&
			\includegraphics[width=#1\textwidth]{{./Figures/Patches/#2/self_assignment/Assignment_self_assignment_K=4_size_neigh=#3_s=1.0_p=7_rotations=True}.png}
			&
			\includegraphics[width=#1\textwidth]{{./Figures/Patches/#2/self_assignment/Assignment_self_assignment_K=4_size_neigh=#3_s=1.0_p=11_rotations=True}.png}
			&
			\includegraphics[width=#1\textwidth]{{./Figures/Patches/#2/self_assignment/Assignment_self_assignment_K=4_size_neigh=#3_s=1.0_p=15_rotations=True}.png}
			\\[0.05em]

			\rotatebox{90}{\footnotesize \boldmath $\, \mc{F}_{\ast}$}
			&
			\includegraphics[width=#1\textwidth]{{./Figures/Patches/#2/self_assignment/Prototypes_Evolved_self_assignment_K=4_size_neigh=#3_s=0.0_p=7_rotations=True}.png}
			&
			\includegraphics[width=#1\textwidth]{{./Figures/Patches/#2/self_assignment/Prototypes_Evolved_self_assignment_K=4_size_neigh=#3_s=0.0_p=11_rotations=True}.png}
			&
			\includegraphics[width=#1\textwidth]{{./Figures/Patches/#2/self_assignment/Prototypes_Evolved_self_assignment_K=4_size_neigh=#3_s=0.0_p=15_rotations=True}.png}
			&
			\includegraphics[width=#1\textwidth]{{./Figures/Patches/#2/self_assignment/Prototypes_Evolved_self_assignment_K=4_size_neigh=#3_s=1.0_p=7_rotations=True}.png}
			&
			\includegraphics[width=#1\textwidth]{{./Figures/Patches/#2/self_assignment/Prototypes_Evolved_self_assignment_K=4_size_neigh=#3_s=1.0_p=11_rotations=True}.png}
			&
			\includegraphics[width=#1\textwidth]{{./Figures/Patches/#2/self_assignment/Prototypes_Evolved_self_assignment_K=4_size_neigh=#3_s=1.0_p=15_rotations=True}.png}
			\\[0.6em]

			\rotatebox{90}{\quad \, \footnotesize \textbf{partition}}
			&
			\includegraphics[width=#1\textwidth]{{./Figures/Patches/#2/self_assignment/Assignment_False_Colors_self_assignment_K=10_size_neigh=#3_s=0.0_p=7_rotations=True}.png}
			&
			\includegraphics[width=#1\textwidth]{{./Figures/Patches/#2/self_assignment/Assignment_False_Colors_self_assignment_K=10_size_neigh=#3_s=0.0_p=11_rotations=True}.png}
			&
			\includegraphics[width=#1\textwidth]{{./Figures/Patches/#2/self_assignment/Assignment_False_Colors_self_assignment_K=10_size_neigh=#3_s=0.0_p=15_rotations=True}.png}
			&
			\includegraphics[width=#1\textwidth]{{./Figures/Patches/#2/self_assignment/Assignment_False_Colors_self_assignment_K=10_size_neigh=#3_s=1.0_p=7_rotations=True}.png}
			&
			\includegraphics[width=#1\textwidth]{{./Figures/Patches/#2/self_assignment/Assignment_False_Colors_self_assignment_K=10_size_neigh=#3_s=1.0_p=11_rotations=True}.png}
			&
			\includegraphics[width=#1\textwidth]{{./Figures/Patches/#2/self_assignment/Assignment_False_Colors_self_assignment_K=10_size_neigh=#3_s=1.0_p=15_rotations=True}.png}
			\\

			\rotatebox{90}{\quad \footnotesize \textbf{assignment}}
			&
			\includegraphics[width=#1\textwidth]{{./Figures/Patches/#2/self_assignment/Assignment_self_assignment_K=10_size_neigh=#3_s=0.0_p=7_rotations=True}.png}
			&
			\includegraphics[width=#1\textwidth]{{./Figures/Patches/#2/self_assignment/Assignment_self_assignment_K=10_size_neigh=#3_s=0.0_p=11_rotations=True}.png}
			&
			\includegraphics[width=#1\textwidth]{{./Figures/Patches/#2/self_assignment/Assignment_self_assignment_K=10_size_neigh=#3_s=0.0_p=15_rotations=True}.png}
			&
			\includegraphics[width=#1\textwidth]{{./Figures/Patches/#2/self_assignment/Assignment_self_assignment_K=10_size_neigh=#3_s=1.0_p=7_rotations=True}.png}
			&
			\includegraphics[width=#1\textwidth]{{./Figures/Patches/#2/self_assignment/Assignment_self_assignment_K=10_size_neigh=#3_s=1.0_p=11_rotations=True}.png}
			&
			\includegraphics[width=#1\textwidth]{{./Figures/Patches/#2/self_assignment/Assignment_self_assignment_K=10_size_neigh=#3_s=1.0_p=15_rotations=True}.png}
			\\[0.05em]

			\rotatebox{90}{\; \footnotesize \boldmath $\, \mc{F}_{\ast}$}
			&
			\includegraphics[width=#1\textwidth]{{./Figures/Patches/#2/self_assignment/Prototypes_Evolved_self_assignment_K=10_size_neigh=#3_s=0.0_p=7_rotations=True}.png}
			&
			\includegraphics[width=#1\textwidth]{{./Figures/Patches/#2/self_assignment/Prototypes_Evolved_self_assignment_K=10_size_neigh=#3_s=0.0_p=11_rotations=True}.png}
			&
			\includegraphics[width=#1\textwidth]{{./Figures/Patches/#2/self_assignment/Prototypes_Evolved_self_assignment_K=10_size_neigh=#3_s=0.0_p=15_rotations=True}.png}
			&
			\includegraphics[width=#1\textwidth]{{./Figures/Patches/#2/self_assignment/Prototypes_Evolved_self_assignment_K=10_size_neigh=#3_s=1.0_p=7_rotations=True}.png}
			&
			\includegraphics[width=#1\textwidth]{{./Figures/Patches/#2/self_assignment/Prototypes_Evolved_self_assignment_K=10_size_neigh=#3_s=1.0_p=11_rotations=True}.png}
			&
			\includegraphics[width=#1\textwidth]{{./Figures/Patches/#2/self_assignment/Prototypes_Evolved_self_assignment_K=10_size_neigh=#3_s=1.0_p=15_rotations=True}.png}

		\end{tabular}
\end{centering}
}

\newcommand{\showExpRGBi}[2]{
	\begin{centering}
		\begin{tabular}{c@{\hskip0.6em}c@{\hskip 0.4em}c@{\hskip0.4em}c@{\hskip0.4em}c@{\hskip0.4em}c}
			& \footnotesize \boldmath $s=0$ & \footnotesize \boldmath $s=0.25$ & \footnotesize \boldmath $s=0.50$ & \footnotesize \boldmath $s=0.75$ & \footnotesize \boldmath $s=1$ \\[0.6em]
			\rotatebox{90}{\quad\, \footnotesize \boldmath$3\times 3$}
			&
			\includegraphics[width=#1\textwidth]{{./Figures/RGB/#2/self_assignment/Assignment_self_assignment_K=16_size_neigh=3_s=0.0}.png}
			&
			\includegraphics[width=#1\textwidth]{{./Figures/RGB/#2/self_assignment/Assignment_self_assignment_K=16_size_neigh=3_s=0.25}.png}
			&
			\includegraphics[width=#1\textwidth]{{./Figures/RGB/#2/self_assignment/Assignment_self_assignment_K=16_size_neigh=3_s=0.5}.png}
			&
			\includegraphics[width=#1\textwidth]{{./Figures/RGB/#2/self_assignment/Assignment_self_assignment_K=16_size_neigh=3_s=0.75}.png}
			&
			\includegraphics[width=#1\textwidth]{{./Figures/RGB/#2/self_assignment/Assignment_self_assignment_K=16_size_neigh=3_s=1.0}.png}
			\\[-0.31em]

			&
			\includegraphics[width=#1\textwidth]{{./Figures/RGB/#2/self_assignment/Prototypes_EvolvedF_self_assignment_K=16_size_neigh=3_s=0.0}.pdf}
			&
			\includegraphics[width=#1\textwidth]{{./Figures/RGB/#2/self_assignment/Prototypes_EvolvedF_self_assignment_K=16_size_neigh=3_s=0.25}.pdf}
			&
			\includegraphics[width=#1\textwidth]{{./Figures/RGB/#2/self_assignment/Prototypes_EvolvedF_self_assignment_K=16_size_neigh=3_s=0.5}.pdf}
			&
			\includegraphics[width=#1\textwidth]{{./Figures/RGB/#2/self_assignment/Prototypes_EvolvedF_self_assignment_K=16_size_neigh=3_s=0.75}.pdf}
			&
			\includegraphics[width=#1\textwidth]{{./Figures/RGB/#2/self_assignment/Prototypes_EvolvedF_self_assignment_K=16_size_neigh=3_s=1.0}.pdf}
			\\[0.6em]

			\rotatebox{90}{\quad\, \footnotesize \boldmath$7\times 7$ }
			&
			\includegraphics[width=#1\textwidth]{{./Figures/RGB/#2/self_assignment/Assignment_self_assignment_K=16_size_neigh=7_s=0.0}.png}
			&
			\includegraphics[width=#1\textwidth]{{./Figures/RGB/#2/self_assignment/Assignment_self_assignment_K=16_size_neigh=7_s=0.25}.png}
			&
			\includegraphics[width=#1\textwidth]{{./Figures/RGB/#2/self_assignment/Assignment_self_assignment_K=16_size_neigh=7_s=0.5}.png}
			&
			\includegraphics[width=#1\textwidth]{{./Figures/RGB/#2/self_assignment/Assignment_self_assignment_K=16_size_neigh=7_s=0.75}.png}
			&
			\includegraphics[width=#1\textwidth]{{./Figures/RGB/#2/self_assignment/Assignment_self_assignment_K=16_size_neigh=7_s=1.0}.png}
			\\[-0.31em]

			&
			\includegraphics[width=#1\textwidth]{{./Figures/RGB/#2/self_assignment/Prototypes_EvolvedF_self_assignment_K=16_size_neigh=7_s=0.0}.pdf}
			&
			\includegraphics[width=#1\textwidth]{{./Figures/RGB/#2/self_assignment/Prototypes_EvolvedF_self_assignment_K=16_size_neigh=7_s=0.25}.pdf}
			&
			\includegraphics[width=#1\textwidth]{{./Figures/RGB/#2/self_assignment/Prototypes_EvolvedF_self_assignment_K=16_size_neigh=7_s=0.5}.pdf}
			&
			\includegraphics[width=#1\textwidth]{{./Figures/RGB/#2/self_assignment/Prototypes_EvolvedF_self_assignment_K=16_size_neigh=7_s=0.75}.pdf}
			&
			\includegraphics[width=#1\textwidth]{{./Figures/RGB/#2/self_assignment/Prototypes_EvolvedF_self_assignment_K=16_size_neigh=7_s=1.0}.pdf}
			\\[0.6em]

			\rotatebox{90}{\quad \footnotesize \boldmath$11\times 11$ }
			&
			\includegraphics[width=#1\textwidth]{{./Figures/RGB/#2/self_assignment/Assignment_self_assignment_K=16_size_neigh=11_s=0.0}.png}
			&
			\includegraphics[width=#1\textwidth]{{./Figures/RGB/#2/self_assignment/Assignment_self_assignment_K=16_size_neigh=11_s=0.25}.png}
			&
			\includegraphics[width=#1\textwidth]{{./Figures/RGB/#2/self_assignment/Assignment_self_assignment_K=16_size_neigh=11_s=0.5}.png}
			&
			\includegraphics[width=#1\textwidth]{{./Figures/RGB/#2/self_assignment/Assignment_self_assignment_K=16_size_neigh=11_s=0.75}.png}
			&
			\includegraphics[width=#1\textwidth]{{./Figures/RGB/#2/self_assignment/Assignment_self_assignment_K=16_size_neigh=11_s=1.0}.png}
			\\[-0.31em]

			&
			\includegraphics[width=#1\textwidth]{{./Figures/RGB/#2/self_assignment/Prototypes_EvolvedF_self_assignment_K=16_size_neigh=11_s=0.0}.pdf}
			&
			\includegraphics[width=#1\textwidth]{{./Figures/RGB/#2/self_assignment/Prototypes_EvolvedF_self_assignment_K=16_size_neigh=11_s=0.25}.pdf}
			&
			\includegraphics[width=#1\textwidth]{{./Figures/RGB/#2/self_assignment/Prototypes_EvolvedF_self_assignment_K=16_size_neigh=11_s=0.5}.pdf}
			&
			\includegraphics[width=#1\textwidth]{{./Figures/RGB/#2/self_assignment/Prototypes_EvolvedF_self_assignment_K=16_size_neigh=11_s=0.75}.pdf}
			&
			\includegraphics[width=#1\textwidth]{{./Figures/RGB/#2/self_assignment/Prototypes_EvolvedF_self_assignment_K=16_size_neigh=11_s=1.0}.pdf}
			\\[0.6em]

			\rotatebox{90}{\quad \footnotesize \boldmath$21\times 21$ }
			&
			\includegraphics[width=#1\textwidth]{{./Figures/RGB/#2/self_assignment/Assignment_self_assignment_K=16_size_neigh=21_s=0.0}.png}
			&
			\includegraphics[width=#1\textwidth]{{./Figures/RGB/#2/self_assignment/Assignment_self_assignment_K=16_size_neigh=21_s=0.25}.png}
			&
			\includegraphics[width=#1\textwidth]{{./Figures/RGB/#2/self_assignment/Assignment_self_assignment_K=16_size_neigh=21_s=0.5}.png}
			&
			\includegraphics[width=#1\textwidth]{{./Figures/RGB/#2/self_assignment/Assignment_self_assignment_K=16_size_neigh=21_s=0.75}.png}
			&
			\includegraphics[width=#1\textwidth]{{./Figures/RGB/#2/self_assignment/Assignment_self_assignment_K=16_size_neigh=21_s=1.0}.png}
			\\[-0.31em]

			&
			\includegraphics[width=#1\textwidth]{{./Figures/RGB/#2/self_assignment/Prototypes_EvolvedF_self_assignment_K=16_size_neigh=21_s=0.0}.pdf}
			&
			\includegraphics[width=#1\textwidth]{{./Figures/RGB/#2/self_assignment/Prototypes_EvolvedF_self_assignment_K=16_size_neigh=21_s=0.25}.pdf}
			&
			\includegraphics[width=#1\textwidth]{{./Figures/RGB/#2/self_assignment/Prototypes_EvolvedF_self_assignment_K=16_size_neigh=21_s=0.5}.pdf}
			&
			\includegraphics[width=#1\textwidth]{{./Figures/RGB/#2/self_assignment/Prototypes_EvolvedF_self_assignment_K=16_size_neigh=21_s=0.75}.pdf}
			&
			\includegraphics[width=#1\textwidth]{{./Figures/RGB/#2/self_assignment/Prototypes_EvolvedF_self_assignment_K=16_size_neigh=21_s=1.0}.pdf}
		\end{tabular}
\end{centering}
}

\newcommand{\showExpRGBLabelHist}[3]{
	\begin{centering}
		\begin{tabular}{c@{\hskip 0.6em}c@{\hskip 0.4em}c@{\hskip 0.4em}c@{\hskip 0.4em}c@{\hskip 0.4em}c}
			& \footnotesize \boldmath $s=0$ & \footnotesize \boldmath $s=0.25$ & \footnotesize \boldmath $s=0.50$ & \footnotesize \boldmath $s=0.75$ & \footnotesize \boldmath $s=1$ \\[0.6em]
			\rotatebox{90}{\, \footnotesize \textbf{cluster sizes}}
			&
			\includegraphics[width=#1\textwidth]{{./Figures/RGB/#2/self_assignment/Whist_Plot_self_assignment_K=16_size_neigh=#3_s=0.0}.pdf}
			&
			\includegraphics[width=#1\textwidth]{{./Figures/RGB/#2/self_assignment/Whist_Plot_self_assignment_K=16_size_neigh=#3_s=0.25}.pdf}
			&
			\includegraphics[width=#1\textwidth]{{./Figures/RGB/#2/self_assignment/Whist_Plot_self_assignment_K=16_size_neigh=#3_s=0.5}.pdf}
			&
			\includegraphics[width=#1\textwidth]{{./Figures/RGB/#2/self_assignment/Whist_Plot_self_assignment_K=16_size_neigh=#3_s=0.75}.pdf}
			&
			\includegraphics[width=#1\textwidth]{{./Figures/RGB/#2/self_assignment/Whist_Plot_self_assignment_K=16_size_neigh=#3_s=1.0}.pdf}
			\\[0.6em]

			\rotatebox{90}{\quad \  \footnotesize \textbf{entropy}}
			&
			\includegraphics[width=#1\textwidth]{{./Figures/RGB/#2/self_assignment/Entropy_self_assignment_K=16_size_neigh=#3_s=0.0}.pdf}
			&
			\includegraphics[width=#1\textwidth]{{./Figures/RGB/#2/self_assignment/Entropy_self_assignment_K=16_size_neigh=#3_s=0.25}.pdf}
			&
			\includegraphics[width=#1\textwidth]{{./Figures/RGB/#2/self_assignment/Entropy_self_assignment_K=16_size_neigh=#3_s=0.5}.pdf}
			&
			\includegraphics[width=#1\textwidth]{{./Figures/RGB/#2/self_assignment/Entropy_self_assignment_K=16_size_neigh=#3_s=0.75}.pdf}
			&
			\includegraphics[width=#1\textwidth]{{./Figures/RGB/#2/self_assignment/Entropy_self_assignment_K=16_size_neigh=#3_s=1.0}.pdf}
			\\[0.6em]

			\rotatebox{90}{\, \footnotesize \boldmath $\tr \big( B(W) \big)$}
			&
			\includegraphics[width=#1\textwidth]{{./Figures/RGB/#2/self_assignment/Tr_WCtW_Plot_self_assignment_K=16_size_neigh=#3_s=0.0}.pdf}
			&
			\includegraphics[width=#1\textwidth]{{./Figures/RGB/#2/self_assignment/Tr_WCtW_Plot_self_assignment_K=16_size_neigh=#3_s=0.25}.pdf}
			&
			\includegraphics[width=#1\textwidth]{{./Figures/RGB/#2/self_assignment/Tr_WCtW_Plot_self_assignment_K=16_size_neigh=#3_s=0.5}.pdf}
			&
			\includegraphics[width=#1\textwidth]{{./Figures/RGB/#2/self_assignment/Tr_WCtW_Plot_self_assignment_K=16_size_neigh=#3_s=0.75}.pdf}
			&
			\includegraphics[width=#1\textwidth]{{./Figures/RGB/#2/self_assignment/Tr_WCtW_Plot_self_assignment_K=16_size_neigh=#3_s=1.0}.pdf}

		\end{tabular}
\end{centering}
}

\newcommand{\showExpGraph}[1]{
	\begin{centering}
		\begin{tabular}{c@{\hskip 0.4em}c}

			\footnotesize \textbf{Weighted Graph} & \footnotesize \textbf{Ground Truth} \\[0.2em]
			\includegraphics[width=#1\textwidth]{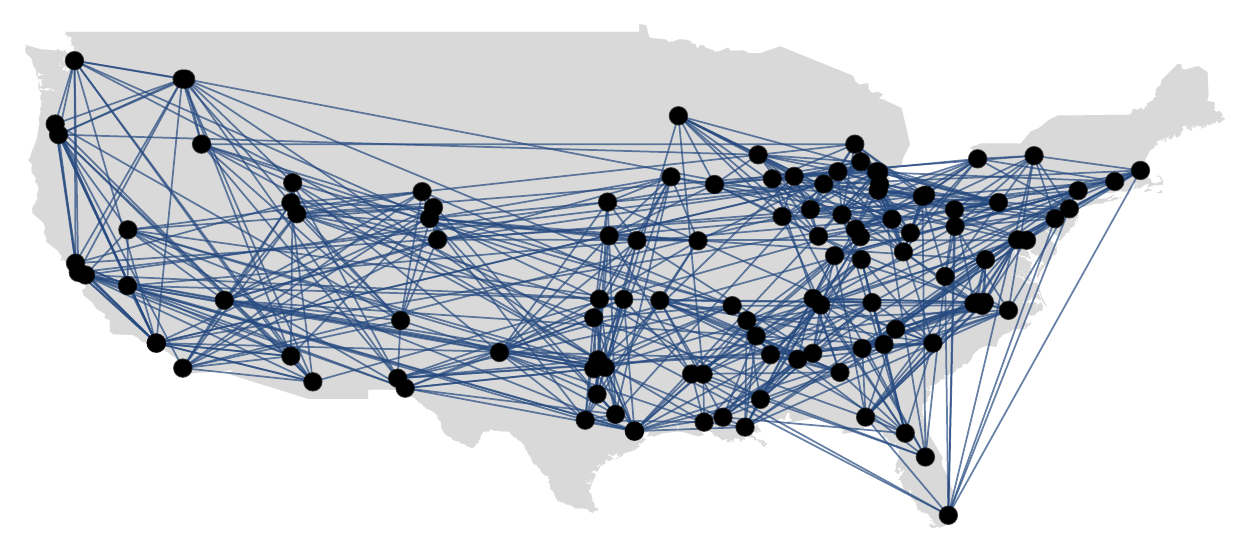}
			&
			\includegraphics[width=#1\textwidth]{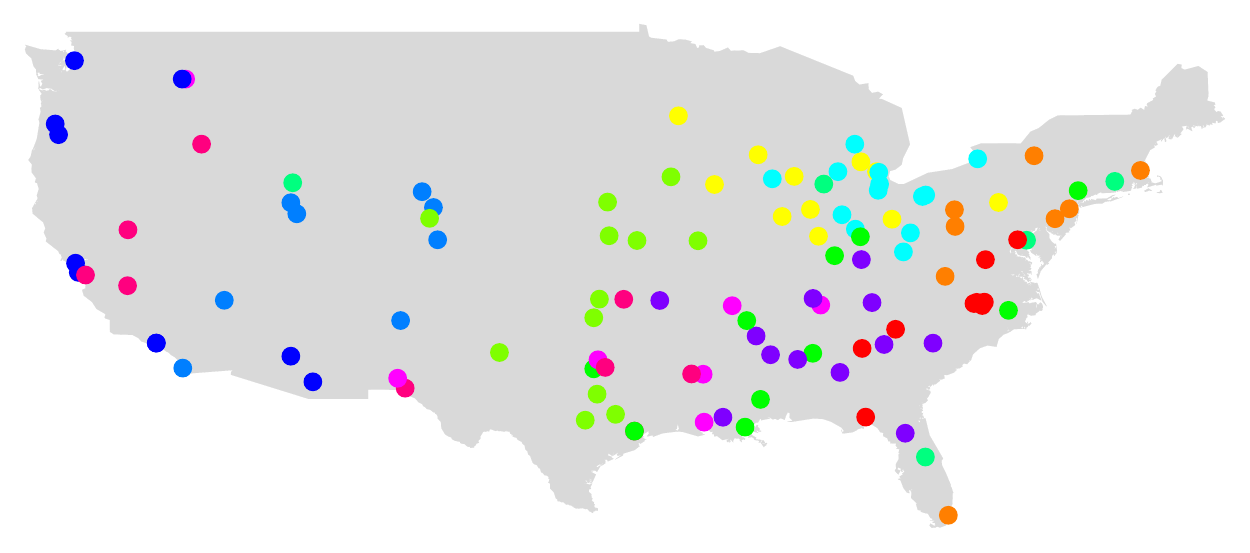}
			\\[0.8em]

			\footnotesize \textbf{Initialization} & \footnotesize \textbf{Spectral Clustering \cite{Shi2000}}\\[0.2em]
			\includegraphics[width=#1\textwidth]{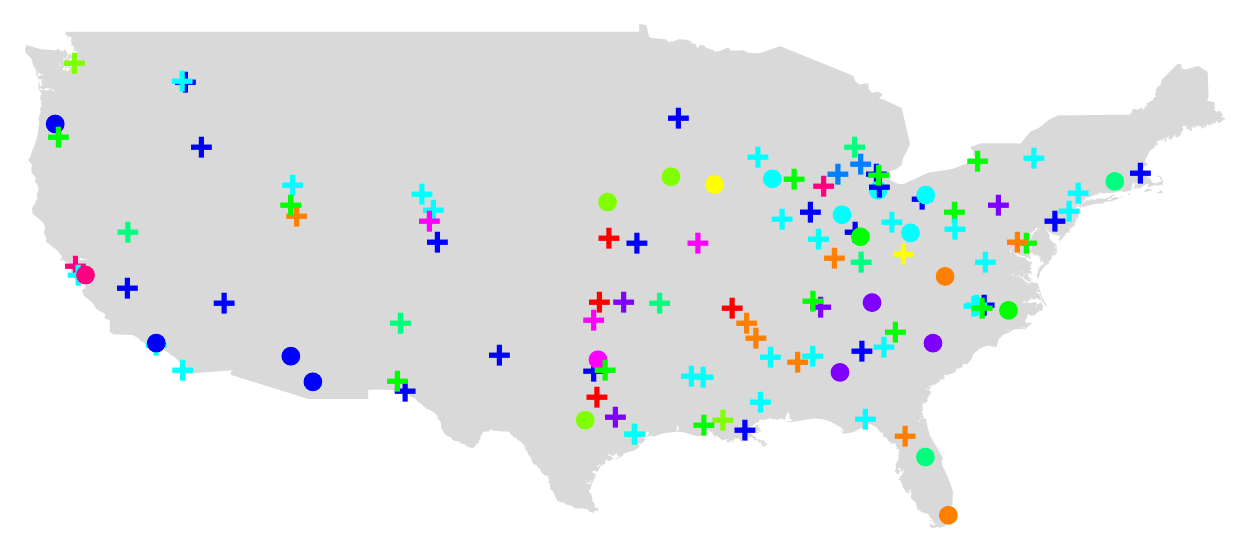}
			&
			\includegraphics[width=#1\textwidth]{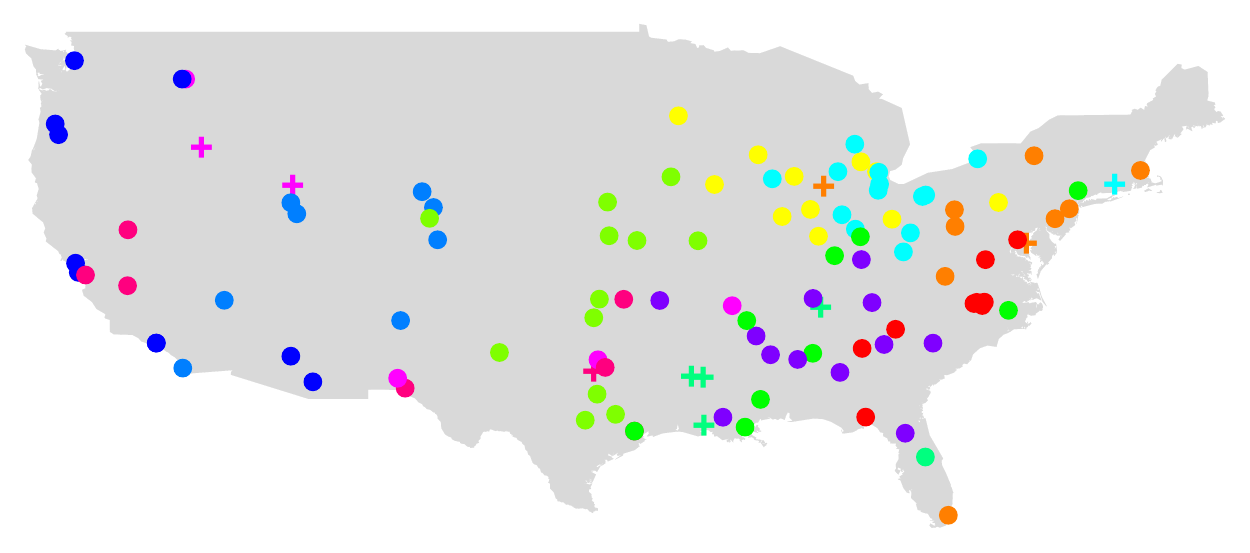}
			\\[0.8em]

			\footnotesize \textbf{SAF, } \boldmath $s=0$ & \footnotesize \textbf{SAF, } \boldmath $s=1$ \\[0.2em]
			\includegraphics[width=#1\textwidth]{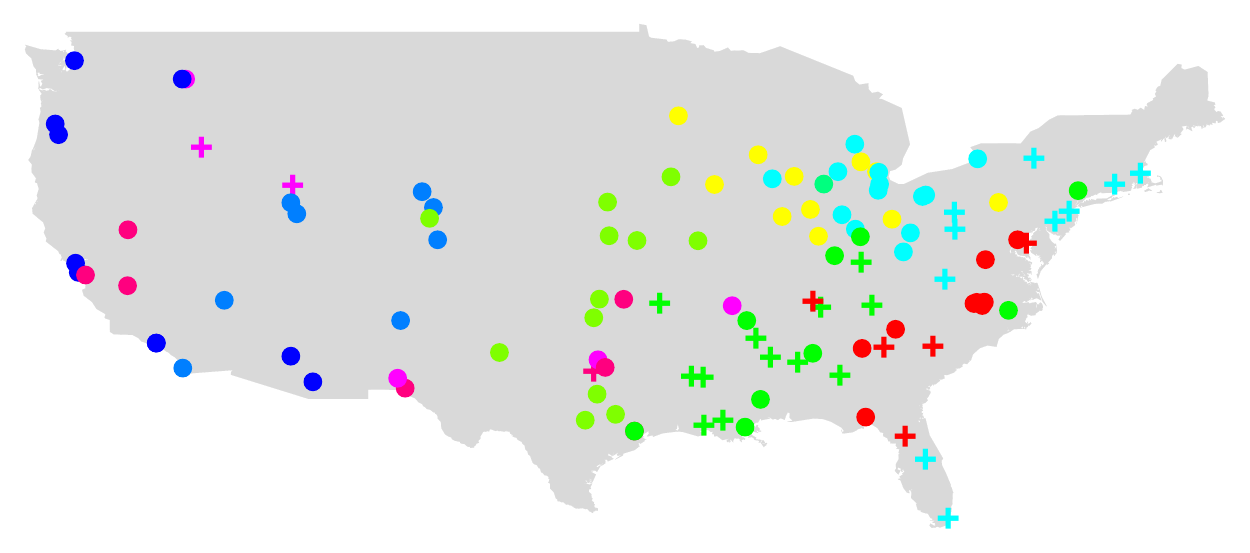}
			&
			\includegraphics[width=#1\textwidth]{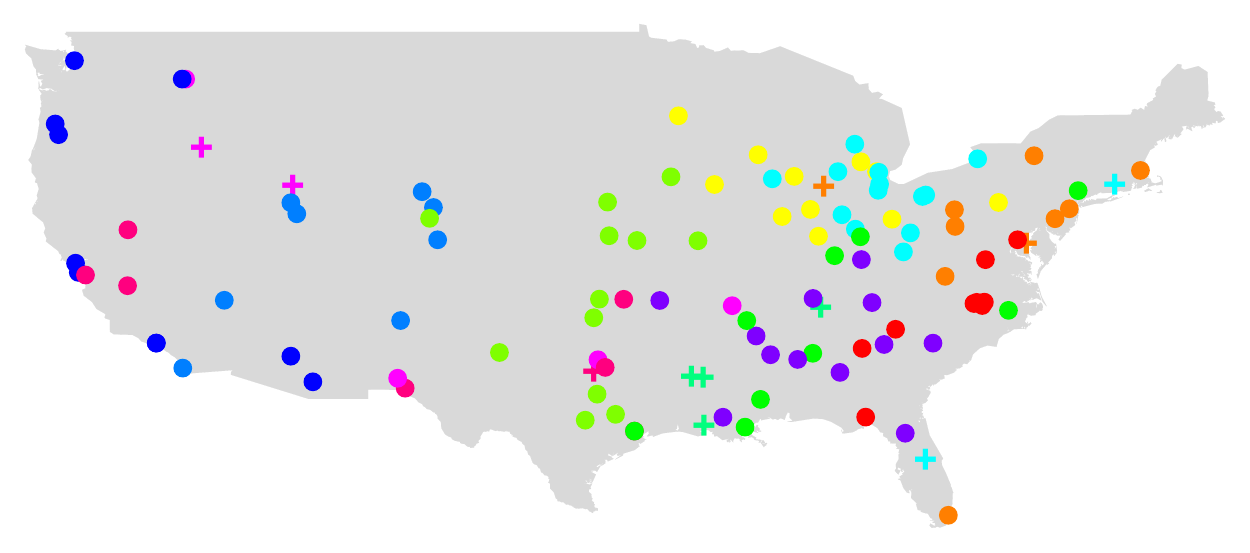}

		\end{tabular}
\end{centering}
}

\newcommand{\showExpInput}[2]{
	\begin{centering}
		\begin{tabular}{c@{\hskip 0.4em}c@{\hskip 0.4em}c}
			\footnotesize \textbf{Seastar} & \footnotesize \textbf{Fingerprint} & \footnotesize \textbf{Cactus}  \\[0.6em]
			\includegraphics[height=#1\textheight]{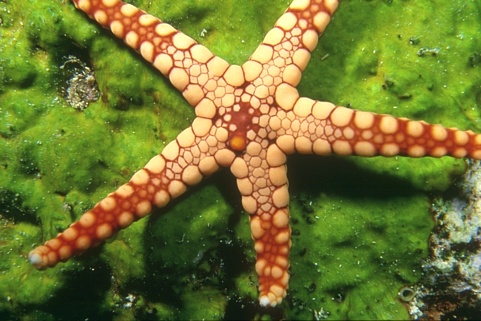}
			&
			\includegraphics[height=#1\textheight]{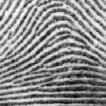}
			&
			\includegraphics[height=#2\textheight]{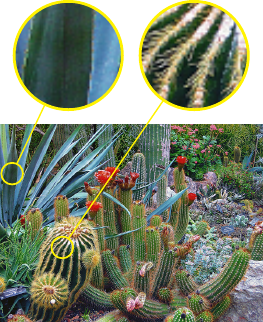}
		\end{tabular}
	\end{centering}
}

\newcommand{\showPatchDist}[1]{
	\begin{centering}
		\begin{tabular}{c@{\hskip 0.4em}c@{\hskip 0.4em}c@{\hskip 0.4em}c@{\hskip 0.4em}c@{\hskip 0.6em}c}

			& \footnotesize \textbf{Patch \qquad \qquad \, Image \, \quad} & \footnotesize \textbf{Distance} (a) & \footnotesize \textbf{Distance} (b) & \footnotesize \textbf{Distance} (Sym) &  \\
			& \footnotesize $ P_{k} $ \hspace{2.8cm} & \footnotesize $d_{\mc{F}} \big(\mc{P}(\mc{F}_{n}), P_{k} \big)$ & \footnotesize $d_{\mc{F}} \big(P_{k}, \mc{P}(\mc{F}_{n}) \big)$ & \footnotesize $d^{\textnormal{sym}}_{\mc{F}}\big(\mc{P}(\mc{F}_{n}), P_{k} \big)$ & \\[0.4em]

			&
			\includegraphics[height=#1\textheight]{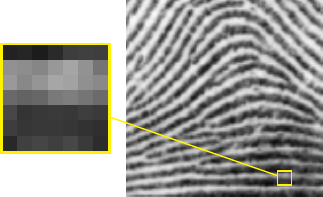}
			&
			\includegraphics[height=#1\textheight]{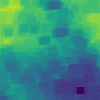}
			&
			\includegraphics[height=#1\textheight]{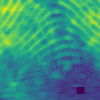}
			&
			\includegraphics[height=#1\textheight]{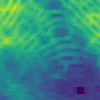}
			&
			\includegraphics[height=#1\textheight]{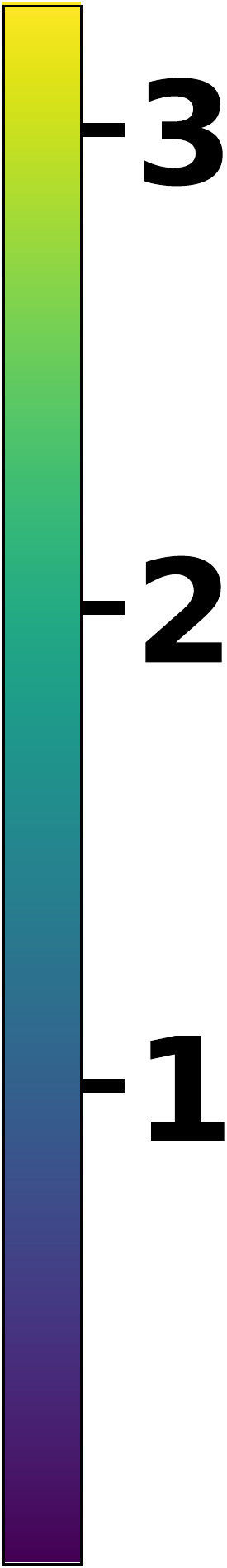}
		\end{tabular}
	\end{centering}
}

\newcommand{\showPatchFingerprintVisu}[2]{
	\begin{centering}
		\begin{tabular}{c@{\hskip 0.6em}c@{\hskip 0.4em}c@{\hskip 0.4em}c@{\hskip 0.4em}c}

			& \footnotesize \textbf{Partition} & \footnotesize \textbf{Overlay} & \footnotesize \textbf{Assignment} & \footnotesize \textbf{Difference} \\[0.4em]

			&
			\includegraphics[height=#1\textheight]{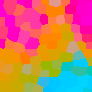}
			&
			\includegraphics[height=#1\textheight]{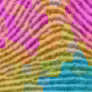}
			&
			\includegraphics[height=#1\textheight]{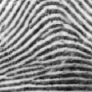}
			&
			\includegraphics[height=#1\textheight]{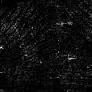}
			\\[0.6em]

			\rotatebox{90}{\qquad \quad \; \footnotesize \textbf{patches} \boldmath $\, \mc{F}_{\ast}$}
			&
			\multicolumn{4}{l}
			{
				\includegraphics[width=#2\textwidth]{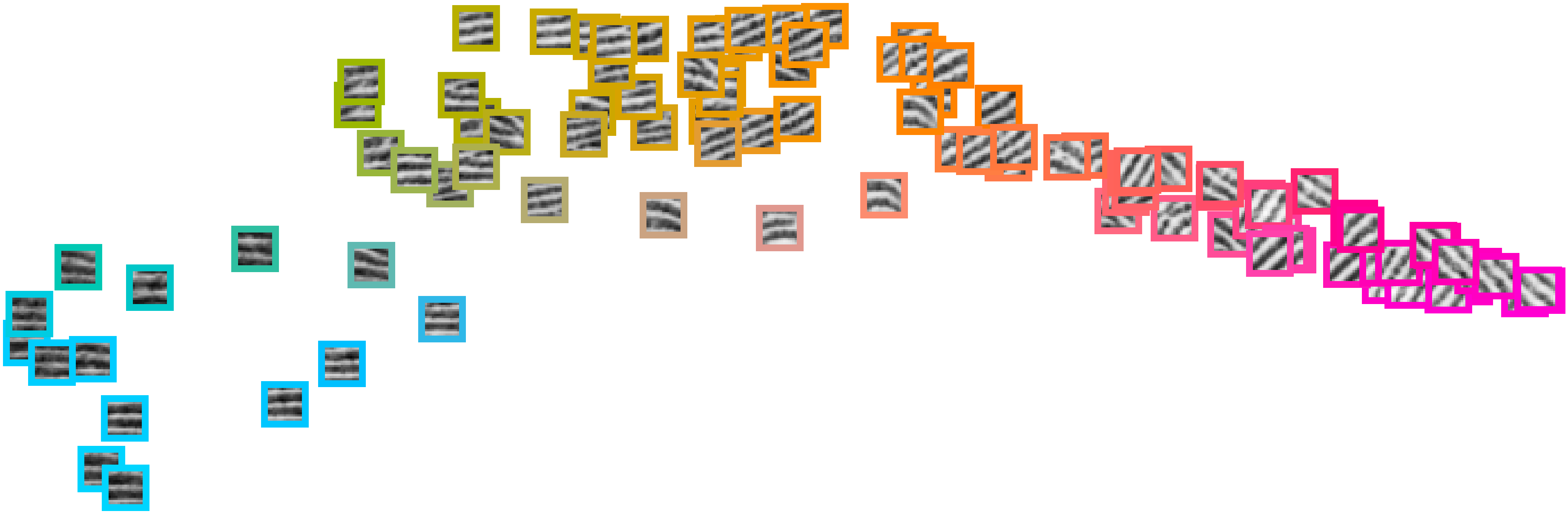}
			}

		\end{tabular}
	\end{centering}
}

\newcommand{\showNoisyCells}[2]{
	\begin{centering}
		\begin{tabular}{c@{\hskip 0.6em}c@{\hskip 0.6em}c@{\hskip 0.6em}c}

			 \footnotesize \textbf{Noisy Input} & \footnotesize \textbf{Assignment Flow} & \footnotesize \textbf{Self-Assignment Flow} & \footnotesize \textbf{Ground Truth} \\[-0.2em]
			 & \footnotesize {(Supervised)} & \footnotesize {(Unsupervised)} & \\[0.6em]

			\includegraphics[height=#1\textheight]{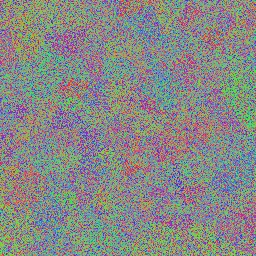}
			&
			\includegraphics[height=#1\textheight]{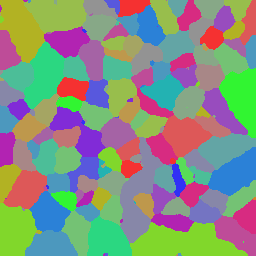}
			&
			\includegraphics[height=#1\textheight]{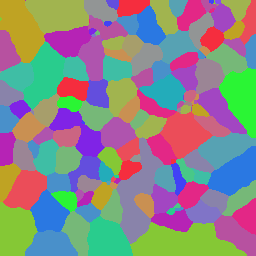}
			&
			\includegraphics[height=#1\textheight]{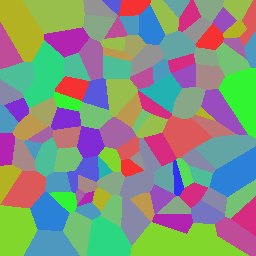}
			\\[0.6em]

			\multicolumn{1}{r}{\footnotesize \textbf{Learned Prototypes}} 
			&
			\multicolumn{3}{@{}l}
			{
				\includegraphics[width=#2\textwidth]{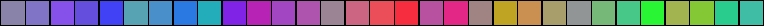}
			}\\[0.4em]

			\multicolumn{1}{r}{\footnotesize \textbf{Ground Truth}} 
			&
			\multicolumn{3}{@{}l}
			{
				\includegraphics[width=#2\textwidth]{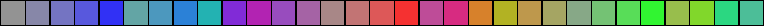}
			}
		\end{tabular}
	\end{centering}
}

\newcommand{\showAFDiagram}[1]{
\begin{centering}
  \begin{tikzpicture}[scale=#1]
  \clip (-1.6,-1) rectangle (7,3.5);
  \draw[ultra thick,loosely dotted] (2.5,3) -- (2.5,2.35);
  \draw[ultra thick,loosely dotted] (2.5,1.9) -- (2.5,-0.95);
  \node[rectangle, rounded corners] (L) at (4,2) {$L(W)$};
  \node[rectangle, rounded corners] (S) at (5,1) {$S(W)$};
  \node[rectangle, rounded corners] (W) at (3,1) {$W(t)$};
  \node[rectangle, rounded corners] (dotW) at (4,0) {$\dot{W} = R_W S(W)$};
  \node[rectangle, rounded corners] (D) at (1,2) {$D_{\mc{F}}$};
  \node[rectangle, rounded corners] (F) at (1,0) {$\mc{F}_{n}$};
  \node[rectangle, rounded corners] (G) at (0,1) {$\mc{F}_{\ast}$};
  \node (textD) at (1,2.75) [text width=5em,align=center, font=\linespread{0.8}\selectfont] {distance matrix};
  \node (textL) at (4,2.75) [text width=5em,align=center, font=\linespread{0.8}\selectfont] {likelihood matrix};
  \node (textS) at (6,1) [text width=5em,align=center, font=\linespread{0.8}\selectfont] {similarity matrix};
  \node (textG) at (-1,1) [text width=5em,align=right] {prototypes};
  \node (textF) at (0,0) [text width=5em,align=right] {data};
  \node (textM) at (5.75,2.25) [text width=5em,align=center, font=\linespread{0.8}\selectfont] {geometric averaging};
  \node (textFlow) at (6,0) {assignment flow};
  \node[rectangle, rounded corners] (feature) at (0.75,-0.75) {\large feature space $\mathcal{F}$};
  \node[rectangle, rounded corners] (assignment) at (4.5,-0.75) {\large assignment manifold $\mathcal{W}$};
  \path (G) edge[->,thick] (D);
  \path (F) edge[->,thick] (D);
  \path (D) edge[->,thick] node[anchor=south] (exp) {$\exp_W$} (L);
  \path (L) edge[->,thick] node[anchor=south west] (mean) {$\operatorname{mean}_{\mathcal{S}}$} (S);
  \path (S) edge[->,thick] (dotW);
  \path (dotW) edge[->,thick] (W);
  \path (W) edge[->,thick] (L);
  \end{tikzpicture}
\end{centering}
}

\newcommand{\showSAFDiagram}[1]{
\begin{centering}
  \begin{tikzpicture}[scale=#1]
  \clip (-1.6,-1) rectangle (7,3.5);
  \draw[ultra thick,loosely dotted] (2.5,3) -- (2.5,2.35);
  \draw[ultra thick,loosely dotted] (2.5,1.9) -- (2.5,1.4);
  \draw[ultra thick,loosely dotted] (2.5,1.1) -- (2.5,-0.95);
  \node[rectangle, rounded corners] (L) at (4,2) {$L(W)$};
  \node[rectangle, rounded corners] (S) at (5,1) {$S(W)$};
  \node[rectangle, rounded corners] (W) at (3,1) {$W(t)$};
  \node[rectangle, rounded corners] (dotW) at (4,0) {$\dot{W} = R_W S(W)$};
  \node[rectangle, rounded corners] (D) at (1,2) {$\big\la K_{\mc{F}}, A_{s}(W)\big\ra$};
  \node[rectangle, rounded corners] (F) at (1,0) {$\mc{F}_{n}$};
  \node[rectangle, rounded corners] (G) at (0,1) {$\mc{F}_{n}$};
  \node (textD) at (1,2.75) [text width=5em,align=center, font=\linespread{0.8}\selectfont] {self-assignment};
  \node (textL) at (4,2.75) [text width=5em,align=center, font=\linespread{0.8}\selectfont] {likelihood matrix};
  \node (textS) at (6,1) [text width=5em,align=center, font=\linespread{0.8}\selectfont] {similarity matrix};
  \node (textG) at (-0.75,1) [text width=5em,align=center, font=\linespread{0.8}\selectfont] {copy of data};
  \node (textF) at (0,0) [text width=5em,align=right] {data};
  \node (textM) at (5.75,2.25) [text width=5em,align=center, font=\linespread{0.8}\selectfont] {geometric averaging};
  \node (textFlow) at (6,0) {assignment flow};
  \node[rectangle, rounded corners] (feature) at (0.75,-0.75) {\large feature space $\mathcal{F}$};
  \node[rectangle, rounded corners] (assignment) at (4.5,-0.75) {\large assignment manifold $\mathcal{W}$};
  \path (G) edge[->,thick] (D);
  \path (F) edge[->,thick] (D);
  \path (W) edge[->,thick] (D);
  \path (D) edge[->,thick] node[anchor=south] (exp) {$\exp_{W} \circ \ \partial$} (L);
  \path (L) edge[->,thick] node[anchor=south west] (mean) {$\operatorname{mean}_{\mathcal{S}}$} (S);
  \path (S) edge[->,thick] (dotW);
  \path (dotW) edge[->,thick] (W);
  \path (W) edge[->,thick] (L);
  \end{tikzpicture}
\end{centering}
}

\newcommand{\showNewPatchExpLine}[2]{
        \includegraphics[width=#1\textwidth]{{./Figures/#2}.png}
        &
        \includegraphics[width=#1\textwidth]{{./Figures/Patches/#2/self_assignment/Assignment_self_assignment_K=20_s=1.0_p=7}.png}
        &
        \includegraphics[width=#1\textwidth]{{./Figures/Patches/#2/self_assignment/Assignment_self_assignment_K=60_s=1.0_p=7}.png}
        &
        \includegraphics[width=#1\textwidth]{{./Figures/Patches/#2/self_assignment/Assignment_self_assignment_K=100_s=1.0_p=7}.png}
}

\newcommand{\showNewPatchExp}[1]{
	\begin{centering}
		\begin{tabular}{c@{\hskip 0.4em}c@{\hskip 0.4em}c@{\hskip 0.4em}c}

			\footnotesize \textbf{Input} & \footnotesize \boldmath $c=20$ & \footnotesize \boldmath $c=60$ & \footnotesize \boldmath $c=100$ \\[0.6em]

            \showNewPatchExpLine{#1}{225022} \\[0.6em]

            \showNewPatchExpLine{#1}{24004} \\[0.6em]

            \showNewPatchExpLine{#1}{asian} \\[0.6em]

            \showNewPatchExpLine{#1}{aerial} \\[0.6em]

            \showNewPatchExpLine{#1}{object0024.view01}
		\end{tabular}
	\end{centering}
}

\title[Self-Assignment Flows]{Self-Assignment Flows \\for Unsupervised Data Labeling on Graphs}
\author[M.~Zisler, A.~Zern, S.~Petra, C.~Schn\"orr]{Matthias Zisler, Artjom Zern, Stefania Petra, Christoph Schn\"orr}
\address[M.~Zisler]{Image and Pattern Analysis Group, Heidelberg University, Germany}
\email{zisler@math.uni-heidelberg.de}
\address[A.~Zern]{Image and Pattern Analysis Group, Heidelberg University, Germany}
\email{artjom.zern@iwr.uni-heidelberg.de}
\address[S.~Petra]{Mathematical Imaging Group, Heidelberg University, Germany}
\email{petra@math.uni-heidelberg.de}
\urladdr{\url{https://www.stpetra.com}}
\address[C.~Schn\"{o}rr]{Image and Pattern Analysis Group, Heidelberg University, Germany}
\email{schnoerr@math.uni-heidelberg.de}
\urladdr{\url{https://ipa.math.uni-heidelberg.de}}
\date{}


\begin{document}

\begin{abstract}
This paper extends the recently introduced assignment flow approach for supervised image labeling to unsupervised scenarios where no labels are given. The resulting self-assignment flow takes a pairwise data affinity matrix as input data and maximizes the correlation with a low-rank matrix that is parametrized by the variables of the assignment flow,  which entails an assignment of the data to themselves  through the formation of latent labels (feature prototypes). A single user parameter, the neighborhood size for the geometric regularization of assignments, drives the entire process. By smooth geodesic interpolation between different  normalizations of self-assignment matrices on the positive definite matrix manifold, a one-parameter family of self-assignment flows is defined. Accordingly, our approach can be characterized from different viewpoints, e.g.~as performing spatially regularized, rank-constrained  discrete optimal transport, or as computing spatially regularized normalized spectral cuts. Regarding combinatorial optimization, our approach successfully determines  completely positive factorizations of self-assignments in large-scale scenarios, subject to spatial regularization. Various experiments including the unsupervised learning of patch dictionaries using a locally invariant distance function, illustrate the properties of the approach.
\end{abstract}

\maketitle
\tableofcontents


\section{Introduction}
\subsection{Overview, contribution.}
Assignment flows \cite{Astrom:2017ac} correspond to a smooth dynamical system for contextual data labeling (classification) on an arbitrary given graph. The basic \textit{supervised} setting assumes a set of prototypes to be given, that are assigned to the data by numerically computing the flow. `Contextual' means that decisions within local neighborhoods affect each other and are taken into account.

Assignment flows are defined using information geometry \cite{Lauritzen:1987aa,Amari:2000aa}. An elementary statistical manifold provides both a target space for data embedding and a state space on which the assignment flow evolves. Corresponding vector fields are parametrized and thus enable to learn the adaptivity of contextual label assignments, rather than parameters of a fixed regularizer as with traditional graphical models of variational approaches to inverse problems. Modular compositional design facilitates extensions beyond the basic supervised scenario, including those investigated in the present paper. Smoothness enables the design of efficient algorithms using geometric integration \cite{Zeilmann:2018aa}. The assignment flow for supervised labeling is specified in \cref{sec:assignments} and sketched by \cref{fig:AF-supervised}. Well-posedness and stability of the assignment flow was established by \cite{Zern:2020aa} under suitable conditions, that we assume to hold throughout in this paper. For further discussion and a review of our recent work in a broader context, we refer to \cite{Schnorr:2019aa}.

\begin{figure}
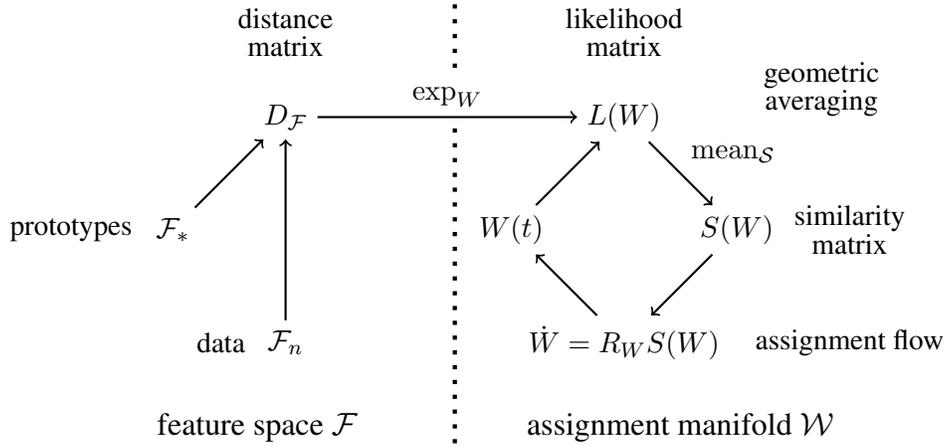

  \begin{center}
  \showAFDiagram{1.5}
  \end{center}
\caption{\textbf{Assignment flow for supervised data labeling.} Features (observed data) and representative features (labels) in a metric space represent the application domain on the left-hand side. The process of inference of label assignments to data is sketched on the right-hand side.
A vector field of distances $D_{\mc{F}} = \{D_{\mc{F};i}\}_{i \in \mc{I}}$ between a feature observed at vertex $i \in \mc{I}$ of the underlying graph and all labels constitutes the input data. Assignments of labels to data are represented by a field of assignment vectors $W = \{W_{i}\}_{i \in \mc{I}}$, regarded as a point on the assignment manifold. Local assignments $L(W)$, in terms of data $D$ that are lifted to the assignment manifold, are regularized by geometric averaging within local neighborhoods. The regularized assignments $S(W)$ parametrize a vector field $R_{W} S(W)$ that generates the flow $W(t)$ on the assignment manifold. The depicted dependencies on $W(t)$ are resolved by geometric numerical integration of the assignment flow until $W(t)$ reaches an integral (unambiguous) label assignment.}
\label{fig:AF-supervised}
\end{figure}
\begin{figure}
  \begin{center}
  \showSAFDiagram{1.5}
  \end{center}
\caption{\textbf{Self-assignment flows for unsupervised data labeling.} Representative features (labels) of supervised scenarios -- cf. \cref{fig:AF-supervised} -- are not available and replaced by a copy of given data. Applying any pairwise similarity measure yields the matrix $K_{\mc{F}}$. The gradient of the inner product of $K_{\mc{F}}$ with the self-assignment matrix $A_{s}(W)$ replaces the distance matrix $D_{\mc{F}}$ of the supervised case. The remaining ingredients of the assignment flow approach remain unchanged. Geometric integration of the self-assignment flow partitions the underlying graph and generates a distribution that defines for each component a label (class representative) as weighted Riemannian mean of the corresponding data.}
\label{fig:SAF}
\end{figure}

The availability of prototypes as class representatives is a strong requirement in practice. In many applications either prototypes are not available or it is not clear what prototypes represent the classes properly. A basic remedy is to cluster the data in a preprocessing step. However, the clustering step then does not take into account the framework in which the resulting prototypes are subsequently used for classification. In our recent work \cite{Zern:2019aa}, we took a step towards a more natural approach: the assignment flow for supervised classification was extended so as to enable the \textit{adaption} of prespecified prototypes. While this adaption is based on the \textit{same} framework that is used for subsequent contextual classification, some initial prototypes still have to be given.

In this paper, we adopt a \textit{completely unsupervised} scenario where no prototypes are given at all, cf. \cref{fig:SAF}. Data are merely given in terms of pairwise distances or affinity values forming a distance or affinity matrix. This includes the basic scenarios of pattern recognition and machine learning: distances between Euclidean feature vectors, Riemannian distances between manifold-valued features, and kernel matrices after embedding given feature vectors into reproducing kernel Hilbert space (RKHS) \cite{Hofmann:2008aa}. Our approach utilizes various relaxations of a graph partitioning problem that naturally arises when the missing prototypes of the supervised setting are removed and replaced by a copy of the given data, from which prototypes have to be learned from scratch. The relaxations involve variants of corresponding self-assignment matrices that are parametrized by the assignment flow. A key parameter is the \textit{scale} of the supervised assignment flow in terms of the size of local neighborhoods where evolving assignments driven by the flow affect each other. This parameter determines how fine or coarse the resulting partition is, and how many corresponding prototypes can be recovered under the additional assumption: The metric feature space $\mc{F}$ actually is a Riemannian manifold and determining weighted Riemannian means is computationally feasible.

A key property of our approach is that \textit{no bias} affects the emergence of these prototypes, and that the \textit{very same} framework is used for \textit{both learning} these prototypes \textit{and} subsequent contextual data \textit{labeling} (classification). In addition, as a comparison of \cref{fig:AF-supervised} and \cref{fig:SAF} shows, a \textit{single} component of the supervised assignment flow has only to be modified in order to extend this approach to the completely unsupervised setting. In particular, geometric schemes for numerically integrating the assignment flow \cite{Zeilmann:2018aa} still apply. \Cref{fig:31-colors} illustrates the application of our approach to a scenario adopted from \cite[Figure 6]{Astrom:2017ac}.

\begin{figure}
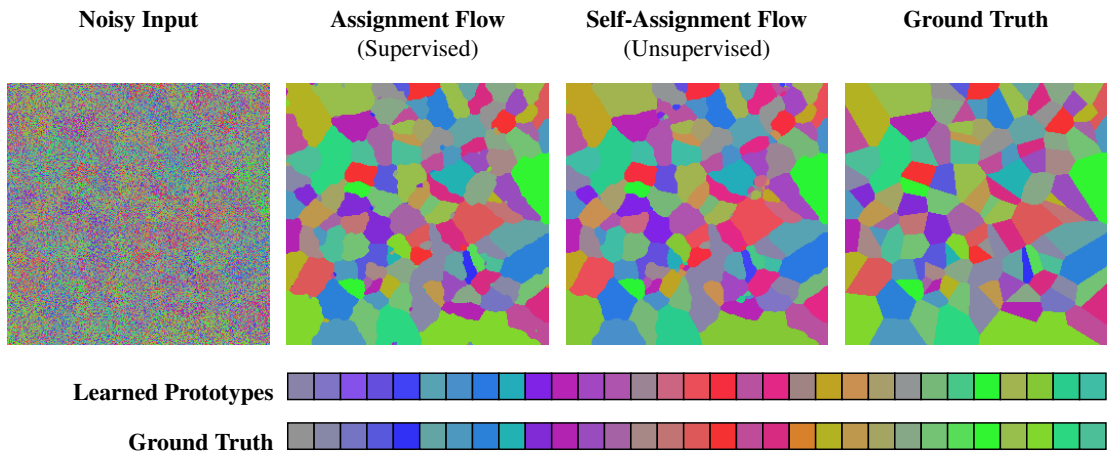

	\begin{center}
		\showNoisyCells{0.16}{0.699}
	\end{center}
\caption{\textbf{Supervised label assignment vs.~labeling through unsupervised self-assignment.} \textsc{left:} Noisy input data. \textsc{center left:} Supervised labeling using the assignment flow approach of \cite{Astrom:2017ac} and assuming the ground truth labels to be known. \textsc{center right:} Application of the approach presented in this paper. Assuming no labels to be given, solving for the self-assignment flow both partitions the input data and determines prototypes (labels) for each component. Since averaging out the strong noise of the input data during the evolution of self-assignments reduces the range of the emerging label vector components, they are depicted after rescaling them to the original color range, to enable better visual comparison to ground truth.}
\label{fig:31-colors}
\end{figure}

\subsection{Related work}
The literature on clustering and unsupervised learning is vast. No attempt is made to review it here. We confine ourselves in \cref{sec:Related-Work} to elucidating common and different aspects of our approach from three different viewpoints that have become prominent in the literature: (i) spectral relaxation and clustering using normalized graph Laplacians \cite{Shi2000,Luxburg:2007aa}; (ii) regularized transport of discrete probability measures \cite{Burkhard:1999aa,Peyre:2018aa}; (iii) matrix factorization and aspects of combinatorial optimization \cite{Rendl:1995aa,Zass:2005aa,Kuang2015,Yang2016}. From each viewpoint, our approach can be characterized as combining tight relaxation of graph partitioning, geometric spatial regularization of assignments, and geometric numerical integration in a mathematically novel way. The present paper considerably elaborates the conference version \cite{Zisler:2019aa}.

\subsection{Organization} We introduce basic notation and collect background in \cref{sec:Preliminaries}, including the supervised assignment flow as basic framework. \Cref{sec:selfAssignments} shows how the graph partitioning problem and various relaxations emerge within this framework, after replacing the prototypes by the data and assigning them to themselves. We highlight differences between two major relaxations in terms of two specific instances of the one-parameter family of self-assignment matrices (\cref{def:SelfAffinityMatrix,def:SelfInfluenceMatrix,def:SelfAssignmentMatrix}) and show how latent prototypes emerge as the assignment flow evolves. Informally speaking, these relaxations differ with regard to the sensitivity of the approach to the spatial structure and to the values of given data, respectively. See \cref{fig:Seastar}, first column ($s=0$) and last column ($s=1$), for illustrative numerical results and \cref{sec:Comparison-A-0-1} for a comparison from the mathematical point of view. After terminating the self-assignment flow at some labeling, these prototypes can be recovered explicitly under an additional assumption: weighted averaging in feature space has to be well-defined and computationally feasible (\cref{sec:Recovery-Prototypes}). A family of self-assignment flows, based on the relaxations of \cref{sec:selfAssignments}, is defined in \cref{sec:selfAssignmentSection}. It is shown that the latent prototypes minimize within-class variation and maximize cluster separability simultaneously. In this sense, the self-assignment flow consistently performs \textit{self-supervision}. Related work is discussed in \cref{sec:Related-Work}.

The approach is illustrated in \cref{sec:Experiments} using various basic examples of image analysis and more advanced examples, including  unsupervised and locally invariant patch learning, assignment, and transfer to novel data. In order to highlight the broad applicability of our approach, an experiment using weighted graph data is included, too.

\section{Preliminaries}\label{sec:Preliminaries}

We collect in this section material required in the remainder of this paper. We briefly mention why and where these concepts will be used at the beginning of each subsection.

\subsection{Basic Notation}
We set $[n]=\{1,2,\dotsc,n\}$ for $n \in \N$ and $\eins_{n} = (1,1,\dotsc,1)^{\T} \in \R^{n}$. The cardinality of a finite set $S$ is denoted by $|S|$. The following spaces of matrices will be used.
\begin{itemize}
 \item $\mb{S}^{n}$: symmetric $n\times n$ matrices
 \item $\mb{S}^{n}_{+}$: symmetric nonnegative $n\times n$ matrices
 \item $\mb{R}_{+}^{n \times c}$: nonnegative $n\times c$ matrices
 \item $\mc{P}^{n}$: symmetric positive definite $n\times n$ matrices
\end{itemize}
$\|\cdot\|$ denotes the Euclidean norm and the Frobenius norm for vectors and matrices, respectively. All other norms will be indicated by a corresponding subscript. For a matrix $A \in \R^{n \times c}$, $A_{i},\,i \in [n]$ denote the row vectors and $A^{j},\,j \in [c]$ denote the column vectors, $A^{\T} \in \R^{c \times n}$ the transpose and $A^{\dagger}$ the Moore-Penrose generalized inverse of $A$. $\tr(A)=\sum_{i \in [n]}A_{i,i}$ denotes the trace of a square matrix $A \in \R^{n \times n}$.
\begin{equation}\label{eq:def-simplex}
\Delta_{n} = \{p \in \R_{+}^{n} \colon \la\eins_{n},p\ra=1\}
\end{equation}
denotes the probability simplex. The orthogonal projection onto a closed convex set $C$ is denoted by $\Pi_{C}$. For a differentiable function $E \colon \R^{n} \to \R$, its ordinary gradient is denoted by $\partial E = (\partial_{1}E,\dotsc,\partial_{n}E)^{\T}$.

For strictly positive vectors $p > 0$, we efficiently denote componentwise subdivision by $\frac{v}{p}$. Likewise, we set $p v = (p_{1} v_{1},\dotsc, p_{n} v_{n})^{\T}$. The exponential \textit{function} applies componentwise to vectors (and similarly for $\log$) and will always be denoted by $e^{v}=(e^{v_{1}},\dotsc,e^{v_{n}})^{\T}$, in order not to confuse it with the exponential \textit{maps} \eqref{eq:schnoerr-eq:exp-maps}.

\subsection{Scatter Matrices}
\label{sec:scatter-matrices}
A major objective of this paper is to determine prototypical features (labels) $f_{j}^{\ast} \in \mc{F},\,j \in \mc{J}$ as weighted Riemannian means in an \textit{unsupervised} manner, as a function of the self-assignment flow $W(t)$ (\cref{sec:Recovery-Prototypes}). Thus, these prototypes may be considered as latent variables that emerge and evolve as a function of the self-assignment flow.

In order to substantiate this approach, we adopt in \cref{sec:Self-Supervision} the simplest case of a Euclidean feature space $\mc{F}$ where the prototypes can be determined in closed form. Then we show that the prototypes are \textit{consistently} determined -- we call this property: \textit{self-supervision} -- as if the partition and class assignments of the data were known beforehand, like in \textit{supervised} scenarios.

The following basic concepts of statistical pattern recognition \cite{Devyver:1982aa} will be used in \cref{sec:Self-Supervision}.
Let $\mc{F}_{n} = \{f_{i} \in \R^{d},\, i \in \mc{I}\}$
denote given data in terms of feature vectors in a Euclidean space. Suppose these data are classified corresponding to a partition $\mc{I} = \dotcup_{j \in [c]} \mc{I}_{j},\, j \in \mc{J}$, that is datum $f_{i}$ belongs to class $j$ iff $i \in \mc{I}_{j}$. Then the variation of these data in terms of their first- and second-order empirical moments can be decomposed and represented by the \textit{scatter matrices} defined in \cref{app:scatter-matrices}, that satisfy the equation

\begin{equation}\label{eq:St-decomposition}
S_{t} = S_{w} + S_{b}.
\end{equation}
In \textit{supervised} scenarios the class-label assignments $i \in \mc{I}_{j}$ are \textit{known} and hence the decomposition \eqref{eq:St-decomposition} can be computed. Assuming $S_{w}$ has full rank, a basic objective for dimension reduction by extracting lower-dimensional features from the data $\mc{F}_{n}$ is given by the \textit{class-separability measure}
\begin{equation}\label{eq:Sw-1-Sb}
\tr(S_{w}^{-1} S_{b}).
\end{equation}
Defining the features by $Y^{\T} f_{i},\, i \in \mc{I}$, for some matrix $Y \in \R^{d \times c}$ to be determined, transforms \eqref{eq:Sw-1-Sb} to $\tr((Y^{\T} S_{w}Y)^{-1} Y^{\T} S_{b} Y)$. Maximizing this objective with respect to $Y$ \textit{simultaneously} maximizes the between-class variation and minimizes the within-class variation \cite{Devyver:1982aa}.

Our viewpoint in this paper differs, however. Since we assume \textit{unlabelled} (unclassified) data, the decomposition \eqref{eq:St-decomposition} is \textit{unknown}: The class separability measure depends on the prototypes $f_{j}^{\ast},\, j \in \mc{J}$ which in turn depend on the assignments of class labels to data, that are determined by the self-assignment flow $W(t)$ (cf. \cref{fig:AF-supervised}).

Accordingly, we are interested in the quality of the \textit{latent prototypes} in terms of \eqref{eq:Sw-1-Sb}, as a function of the self-assignment flow -- see \cref{sec:Self-Supervision}.

\subsection{Sketching Large Affinity Matrices}\label{sec:sketching}
In order to cope with large-scale scenarios, we will have to compress large symmetric and positive semi-definite matrices $K \in \mb{S}^{n}$. The problem is to obtain a computationally feasible approximation of the best rank-$\ell$ approximation
\begin{equation}\label{eq:def-K-ell}
K_{\ell} = U_{1} D_{\ell}(K) U_{1}^{\T},\qquad
\ell\ll n,
\end{equation}
where $D_{\ell}$ and $U_{1} \in \R^{n\times \ell}$ contain the dominant eigenvalues and eigenvectors of the spectral decomposition $K = U D(K) U^{\T}$. Computing \eqref{eq:def-K-ell} directly for large $n$ using the Singular Value Decomposition (SVD) is too expensive.
Computationally feasible approximations \cite{Gittens:2016aa} result in the \textit{compressed matrix}
\begin{subequations}\label{eq:K-compressed}
\begin{gather}\label{eq:K-compressed-a}
\widehat K_{\ell} = C A^{\dagger} C^{\T}
\intertext{
that is parametrized by a \textit{sketching matrix} $S \in \R^{n \times \ell}$ with
}\label{eq:K-compressed-b}
C = K^{q} S,\qquad
A = S^{\T} K^{2 q-1} S,\qquad q \in \N
\end{gather}
\end{subequations}
and hence has rank at most $\ell$. $A^{\dagger}$ is the Moore-Penrose generalized inverse of $A$ and $q \in \{1,2,3\}$ is a small integer in practice. Choosing $q > 1$ is more expensive due to the multiplication of the large matrix $K$ of \eqref{eq:K-compressed-b} but yields in theory a better approximation of \eqref{eq:def-K-ell} by \eqref{eq:K-compressed-a} with respect to the spectral norm.

In this paper, we confine ourselves to the following computationally cheap version of this method for computing \eqref{eq:K-compressed-a}, based on \textit{uniform sampling} of $\ell$ columns directly from $K$.
Assuming w.l.o.g.~that they form the \text{first} $\ell$ columns of $K$, the corresponding partition $[n] = [\ell] \cup \big([n] \setminus [\ell] \big) $  and $S = \bsm I_{\ell} \\ 0 \esm$ yields with $q=1$
\begin{equation}
K = \bpm A & B_{1} \\ B_{1} & B_{2} \epm,\qquad
C = \bpm A \\ B_{1} \epm,
\end{equation}
and using $A A^{\dagger} A = A$,
\begin{equation}
\widehat K_{\ell}
= \bpm A \\ B_{1} \epm A^{\dagger} \bpm A & B_{1} \epm
= \bpm A  & A A^{\dagger} B_{1} \\ B_{1} A^{\dagger} A & B_{1} A^{\dagger} B_{1} \epm.
\end{equation}
Assuming the $A$ has full rank, we obtain the classical \textit{Nystr\"{o}m extension}
\begin{equation}\label{eq:Nystroem-K-compressed}
\widehat K_{\ell}
= \bpm A  & B_{1} \\ B_{1} & B_{1} A^{-1} B_{1} \epm
\end{equation}
introduced in machine learning by \cite{Williams:2001aa}, studied much earlier in linear algebra -- see, e.g., the Schur compression matrix and references in \cite{Ando:1979aa} -- and analyzed by \cite{Drineas:2005aa}.

\subsection{The Manifold $\mc{P}^{n}$ of Positive Definite Symmetric Matrices}
\label{sec:Pn}
The following is taken from \cite{Bhatia:2006aa}.
The set
\begin{equation}
\mc{P}^{n} = \{S \in \mb{S}^{n} \colon \lambda_{i}(S) > 0,\,\forall i \in [n]\}
\end{equation}
of symmetric and positive definite matrices form a smooth Riemannian manifold with tangent spaces $T_{S}\mc{P}^{n} \cong \mb{S}^{n}$ identified with $\mb{S}^{n}$ and Riemannian metric
\begin{subequations}\label{eq:def-Pc-tangents}
\begin{align}\label{eq:def-Pc-tangents-a}
\la S_{1},S_{2}\ra_{S}
&= \tr(S^{-1}S_{1}S^{-1}S_{2}),\qquad
S_{1},S_{2} \in \mb{S}^{n},\quad S \in \mc{P}^{n}
\intertext{and corresponding norm}\label{eq:def-Pc-tangents-b}
\|T\|_{S} &= \|S^{-1/2} T S^{-1/2}\|,\qquad T \in \mb{S}^{n},\quad S \in \mc{P}^{n}.
\end{align}
\end{subequations}
For any $A, B \in \mc{P}^{n}$, there exists a unique geodesic joining $A$ and $B$ given by
\begin{equation}\label{eq:def-gamma-P}
\gamma(s) = A^{1/2}\big(A^{-1/2} B A^{-1/2}\big)^{s} A^{1/2},\qquad s \in [0,1].
\end{equation}

\subsection{Representation of Assignments}
\label{sec:assignments}
This section describes the assignment flow, which is a basic dynamical system for labeling data given on a graph \cite{Astrom:2017ac} in supervised scenarios (\cref{fig:AF-supervised}). \Cref{sec:assignmentManifold} summarizes the mathematical background. We refer to \cite{Schnorr:2019aa} for a more elaborate exposition and discussion. \Cref{sec:AF-supervised} explains the details shown by \cref{fig:AF-supervised}.

\subsubsection{Assignment Manifold}\label{sec:assignmentManifold}
Let $(\mc{F},d_{\mc{F}})$ be a metric space and
\begin{equation}\label{eq:def-mcF-n}
\mc{F}_{n} = \{f_{i} \in \mc{F} \colon i \in \mc{I}\},\qquad |\mc{I}|=n.
\end{equation}
given data. Assume that a predefined set of \textbf{prototypes (labels)}
\begin{equation}\label{eq:def-mcF-ast}
\mc{F}_{\ast} = \{f^{\ast}_{j} \in \mc{F} \colon j \in \mc{J}\},\qquad |\mc{J}|=c.
\end{equation}
is given. \textit{Data labeling} denotes the assignments
\begin{equation}\label{eq:fj-fi}
\mc{I} \to \mc{F}_{\ast},\qquad i \mapsto f_{j_{i}}^{\ast}
\end{equation}
of a single prototype $f_{j}^{\ast} \in \mc{F}_{\ast}$ to each data point $f_{i} \in \mc{F}_{n}$.
The set $\mc{I}$ is assumed to form the vertex set of an  undirected graph $\mc{G}=(\mc{I},\mc{E})$ which defines a relation $\mc{E} \subset \mc{I} \times \mc{I}$ and neighborhoods
\begin{equation}\label{eq:def-Ni}
\mc{N}_{i} = \{k \in \mc{I} \colon ik \in \mc{E}\} \cup \{i\},
\end{equation}
where $ik$ is a shorthand for the unordered pair (edge) $(i,k)=(k,i)$.

The assignments (labeling) \eqref{eq:fj-fi} are represented by matrices in the set
\begin{equation}\label{eq:def-W-ast}
\mc{W}_{\ast}^{c} = \big\{W \in \{0,1\}^{n \times c} \colon W\eins_{c}=\eins_{n}, \, \rank(W)=c \big\}
\end{equation}
with unit vectors $W_{i},\,i \in \mc{I}$, called \textbf{assignment vectors}, as row vectors. Moreover the rank constraint ensures that exactly $c$ labels are assigned. These assignment vectors are computed by numerically integrating the assignment flow below \eqref{eq:assignment-flow}, in the following elementary geometric setting. The integrality constraint and the rank constraint of \eqref{eq:def-W-ast} are relaxed and vectors
\begin{equation}\label{eq:def-Wi}
W_{i} = (W_{i,1},\dotsc,W_{i,c})^{\T} \in \mc{S},\quad i \in \mc{I},
\end{equation}
that are discrete probability measures on the set of labels indexed by $\mc{J}$, but still called \textbf{assignment vectors}. $W_{i},\,i \in \mc{I}$
are points
on the Riemannian manifold (recall \eqref{eq:def-simplex})
\begin{equation}\label{eq:def-S}
(\mc{S},g),\qquad
\mc{S} = \{p \in \Delta_{c} \colon p > 0\}
\end{equation}
with
\begin{equation}\label{eq:barycenter-S}
\eins_{\mc{S}} = \frac{1}{c}\eins \in \mc{S},
\qquad(\textbf{barycenter})
\end{equation}
tangent space
\begin{equation}\label{eq:def-T0}
T_{0}
= \{v \in \R^{c} \colon \la\eins,v \ra=0\}
\end{equation}
and tangent bundle $T\mc{S} = \mc{S} \times T_{0}$,
orthogonal projection
\begin{equation}\label{eq:def-Pi0}
\Pi_{0} \colon \R^{c} \to T_{0},\qquad
\Pi_{0} = \Pi_{T_{0}} = I - \eins_{\mc{S}}\eins^{\T}
\end{equation}
and the Fisher-Rao metric
\begin{equation}\label{schnoerr-eq:FR-metric-S}
g_{p}(u,v) = \sum_{j \in \mc{J}} \frac{u^{j} {v}^{j}}{p^{j}},\quad p \in \mc{S},\quad
u,v \in T_{0}.
\end{equation}
Based on the linear map
\begin{equation}\label{eq:def-Rp}
R_{p} \colon \R^{c} \to T_{0},\qquad
R_{p} = \Diag(p)-p p^{\T},\qquad p \in \mc{S}
\end{equation}
satisfying
\begin{equation}\label{eq:Rp-Pi0}
R_{p} = R_{p} \Pi_{0} = \Pi_{0} R_{p},
\end{equation}
\textbf{exponential maps} and their inverses are defined as
\begin{subequations}\label{eq:schnoerr-eq:exp-maps}
\begin{align}
\Exp &\colon \mc{S} \times T_{0} \to \mc{S}, &
(p,v) &\mapsto
\Exp_{p}(v) = \frac{p e^{\frac{v}{p}}}{\la p,e^{\frac{v}{p}}\ra},
\label{eq:Exp0} \\ \label{eq:IExp0}
\Exp_{p}^{-1} &\colon \mc{S} \to T_{0}, &
q &\mapsto  \Exp_{p}^{-1}(q) = R_{p} \log\frac{q}{p},
\\
\exp_{p} &\colon T_{0} \to \mc{S}, &
\exp_{p} &= \Exp_{p} \circ R_{p},
\label{schnoerr-eq:def-exp-p} \\
\exp_{p}^{-1} &\colon \mc{S} \to T_{0}, &
\exp_{p}^{-1}(q) &= \Pi_{0}\log\frac{q}{p}.
\end{align}
\end{subequations}
\begin{remark}\label{rem:Rp}

We call the linear map \eqref{eq:def-Rp} \textit{replicator map} because it yields, for any vector field $F \colon \mc{S} \to \R^{c}$ that represents affinity measures for the set of labels \eqref{eq:def-mcF-ast}, a vector field $R_{p}F$ on $\mc{S}$ and in turn the corresponding \textit{replicator equation} \cite{Hofbauer:2003aa}
\begin{equation}
\dot p_{j} = \big(R_{p}F(p)\big)_{j}
= p_{j}\big(F_{j}(p)-\EE_{p}[F]\big)
= p_{j} F_{j}(p)-\la p,F(p)\ra p_{j},\qquad j \in\mc{J}.
\end{equation}
If $F = \partial E$ derives as Euclidean gradient from a potential $E$, then $R_{p}F(p) = \ggrad_{\mc{S}} E$ is the corresponding Riemannian gradient with respect to the Fisher-Rao metric \eqref{schnoerr-eq:FR-metric-S} \cite[Prop.~1]{Astrom:2017ac}.
\end{remark}
\begin{remark}\label{rem:Exp-map}
The map $\Exp$ corresponds to the e-connection of information geometry, rather than to the exponential map of the Riemannian connection \cite{Amari:2000aa}. Accordingly, the geodesics with respect to the affine e-connection \eqref{eq:Exp0} are not length-minimizing. But they provide a close approximation \cite[Prop.~3]{Astrom:2017ac} and are more convenient for numerical computations. In particular, all simplex constraints (normalization of assignment vectors as discrete distributions) are smoothly `built in'. Yet, unlike the geometry induced by traditional barrier functions (see, e.g., \cite{Nesterov:2002aa}), the information geometry underlying the assignment flow $W(t)$ entails that it may -- and in fact \textit{does} \cite{Zern:2020aa} -- evolve arbitrarily close to the boundary of the assignment manifold so as to determine unambigous label assignments for $t$ large enough.
\end{remark}
\begin{remark}\label{rem:exp-T0}
Applying the map $\exp_{p}$ to a vector in $\R^{c} = T_{0} \oplus \R\eins$ does not depend on the constant component of the argument, due to \eqref{eq:Rp-Pi0}.
\end{remark}

The \textbf{assignment manifold} is defined as
\begin{equation}\label{schnoerr-eq:def-mcW}
(\mc{W},g),\qquad \mc{W} = \mc{S} \times\dotsb\times \mc{S}.\qquad (n = |\mc{I}|\;\text{factors})
\end{equation}
Points $W \in \mc{W}$ are row-stochastic matrices $W \in \R^{n \times c}$ with row vectors $W_{i} \in \mc{S},\; i \in \mc{I}$ that represent the assignments \eqref{eq:fj-fi} for every $i \in \mc{I}$. We set
\begin{equation}\label{schnoerr-eq:TmcW}
\mc{T}_{0} = T_{0} \times\dotsb\times T_{0}
\qquad (n = |\mc{I}|\;\text{factors})
\end{equation}
with tangent vectors $V \in \R^{n \times c},\; V_{i} \in T_{0},\; i \in \mc{I}$. All the mappings defined above factorize in a natural way and apply row-wise, e.g.~$\Exp_{W} = (\Exp_{W_{1}},\dotsc,\Exp_{W_{n}})$ etc.

\subsubsection{Assignment Flow}
\label{sec:AF-supervised}
Based on \eqref{eq:def-mcF-n} and \eqref{eq:def-mcF-ast}, the distance vector field
\begin{equation}\label{eq:def-distance-vector}
D_{\mc{F};i} = \big(d_{\mc{F}}(f_{i},f_{1}^{\ast}),\dotsc,d_{\mc{F}}(f_{i},f_{c}^{\ast})\big)^{\T},\qquad i \in \mc{I}
\end{equation}
is well-defined. These vectors are collected as row vectors of the \textbf{distance matrix}
\begin{equation}
D_{\mc{F}} \in \mb{R}_{+}^{n \times c}.
\end{equation}
The \textbf{likelihood map} and the \textbf{likelihood vectors}, respectively, are defined as
\begin{equation}\label{schnoerr-eq:def-Li}
L_{i} \colon \mc{S} \to \mc{S},\qquad
L_{i}(W_{i})
= \exp_{W_{i}}\Big(-\frac{1}{\rho}D_{\mc{F};i}\Big)
= \frac{W_{i} e^{-\frac{1}{\rho} D_{\mc{F};i}}}{\la W_{i},e^{-\frac{1}{\rho} D_{\mc{F};i}} \ra},\qquad i \in \mc{I},
\end{equation}
where the scaling parameter $\rho > 0$ is used for normalizing the a-priori unknown scale of the components of $D_{\mc{F};i}$ that depends on the specific application at hand.

A key component of the assignment flow is the interaction of the likelihood vectors through \textit{geometric} averaging within the local neighborhoods \eqref{eq:def-Ni}. Specifically, using the weights
\begin{equation}\label{eq:weights-Omega-i}
\Omega_{i} = \Big\{w_{i,k} \colon k \in \mc{N}_{i},\; w_{i,k} > 0,\; \sum_{k \in \mc{N}_{i}} w_{i,k}=1\Big\},\quad i \in \mc{I},
\end{equation}
the \textbf{similarity map} and the \textbf{similarity vectors}, respectively, are defined as
\begin{equation}\label{eq:def-Si}
S_{i} \colon \mc{W} \to \mc{S},\qquad
S_{i}(W) = \Exp_{W_{i}}\Big(\sum_{k \in \mc{N}_{i}} w_{i,k} \Exp_{W_{i}}^{-1}\big(L_{k}(W_{k})\big)\Big),\qquad i \in \mc{I}.
\end{equation}
If $\Exp_{W_{i}}$ were the exponential map of the Riemannian (Levi-Civita) connection, then the argument inside the brackets of the right-hand side would just be the negative Riemannian gradient with respect to $W_{i}$ of the center of mass objective function comprising the points $L_{k},\,k \in \mc{N}_{i}$, i.e.~the weighted sum of the squared Riemannian distances between $W_{i}$ and  $L_{k}$   \cite[Lemma 6.9.4]{Jost:2017aa}. In view of \cref{rem:Exp-map}, this interpretation is only approximately true mathematically, but still correct informally: $S_{i}(W)$ moves $W_{i}$ towards the geometric mean of the likelihood vectors $L_{k},\,k \in \mc{N}_{i}$. Since $\Exp_{W_{i}}(0)=W_{i}$, this mean is equal to $W_{i}$ if the aforementioned gradient vanishes.

The \textbf{assignment flow} is induced by the system of nonlinear ODEs
\begin{subequations}\label{eq:assignment-flow}
\begin{align}
\dot W &= R_{W}S(W),\qquad W(0)=\eins_{\mc{W}},
\label{eq:assignment-flow-a}
\\
\label{eq:assignment-flow-b}
\dot W_{i} &= R_{W_{i}} S_{i}(W),\qquad W_{i}(0)=\eins_{\mc{S}},\quad i \in \mc{I},
\end{align}
\end{subequations}
where $\eins_{\mc{W}} \in \mc{W}$ denotes the barycenter of the assignment manifold \eqref{schnoerr-eq:def-mcW}.
System \eqref{eq:assignment-flow-a} collects all systems \eqref{eq:assignment-flow-b}, for every vertex $i \in \mc{I}$. The latter systems are coupled within local neighborhoods $\mc{N}_{i}$ due to the similarity vectors $S_{i}(W)$ given by \eqref{eq:def-Si}. The solution $W(t)\in\mc{W}$ is numerically computed by geometric integration \cite{Zeilmann:2018aa} and determines a labeling $W(T) \in \mc{W}_{\ast}^{c}$ for sufficiently large $T$ after a trivial rounding operation.

\subsection{Greedy $k$-Center Metric Clustering}
\label{sec:greedy-clustering}
In order to handle large-scale scenarios, the following simple but effective algorithm from \cite{Har-Peled:2011aa} can be employed for data reduction in a preprocessing step. The algorithm approximates the $k$-center clustering along with a \textit{performance guarantee} -- see \eqref{eq:2-approximation} below -- and only requires \textit{linear complexity} $\mc{O}(n c)$ with respect to the (large) number of data points $n$. By using a min-max objective (see \eqref{eq:metric-clustering-objective} below), selected data points are evenly spread among all data points and hence do not introduce a bias beforehand.

The task of $k$-center clustering is as follows. Given data points $\mc{F}_{n}$ from a metric space $(\mc{F},d_{\mc{F}})$, determine a subset
\begin{equation}
\mc{F}_{c} = \{f_{j} \colon j \in \mc{J}\} \subset \mc{F}_{n},\qquad |\mc{J}|=c.
\end{equation}
that solves the combinatorially hard optimization problem
\begin{equation}\label{eq:metric-clustering-objective}
E_{\infty}^{\ast} = \min_{\mc{F}_{c} \subset \mc{F}_{n}, |\mc{F}_{c}|=c} E_{\infty}(\mc{F}_{c}),\qquad
E_{\infty}(\mc{F}_{c} ) = \max_{f \in \mc{F}_{n}} d_{\mc{F}}(f,\mc{F}_{c} ),
\end{equation}
where $d_{\mc{F}}(f,\mc{F}_{c} ) = \min_{f' \in \mc{F}_{c}} d_{\mc{F}}(f,f' )$.

A greedy approximation is computed as follows. Start with a first initial point $f_{1}$, e.g.~chosen randomly in $\mc{F}_{n}$. Then select the remaining $c-1$ points $f_{2},\dotsc,f_{c}$ successively by determining the point that is most distant from the current subset of already selected points, to obtain a set $\mc{F}_{c}$ that is a $2$-approximation
\begin{equation}\label{eq:2-approximation}
E_{\infty}(\mc{F}_{c}) \leq 2 E_{\infty}^{\ast}
\end{equation}
of the optimum \eqref{eq:metric-clustering-objective} \cite[Thm.~4.3]{Har-Peled:2011aa}. As a consequence, the subset of $c$ points of $\mc{F}_{c}$ are almost uniformly distributed within $\mc{F}_{n}$, as measured by the metric $d_{\mc{F}}$.

\section{Self-Assignment}\label{sec:selfAssignments}
This section prepares the generalization of the assignment flow \eqref{eq:assignment-flow} from supervised labeling to completely unsupervised labeling, that is the transition from \cref{fig:AF-supervised} to \cref{fig:SAF}, where prototypes \eqref{eq:def-mcF-ast} no longer are involved but are determined simultaneously.

\Cref{sec:From-Labeling-To-Partitioning} introduces the objective function $\la K_{\mc{F}},A_{s}(W)\ra = \tr(K_{\mc{F}} A_{s}(W))$ for the special case of the parameter $s=0$. It is shown that the matrix $A_{0}(W)$ arises naturally, in the absence of labels, in connection with a graph partitioning problem. \Cref{sec:selfAssignmentMatrix} generalizes the approach to a one-parameter family of self-assignment matrices $A_{s}(W) = W\gamma_{s}(W)^{-1}W^{\T},\, s \in [0,1]$ that is defined by a smooth geodesic $s \mapsto \gamma_{s}(W)\in \mc{P}^{c}$ of positive definite matrices. \Cref{sec:Relaxation-Interpretation} provides additional interpretations of the `extreme cases' $A_{0}(W)$ and $A_{1}(W)$ from various viewpoints. Specifically, the entries $A_{0;i,k}(W)$ specify the probability that two vertices $i$ and $k$ get assigned the same label (no matter which one), i.e.~that they belong to the same cluster (\cref{sec:InterSelfAffinityMatrix}). The relaxation based on $A_{1}(W)$, on the other hand, focuses on the best column-subspace of the assignment matrix $W$ for self-prediction of given data (\cref{sec:InterSelfInfluenceMatrix}).

The content of \cref{sec:Relaxation-Interpretation} will be complemented in \cref{sec:Related-Work} by a discussion of related work, and continued in \cref{sec:selfAssignmentSection} by explaining the right-hand side of \cref{fig:SAF}.



\subsection{From Labeling to Partitioning}\label{sec:From-Labeling-To-Partitioning}
Since the prototypes $\mc{F}_{\ast}$ are unknown, we replace them by the given data $\mc{F}_{n}$.
Along with $\mc{F}_{n}$ and the underlying graph $\mc{G}=(\mc{I},\mc{E})$, we assume a weighted similarity matrix
\begin{equation}\label{eq:K-mcF-entries}
K_{\mc{F}} \in \mb{S}^{n},\qquad\qquad
K_{\mc{F};i,k}=(K_{\mc{F}})_{i,k} = k_{\mc{F}}(f_{i},f_{k}),\qquad i,k \in \mc{I}
\end{equation}
to be given, with entries
measuring the similarity of the data points $f_{i},f_{k}$ in terms of a symmetric function $k_{\mc{F}}$.
Matrix $K_{\mc{F}}$ is positive definite if $k_{\mc{F}}$ evaluates the inner product of a data embedding into a corresponding reproducing kernel Hilbert space (RKHS) space \cite{Hofmann:2008aa}. A basic example is a Euclidean feature space $(\mc{F},d_{\mc{F}})$ with norm $d_{\mc{F}}(f_{i},f_{k})=\|f_{i}-f_{k}\|$ and
\begin{equation}\label{eq:K-mcF-Gaussian}
k_{\mc{F}}(f_{i},f_{k}) = e^{-d_{\mc{F}}(f_{i},f_{k})^{2} /\sigma^{2}}.
\end{equation}
Let $W \in \mc{W}_{\ast}^{c}$ be a labeling. The column vectors $W^{j},\, j \in \mc{J}$, of $W$ indicate which data points $f_{i}$ are assigned to $j$-th cluster $\mc{I}_{j}$ corresponding to the partition
\begin{equation}
\mc{I} = \dot{\bigcup_{j \in \mc{J}}}\mc{I}_{j},
\qquad\qquad
n_{j} = |\mc{I}_{j}|,\quad j \in \mc{J},
\qquad\qquad
\sum_{j \in \mc{J}} n_{j}=n=|\mc{I}|
\end{equation}
of the data set $\mc{F}_{n}$. Define the diagonal matrix
\begin{equation}\label{eq:def_CW}
C(W) = \Diag(W^{\T} \eins_{n})
= \Diag(n_{1},\dotsc,n_{c}) \in \mb{S}_{+}^{c}
\end{equation}
with the cardinalities $n_{j}$ of each cluster $\mc{I}_{j}$ as entries. The quadratic form
\begin{equation}
\frac{1}{2}\la W^{j},K_{\mc{F}} W^{j}\ra
= \frac{1}{2}\sum_{i,k \in \mc{I}} k_{\mc{F}}(f_{i},f_{k}) W_{i,j} W_{k,j}
= \frac{1}{2}\sum_{i \in \mc{I}_{j}} k_{\mc{F}}(f_{i},f_{i}) + \sum_{i,k \in \mc{I}_{j} \colon i \neq k} k_{\mc{F}}(f_{i},f_{k})
\end{equation}
measures the \textit{size} of cluster $\mc{I}_{j}$ by the first sum of the right-hand side, which for common kernel functions like \eqref{eq:K-mcF-Gaussian} is proportional to the number $n_{j}$ of data points assigned to cluster $j$, and the \textit{connectivity} in terms of the weights $k_{\mc{F}}(f_{i},f_{k})$ of all edges $ik \in \mc{E}$ connecting points $i$ and $k$ in this cluster. Assuming that all clusters are non-empty, which amounts to the assumption
\begin{equation}\label{eq:ass-W-rank-c}
\rank(W) = c,
\end{equation}
we normalize the preceding expression by the cardinality and sum over all clusters, to obtain
\begin{subequations}\label{eq:derivation-partitioning-problem}
\begin{align}\label{eq:cluster-quadratic-form-a}
\sum_{j \in \mc{J}} \frac{1}{2 n_{j}}\la W^{j},K_{\mc{F}} W^{j}\ra
&= \frac{1}{2}\sum_{j \in \mc{J}}\frac{1}{n_{j}}\sum_{i \in \mc{I}_{j}} k_{\mc{F}}(f_{i},f_{i})  + \sum_{j \in \mc{J}} \frac{1}{n_{j}} \sum_{i,k \in \mc{I}_{j} \colon i \neq k} k_{\mc{F}}(f_{i},f_{k})
\\
&= \frac{1}{2}\sum_{j \in \mc{J}} \frac{1}{n_{j}} (W^{\T} K_{\mc{F}} W)_{j,j}
\overset{\eqref{eq:def_CW}}{=}
\frac{1}{2}\tr\big(C(W)^{-1} W^{\T} K_{\mc{F}} W\big)
\\ \label{eq:cluster-quadratic-form-c}
&= \frac{1}{2}\tr\big(K_{\mc{F}} A_{0}(W)\big),
\intertext{with}\label{eq:A0W}
A_{0}(W) &= W C(W)^{-1} W^{\T},\qquad W \in \mc{W}_{\ast}^{c}.
\end{align}
\end{subequations}
For common kernel functions like \eqref{eq:K-mcF-Gaussian}, the first sum of the right-hand side of \eqref{eq:cluster-quadratic-form-a} is just a constant. Objective \eqref{eq:cluster-quadratic-form-c} therefore essentially
 measures the normalized similarity weights \textit{not} cut by the partition of the underlying graph. As a result,
the problem to partition the data and the underlying graph into $c$ clusters takes the form
\begin{equation}\label{eq:partitioning-problem}
\max_{W} \tr\big(K_{\mc{F}} A_{0}(W)\big) \qquad\text{subject to}\qquad W \in \mc{W}_{\ast}^{c}.
\end{equation}
We record basic properties of the matrix $A_{0}(W)$.
\begin{lemma}\label{lem:A0W-properties}
Let $W \in \mc{W}_{\ast}^{c}$. Then the matrix $A_{0}(W)$ given by \eqref{eq:A0W} is \\[0.2cm]
\hspace*{1cm}
\begin{tabular}{lll}
(a) & nonnegative and symmetric,
\\
(b) & doubly stochastic, &
$A_{0}(W)\eins_{n} = A_{0}(W)^{\T}\eins_{n} = \eins_{n}$,
\\
(c) & and completely positive, &
$A_{0}(W) = Y Y^{\T},\qquad Y \geq 0$.
\end{tabular}
\end{lemma}
\begin{proof}
(a) is immediate. (b) follows from \eqref{eq:def_CW} and the constraint $W \in \mc{W}_{\ast}^{c}$ (recall \eqref{eq:def-W-ast}). (c) holds with
$Y = Y(W) = W C(W)^{-1/2}$.
\end{proof}
Property (c), i.e.~a completely positive factorization of the matrix $A_{0}(W)$ depending on $W$, reflects the combinatorial difficulty of the optimization problem \eqref{eq:partitioning-problem} -- see, e.g., \cite{Berman:2018aa,Bomze:2018aa} and references therein for more information about completely positive matrix factorization. Therefore, various relaxations of the constraint $W \in \mc{W}_{\ast}^{c}$ are discussed next.

\subsection{Self-Assignment Matrices, Relaxation}
\label{sec:selfAssignmentMatrix}
We start with the definitions of two basic self-assignment matrices. The first relaxation, based on \eqref{eq:A0W}, drops both the integrality constraint and the rank constraint.
\begin{definition}[\textbf{Self-Affinity Matrix}]\label{def:SelfAffinityMatrix}
The \emph{self-affinity matrix} is defined as the factorization
\begin{equation}\label{eq:Def_AW}
  A_{0}(W) := W C(W)^{-1} W^{\T},\qquad W \in \mc W.
\end{equation}
\end{definition}
The second definition is based on the observation that equivalent expressions for the normalizing matrix
\begin{equation}\label{eq:CW=WTW}
C(W)=W^{\T} W \qquad\text{if}\qquad
W \in \mc{W}_{\ast}^{c}
\end{equation}
differ after relaxing the feasible set $\mc{W}_{\ast}^{c}$. Dropping the integrality constraint, but keeping the rank constraint, yields the set of full-rank assignment matrices
\begin{equation}\label{eq:def-Wc}
 \mc{W}^{c} = \big\{W \in \mc{W} \colon \rank(W)=c \big\}
 \qquad\qquad(\textbf{full-rank assignments})
\end{equation}
and the following definition.
\begin{definition}[\textbf{Self-Influence Matrix}]\label{def:SelfInfluenceMatrix}
The \emph{self-influence matrix} is defined as the factorization
\begin{equation}\label{eq:def_SelfInfluenceMatrix}
	A_{1}(W):=W(W^{\T}W)^{-1}W^{\T},\qquad W \in \mc{W}^{c}.
\end{equation}
\end{definition}
\Cref{def:SelfAffinityMatrix,def:SelfInfluenceMatrix} differ by the normalizing matrices $C(W)$ and $W^{\T} W$, both of which are positive definite. It is then natural to define a one-parameter family of factorized matrices in terms of a geodesic \eqref{eq:def-gamma-P} on the positive definite manifold $\mc{P}^{c}$ that connects $C(W)$ and $W^{\T} W$, which gives rise to the following definition.
\begin{definition}[\textbf{Self-Assignment Matrix}]\label{def:SelfAssignmentMatrix}
The \emph{self-assignment matrix} with parameter $s$ is defined as the factorization
\begin{subequations}\label{eq:def_SelfAssignmentMatrix}
\begin{align}\label{eq:def-As}
A_{s}(W) &:= W\gamma_{s}(W)^{-1}W^{\T}, \quad \quad s \in [0,1], \qquad
W \in \begin{cases}
\mc{W}, &\text{if}\; s=0, \\
\mc{W}^{c}, &\text{if}\; s>0,
\end{cases}
\intertext{
with normalizing matrix
}\label{eq:def-gamma-s}
\gamma_{s}(W)
&= C(W)^{\frac{1}{2}} \big( C(W)^{-\frac{1}{2}} W^{\T}W C(W)^{-\frac{1}{2}}\big)^{s} C(W)^{\frac{1}{2}} \;\in\; \mc{P}^{c}.
\end{align}
\end{subequations}
\end{definition}
\noindent
Note that \cref{def:SelfAssignmentMatrix} corresponds to \cref{def:SelfAffinityMatrix,def:SelfInfluenceMatrix} if $s=0$ and $s=1$, respectively.

The following proposition collects properties of the self-assignment matrices defined above. Property (h) refers to a relation between matrices $A_{1}\big(W(t)\big)$ and $A_{1}\big(W(t')\big)$, for any $t,t' \in [0,T]$: they share the same eigenvalues.
\begin{proposition}[Properties of Self-Assignment Matrices]\label{prop:propSelfAssignments}
Let $A_{0}(W)$ and $A_{1}(W)$ be given by \cref{def:SelfAffinityMatrix,def:SelfInfluenceMatrix}, respectively. Then these matrices have (\cmark) or do not have (\xmark) the following properties.
$ $ \\

\begin{center}
\begin{tabularx}{\linewidth}{llc@{\hskip 3em}c@{\hskip 3em}c}
& & & \textbf{self-affinity} $A_{0}(W)$ & \textbf{self-influence} $A_{1}(W)$ \\
\addlinespace
& admissible assignments & & $W \in \mc{W}$  &  $W \in \mc{W}^{c}$ \\
\hline
\addlinespace
(a) & symmetric & & \cmark & \cmark \\
(b) & positive semi-definite & & \cmark & \cmark \\
(c) & nonnegative & & \cmark & \xmark \\
(d) & doubly stochastic & & \cmark & \xmark \\
(e) & completely positive & & \cmark & \xmark \\
(f) & rank & & $\leq c$ & $ = c$  \\
(g) & orthogonal projection & & \xmark & $\Pi_{\mc{R}(W)}$  \\
\addlinespace
(h) & iso-spectral & & \xmark & \cmark  \\
(i) & eigenvalues $\in$ & & $[0,1]$ & $\{0,1\}$ \\
(j) & multiplicity ($\lambda = 1$)  & & $=1$ & $=c$   \\
(k) & multiplicity ($\lambda = 0$)  & & $ \geq n-c$ & $=n-c$ \\
(l) & eigenvector(s) ($\lambda = 1$)  & & $\eins_n$ & $\big( W(W^{\T}W)^{-\tfrac{1}{2}} \big)^{j}, \quad j \in \mc{J}$ \\
\end{tabularx}
\end{center}
\end{proposition}
\begin{proof}
(a)-(f) are clear. We focus on (g)-(l).

(g) On easily checkes that $A_{1}(W) = A_{1}(W)^{2}$ is idempotent whereas $A_{0}(W)$ is not. Taking into account (a) implies the assertion for $s=1$.

(h) Follows from (i) and (j) for $s=1$.

(i) Case $s=0$. The lower eigenvalue bound $0$ follows from (a),(b), the upper bound $1$ from (d) and \cite[Thm.~5.3]{Plemmons:1994aa}. Case $s=1$. This is immediate due to $(g)$.

(j) Case $s=0$. $W \in \mc{W}$ implies that $A_{0}(W)$ is strictly positive. (i) and \cite[Thm.~1.4]{Plemmons:1994aa} then imply the assertion. Case $s=1$. This is immediate due to (f),(g).

(k) Both assertions follow from (f).

(l) Case $s=0$ follows from (d) and \cite[Thm.~5.3]{Plemmons:1994aa}. Case $s=1$. Setting $Y=W (W^{\T} W)^{-1/2}$, one directly computes $A_{1}(W) Y=Y$ and $Y^{\T} Y=I_{c}$.
\end{proof}
The last definition of this section concerns the `difference' between the normalizing matrices $C(W)$ and $W^{\T} W$ of \cref{def:SelfAffinityMatrix,def:SelfInfluenceMatrix,def:SelfAssignmentMatrix}.
\begin{definition}[\textbf{Cluster-Confusion Matrix}]\label{def:ClusterConfusionMatrix}
The \emph{cluster-confusion matrix} is defined as the matrix factorization
\begin{equation}\label{eq:Def_BW}
  B(W) := C(W)^{-1} W^{\T} W \in \R_{+}^{c \times c},\qquad W \in \mc W.
\end{equation}
\end{definition}
\begin{proposition}[Properties of the Cluster-Confusion Matrix]\label{prop:BW}
The cluster-confusion matrix $B(W)$ has the following properties.
	\begin{center}
		\begin{tabularx}{\linewidth}{llc@{\hskip 2em}l}
		(a) & entry-wise positive: & & $B(W) > 0$, \\
		(b) & row stochastic: & & $B(W)\eins_c = \eins_c$, \\
		(c) & pure clusters: & & $B(W) = I_c$ \; if and only if \; $W \in \mc{W}_{\ast}^{c}$, \\
		(d) & rank lower bound: & & $0 \leq \tr\big( B(W) \big) \leq \rank(W)$ \; with equality if \; $W \in \mc{W}_{\ast}^{c}$.
	\end{tabularx}
	\end{center}
\end{proposition}
\begin{proof}
(a)-(c) directly follow from the definitions of $B(W)$ and $W_{\ast}^{c}$. (d) follows from $\tr(B(W)) = \tr(A_{0}(W))$ together with \cref{prop:propSelfAssignments} (c) and (i).
\end{proof}
%

\subsection{Relaxations: Interpretation}\label{sec:Relaxation-Interpretation}

We take a closer look at the relaxations of problem \eqref{eq:partitioning-problem}.
\subsubsection{Self-Affinity Matrix}
\label{sec:InterSelfAffinityMatrix}

Following \cite{Astrom:2017ac}, we interpret each entry of the assignment matrix $W \in \mc{W}$ as posterior probability
\begin{equation}\label{eq:posterior}
 P(j|i)=W_{i,j}, \qquad j \in \mc{J},\quad i \in \mc{I}
\end{equation}
of label $j$, conditioned on the observation of the data point $f_{i}$. According to the completely unsupervised scenario here, we adopt the uniform prior distribution
\begin{equation}\label{eq:P-uniform}
P(i) = \frac{1}{n},\quad i \in \mc{I}
\end{equation}
of the data. Marginalization yields the label distribution
\begin{equation}
 P(j) = \sum_{i \in \mc{I}} P(j|i)P(i) = \frac{1}{n} \big( W^{\top} \eins_{n} \big)_j,
\end{equation}
which measures the size of cluster $\mc{I}_{j}$ in terms of the relative mass of assignments. Invoking Bayes' rule, we compute the distribution analogous to \eqref{eq:posterior}, but with the roles of data and labels reversed, to obtain
\begin{equation}\label{eq:likelihood}
Q(k|j) =\frac{P(j|k)P(k)}{P(j)}=\frac{W_{k,j}}{\sum_{i\in \mc{I}} W_{i,j}} = \big(C(W)^{-1} W^{\T} \big)_{j,k}.
\end{equation}
The probability of the self-assignments $f_{i} \leftrightarrow f_{k},\; i,k \in\mc{I}$ then result from marginalization over the labels
\begin{equation}\label{eq:A0-factorization}
{A}_{0;i,k}(W):=\sum_{j \in \mc{J}} Q(k|j)P(j|i) = \sum_{j \in \mc{J}} W_{i,j} \big(C(W)^{-1} W^{\T} \big)_{j,k} = \big(W C(W)^{-1} W^{\T} \big)_{i,k}.
\end{equation}
This expression explains the relaxation that is at the basis of \cref{def:SelfAffinityMatrix}. It specifies the probability that two vertices $i$ and $k$ get assigned the same label (no matter which one), i.e.~that they belong to the same cluster.

Finally, the derivation of problem \eqref{eq:partitioning-problem} -- cf.~\eqref{eq:derivation-partitioning-problem} -- showed that optimizing the assignments in order to maximize the correlation (inner product) of $A_{0}(W)$ and $K_{\mc{F}}$ amounts to cover the most similar data points by the components of the partition (clusters).

\begin{figure}[htpb]
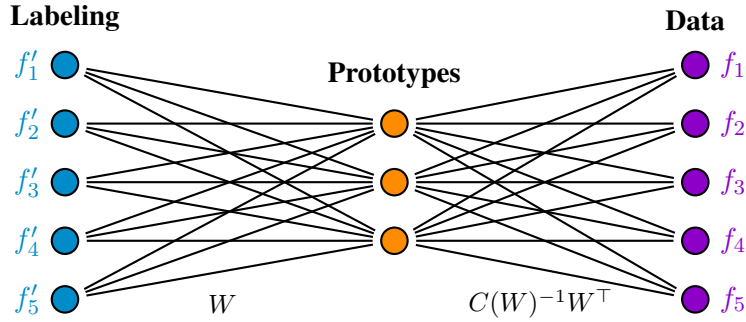

  \begin{center}
    \SelfAssignmentGraph{}
    \caption{The self-affinity matrix $A_{0}(W)$ due to \cref{def:SelfAffinityMatrix} comprises the probabilities for each pair of data points $f_{i}, f_{k} \in \mc{I}$ to belong to the same cluster. The factorization \eqref{eq:A0-factorization} of $A_{0}(W)$ admits the interpretation that optimizing the assignments $W$ implicitly forms prototypes $f_{j}^{\ast},\, j \in \mc{J}$ that are assigned to the data themselves so as to maximize the correlation with pairwise similarities given as entries of the matrix $K_{\mc{F}}$.}
    \label{fig:GraphAW}
  \end{center}
\end{figure}

\subsubsection{Recovery of Latent Prototypes}\label{sec:Recovery-Prototypes}
Although problem \eqref{eq:partitioning-problem} does not involve prototypes \eqref{eq:def-mcF-ast}, such prototypes can be recovered from the solution $W$ to the problem relaxation discussed in \cref{sec:InterSelfAffinityMatrix}. Specifically, the probabilities $Q(i|j)$ given by \eqref{eq:likelihood} indicate the contribution of each data point $f_{i}$ to cluster $j$. Consequently, adopting the manifold assumption that the data $\mc{F}_{n}$ are sampled on a Riemannian manifold, prototypes can be recovered as weighted Riemannian means by solving
\begin{equation}\label{eq:meanPrototypes}
f_j^{\ast} = \operatornamewithlimits{\arg \min}_{f \in \mc{F}} \sum_{i \in \mc{I}} \big(C(W)^{-1} W^{\T} \big)_{j,i} d^{2}_{\mc{F}}(f, f_i),\qquad j \in \mc{J}.
\end{equation}
In the basic case of Euclidean data $\mc{F}_{n} \subset \R^{d}$, this problem yields the closed form averages
\begin{equation}\label{eq:meanPrototypes-Euclidean}
f_j^{\ast} = \sum_{i \in \mc{I}} \big(C(W)^{-1} W^{\T} \big)_{j,i} f_i,\qquad j \in \mc{J}.
\end{equation}
\Cref{fig:GraphAW} illustrates the data self-assignment via the self-affinity matrix and latent prototypes.
\begin{remark}[relation and differences to basic clustering]
Choosing the squared Euclidean norm $d^{2}_{\mc{F}}(f, f_i) = \|f-f_{i}\|^{2}$ in \cref{eq:meanPrototypes} determines the prototype $f_{j}^{\ast}$, like $k$-means clustering, as arithmetic mean \eqref{eq:meanPrototypes-Euclidean} of the data $f_{i}$ assigned to cluster $j\in\mc{J}$ by the variables $\big(C(W)^{-1} W^{\T} \big)_{j,i}$. However, unlike $k$-means clustering and its variants that alternatingly update prototypes and assignment variables, the prototypes $f_{j}^{\ast},\,j \in \mc{J}$ are \textit{not explicitly} involved in our self-assignment flow approach. Rather, \cref{eq:meanPrototypes-Euclidean} is evaluated \textit{after} convergence of the self-assignment flow \eqref{eq:self-assignment-flow}. \Cref{sec:Self-Supervision} reveals that the dependency $f_{j}^{\ast} = f_{j}^{\ast}(W)$ of \eqref{eq:meanPrototypes} is consistent with graph partitioning through the self-assignment flow $W(t)$, in that prototypes $f_{j}^{\ast}$ that are \textit{implicitly} determined by \eqref{eq:meanPrototypes-Euclidean} maximize class separability. The usual initialization problem of basic clustering is handled by the self-assignment approach through the initialization $W(0)$ of \eqref{eq:self-assignment-flow-a} in terms of the given data.  In addition, we point out that basic clustering is lacking the influence of the spatial assignment \textit{regularization} through geometric averaging -- cf.~\eqref{eq:def-Si} -- on the formation of prototypes.
\end{remark}
%

\subsubsection{Self-Influence Matrix}\label{sec:InterSelfInfluenceMatrix}
Let $W \in \mc{W}^{c}$ be given and temporarily assume that $d$-dimensional Euclidean feature vectors are given as data $\mc{F}_{n}$ and collected as row vectors in the matrix
\begin{equation}\label{eq:def-F-matrix}
F = (f_{1},\dotsc,f_{n})^{\T} \in \R^{n\times d}.
\end{equation}
Let the matrix
\begin{equation}
F^{\ast} =(f^{\ast}_1,\dotsc,f^{\ast}_c)^{\T} \in \R^{c \times d}
\end{equation}
collect the prototypes. Given $W$ and $F$, a least-squares fit yields
\begin{equation}\label{eq:basicLSQ}
	F^{\ast} = \operatornamewithlimits{\arg \min}_{G \in \R^{c \times d}} \frac{1}{2} \| W G - F \|_{F}^2 = (W^{\T}W)^{-1}W^{\T}F,
\end{equation}
which is well-defined since $W \in \mc{W}^{c}$ has full rank. Using these prototypes in turn for predicting data $\hat F$ by assignment yields
\begin{equation}\label{eq:basicLSQFit}
\hat F = W F^{\ast} = W(W^{\T}W)^{-1}W^{\T}F
= \Pi_{\mc{R}(W)} F = A_{1}(W) F.
\end{equation}
Finally, optimizing the assignments $W$ in order to obtain the best prediction of the data itself, gives with $A_{1}(W)^{2}=A_{1}(W)$
\begin{equation}\label{eq:trSimLk}
 	 \operatornamewithlimits{\arg \min}_{W \in \mc{W}^{c}} \frac{1}{2} \| A_{1}(W)F - F \|_{F}^2
	= \operatornamewithlimits{\arg \max}_{W \in \mc{W}^{c}} \tr\big(A_{1}(W)FF^{\T}\big),
\end{equation}
and the initial assumption of Euclidean data can be dropped by replacing the Euclidean Gram matrix $F F^{\T}$ by a general inner product matrix $K_{\mc{F}}$ corresponding to the embedding of the data into a reproducing kernel Hilbert space.

As a result, the relaxation of problem \eqref{eq:partitioning-problem} due to \cref{def:SelfInfluenceMatrix} can be interpreted as finding the best $c$-dimensional subspace $\mc{R}(W)$ spanned by the (soft) indicator vectors of the $c$ clusters (column vectors of $W$) for self-prediction of the given data.

Another related `spectral' interpretation results from rewriting the objective in the form
\begin{subequations}\label{eq:A1-Rayleigh}
\begin{align}
 \tr\big(A_{1}(W)K_{\mc{F}}\big)
 &= \tr\big(W(W^{\T}W)^{-1}W^{\T} K_{\mc{F}}\big)
 = \tr\big((W^{\T}W)^{-\frac{1}{2}}W^{\T} K_{\mc{F}} W (W^{\T}W)^{-\frac{1}{2}}\big) \\
 &= \tr\big(Y(W)^{\T} K_{\mc{F}} Y(W)\big),\qquad
 Y(W) = W (W^{\T}W)^{-\frac{1}{2}}.
\end{align}
\end{subequations}
We conclude from \cref{prop:propSelfAssignments} that $Y(W)$ varies over the compact Stiefel manifold,
\begin{equation}\label{eq:Stiefel}
Y(W) \in \mrm{St}(c,n)=\{X \in \R^{n \times c} \colon X^{\T}X=I_c\},
\end{equation}
and that the objective \eqref{eq:A1-Rayleigh} is the Rayleigh quotient whose maximizer $Y$ spans the subspace of the $c$ dominant eigenvectors of $K_{\mc{F}}$
\cite[Ch.~1]{Helmke:1996aa}. Note, however, that $Y(W)$ cannot vary freely but is parameterized by $W \in \mc{W}^{c}$.

\subsubsection{Comparison of $A_{0}(W)$ and $A_{1}(W)$}\label{sec:Comparison-A-0-1}
$A_{1}(W)$ differs from $A_{0}(W)$ in that the normalizing matrix $C(W)$ of the former self-assignment matrix is replaced by $W^{\T} W$ in the latter. A consequence due to \cref{prop:propSelfAssignments} is that $A_{1}(W)$ is no longer doubly stochastic and may have negative entries. Hence the probabilistic interpretation \eqref{eq:A0-factorization} of the factorization of $A_{0}(W)$ no longer holds for $A_{1}(W)$. On the other hand, unlike $A_{0}(W)$, matrix $A_{1}(W)$ has fixed rank $c$ and embeds data in a corresponding subspace.

Formulas \eqref{eq:meanPrototypes-Euclidean} and \eqref{eq:basicLSQ} for the formation of latent prototypes (Euclidean case) are the same when using $A_{0}(W)$ or $A_{1}(W)$, up to the different normalizing matrices.  And how these prototypes are used to represent the data is made explicit by \cref{fig:GraphAW} and \cref{eq:basicLSQFit}, respectively. Both matrices $A_{0}(W)$ and $A_{1}(W)$ are equivalent for \textit{labelings} $W \in \mc{W}_{\ast}^{c}$. What labelings are computed, however,  depends on the self-assignment flow (\cref{sec:selfAssignmentSection}) and hence on the parameter $s \in [0,1]$.

\subsubsection{Cluster-Confusion Matrix}
\label{sec:InterClusterConfusionMatrix}
Using \eqref{eq:posterior} and \eqref{eq:likelihood} the entries of the cluster-confusion matrix \eqref{eq:Def_BW} take the form
\begin{equation}\label{eq:cluster-confusion-entry}
B_{j,l}(W):=\sum_{i \in \mc{I}} P(l|i)Q(i|j)
= \big(C(W)^{-1} W^{\T} W \big)_{j,l},\qquad
j,l \in \mc{J}.
\end{equation}
This expression may be interpreted as probability that clusters $\mc{I}_j$ and $\mc{I}_l$ are connected (soft partition), as opposed to the case of integral assignments (labelings) $W \in \mc{W}_{\ast}^{c}$, in which case $B(W)=I_{c}$ and all clusters are disjoint (hard partition).

\section{Self-Assignment Flows}\label{sec:selfAssignmentSection}

In this section, we generalize the assignment flow \eqref{eq:assignment-flow} to the unsupervised scenario discussed in \cref{sec:selfAssignments}. Generalizing the likelihood map \eqref{schnoerr-eq:def-Li} is the major step (\cref{sec:Generalized-L}). The remaining components of the assignment flow remain unchanged, except for starting the flow at the perturbed barycenter $W(0)$ of the assignment manifold -- see \cref{eq:self-assignment-flow-a} below -- in order to break the symmetry of uniform label assignments through the data, in the absence of labels and any prior information (\cref{sec:Self-Assignment-Flows}). Next, we complement in \cref{sec:Self-Supervision} the interpretations of the relaxations underlying the self-assignment flow (\cref{sec:Relaxation-Interpretation}) and show that the latent prototypes determined by the flow maximize class separability. Finally, numerical aspects are discussed in \cref{sec:Geometrical-Integration}.

\subsection{Generalized Likelihood Map}\label{sec:Generalized-L}
In the supervised case, for a given distance matrix $D_{\mc{F}}$ \eqref{eq:def-distance-vector}, \textit{local} label assignment is simply achieved by determining separately the smallest component of the vectors $D_{\mc{F};i}$, for every vertex $i \in \mc{I}$. This corresponds to solving
\begin{equation}
\min_{W \in \mc{W}}\tr(D_{\mc{F}} W^{\T})
\end{equation}
and the likelihood map \eqref{schnoerr-eq:def-Li} lifts the scaled negative \textit{gradient} of this objective function to $\mc{S}$. In view of problem \eqref{eq:partitioning-problem} and the family of self-assignment matrices due to \eqref{eq:def_SelfAssignmentMatrix}, a natural approach to generalize this supervised set-up to the unsupervised case is to consider the problem
\begin{subequations}\label{eq:def-Es}
\begin{align}
\max_{W}\; E_{s}(W) &\quad\text{subject to}\quad
W \in \begin{cases}
\mc{W}, &\text{if}\; s=0 \\
\mc{W}^{c}, &\text{if}\; s \in (0,1]
\end{cases}
\label{eq:def-Es-a} \\ \label{eq:def-Es-b}
E_{s}(W) &= \tr\big(K_{\mc{F}} A_{s}(W)\big)
\end{align}
\end{subequations}
and to replace $-D_{\mc{F}}$ in the likelihood map by the gradient $\partial E_{s}(W)$.
For $s=0$ and $s=1$, respectively, we have
\begin{subequations}
\begin{align} \label{eq:grad-E0}
\partial E_{0}(W)
&= 2K_{\mc{F}} W C(W)^{-1} - \eins_{n} \diag\big(C(W)^{-1} W^{\T} K_{\mc{F}} W C(W)^{-1}\big)^{\T},
\\ \label{eq:grad-E1}
\partial E_{1}(W)
&= 2 \big( I_n - A_{1}(W) \big) K_{\mc{F}} W(W^{\T}W)^{-1}.
\end{align}
\end{subequations}
In order to substantiate this approach, we interpret these gradients using the concepts from \cref{sec:Relaxation-Interpretation}. For illustration, let $K_{\mc{F}}=F F^{\T}$ be a Euclidean inner product matrix, with $F$ given by \eqref{eq:def-F-matrix}. Equation \eqref{eq:meanPrototypes-Euclidean} determining the latent prototypes as averages weighted by the likelihood $Q(i|j)$, \cref{eq:likelihood}, reads
\begin{equation}\label{eq:def-FStar}
f_{j}^{\ast} = \sum_{i \in \mc{I}} \big(C(W)^{-1} W^{\T} \big)_{j,i} f_{i}
= \big(C(W)^{-1} W^{\T} F \big)_{j},\qquad
(F^{\ast})^{\T} = F^{\T} W C(W)^{-1}.
\end{equation}
We have
\begin{subequations}
\begin{align}
\partial E_{0}(W)
&= 2 F F^{\T}  W C(W)^{-1}-\eins_{n}\diag\big( (F^{\T} W C(W)^{-1})^{\T} F^{\T}  W C(W)^{-1} \big)^{\T}
\\
&= 2 F (F^{\ast})^{\T} - \eins_{n} \diag(F^{\ast} (F^{\ast})^{\T})^{\T},
\\
\big(\partial E_{0}(W)\big)_{i,j}
&= 2\la f_{i},f^{\ast}_{j}\ra - \|f^{\ast}_{j}\|^{2}
= -\|f_{i}-f^{\ast}_{j}\|^{2} + \|f_{i}\|^{2}, \label{eq:grad-euclid-A0}
\end{align}
\end{subequations}
where the prototypes $f^{\ast}_{j} = f^{\ast}_{j}(W)$ depend on $W$. The last term on the r.h.s.~of \eqref{eq:grad-euclid-A0} does not depend on $j$ and hence is factored out -- cf. \cref{rem:exp-T0} -- when lifting the vector \eqref{eq:grad-euclid-A0} to the assignment manifold. Hence, we ignore this term
and generalize the likelihood map \eqref{schnoerr-eq:def-Li} to
\begin{equation}\label{eq:def-L0i}
L_{0;i}(W_{i})
= \exp_{W_{i}}\Big(\frac{1}{\rho}\partial E_{0}(W)_{i}\Big)
= \exp_{W_{i}}\Big(- \frac{1}{\rho} \big(\|f_{i}-f^{\ast}_{j}\|^{2}\big)_{j \in \mc{J}} \Big),
\end{equation}
which amounts to replace the distance vectors $D_{\mc{F};i}$, for \textit{given} prototypes in the supervised case, by a \textit{varying} squared distance depending on \textit{latent} prototypes, that emerge when the assignments $W(t)$ follow the assignment flow.

Now let $s=1$. We return to the `spectral' interpretation in terms of \eqref{eq:A1-Rayleigh} and \eqref{eq:Stiefel}. The Riemannian gradient of the Rayleigh quotient $E_{1}(Y)=\tr(Y^{\T} K_{\mc{F}} Y)$ over the compact Stiefel manifold \eqref{eq:Stiefel} equipped with the standard Euclidean metric reads \cite[Sec.~4.8]{Absil2009})
\begin{equation}\label{eq:Rgrad-E1}
 \textnormal{grad} E_{1}(Y) = 2 (I_{n}- Y Y^{\T})K_{\mc{F}} Y \quad \in \quad T_{Y} \mrm{St}(c,n).
\end{equation}
Next we relate the Euclidean gradient \eqref{eq:grad-E1} to the Riemannian gradient \eqref{eq:Rgrad-E1}, taking into account the parametrization $Y(W) \in  \mrm{St}(c,n)$ in \eqref{eq:A1-Rayleigh}, to obtain
\begin{subequations}
\begin{align}
\partial E_{1}(W)
&= 2 \big( I_n - A_{1}(W) \big) K_{\mc{F}} W(W^{\T}W)^{-1}
\\
&= 2 \big( I_n - Y(W)Y(W)^{\T} \big) K_{\mc{F}} Y(W) (W^{\T}W)^{-\tfrac{1}{2}}
\\
&= \textnormal{grad} E_{1}(Y(W)) (W^{\T}W)^{-\tfrac{1}{2}}. \label{eq:relate-Stiefel-c}
\end{align}
\end{subequations}
Since the second factor in \eqref{eq:relate-Stiefel-c} is non-singular, we conclude
\begin{equation}\label{eq:grad-Euclidean-Stiefel}
\partial E_{1}(W) = 0 \quad \Leftrightarrow \quad \textnormal{grad} E_{1}(Y(W)) = 0.
\end{equation}
In words, $W \in \mc{W}^{c}$ is a stationary point if and only if $Y(W) \in \mrm{St}(c,n)$ is a stationary point of the Rayleigh quotient over the compact Stiefel manifold. Consequently the gradient $\eqref{eq:grad-E1}$ is directly linked to the search direction on the compact Stiefel manifold, in order to determine the invariant subspace corresponding to the $c$ dominant eigenvectors of $K_{\mc{F}}$.

As a consequence of these considerations, we define
for arbitrary $s \in [0,1]$ the \textbf{generalized likelihood map} as
\begin{equation}\label{eq:def-Lsi}
L_{s;i}(W_{i})
= \exp_{W_{i}}\Big(\frac{1}{\rho}\partial E_{s}(W)_{i}\Big),
\end{equation}
with $E_{s}(W)$ given by \eqref{eq:def-Es}.

\subsection{Self-Assignment Flows}\label{sec:Self-Assignment-Flows}
Besides replacing the likelihood map \eqref{schnoerr-eq:def-Li} by the generalized likelihood map \eqref{eq:def-Lsi}, no further changes are required in order to generalize the assignment flow \eqref{eq:assignment-flow} to the unsupervised case (cf. \cref{fig:AF-supervised,fig:SAF}), except for the initialization which cannot both start at the barycenter and break the symmetry, without any prior information. This will be achieved by taking a small perturbation of the barycenter as initial point.

Accordingly, we define the one-parameter family of \textbf{self-assignment flows (SAFs)}
\begin{subequations}\label{eq:self-assignment-flow}
\begin{align}
\dot W = R_{W}S(W),\qquad W(0)&=\exp_{\eins_{\mc{W}}}(-\veps D_{\mc{F},0}),\quad 0 < \veps \ll 1
\label{eq:self-assignment-flow-a} \\ \label{eq:self-assignment-flow-b}
W(t) &\in \begin{cases}
\mc{W},&\text{if}\;s=0, \\
\mc{W}^{c},&\text{if}\; s \in (0,1].
\end{cases}
\end{align}
\end{subequations}
The matrix $D_{\mc{F},0}$ is computed using the given data $\mc{F}_{n}$ as explained in \cref{sec:greedy-clustering}. The flow $W(t)$ is restricted to the submanifold of full-rank assignments if $s>0$.

\Cref{prop:propSelfAssignments} and \cref{eq:CW=WTW} yield the following.
\begin{corollary}\label{cor:Wt}
Let $W(t)$ solve \eqref{eq:self-assignment-flow}. Then, for any $t \geq 0$,
\begin{enumerate}[(i)]
\item the self-affinity matrix $A_{0}\big(W(t)\big)$ is doubly stochastic and completely positive, if $s=0$;
\item the self-influence matrix $A_{1}\big(W(t)\big)$ is iso-spectral, i.e.~its eigenvalues satisfy $\lambda_{1} = \dotsb = \lambda_{c} = 1$ and $\lambda_{n-c}=\dotsb=\lambda_{n}=0$, if $s=1$.
\item
$A_{0}\big(W(T)\big)=A_{1}\big(W(T)\big)$ if $W(T) \in \mc{W}_{\ast}^{c}$.
\end{enumerate}
\end{corollary}
Property $(iii)$ relates to the fact that $W(t)$ solving \eqref{eq:self-assignment-flow} approaches a labeling $W(T) \in \mc{W}_{\ast}^{c}$ for sufficiently large $T$ after a trivial rounding step. We point out, however, that solving \eqref{eq:self-assignment-flow} generally yields different paths $W(t),\,t \in [0,T]$ depending on $s \in [0,1]$ and corresponding to the different relaxations, as discussed in \cref{sec:Relaxation-Interpretation}. Once a labeling $W(T) \in \mc{W}_{\ast}^{c}$ has been computed, using any $s \in [0,1]$, the solution is a local optimum of the partitioning problem \eqref{eq:partitioning-problem}. This is what \cref{cor:Wt}(iii) says.
\begin{remark}[parameters of the self-assignment flow]
We briefly explain the role of each parameter involved in order to point out, that there is essentially a single user parameter only, that has to be specified.
\begin{itemize}
\item
Any small positive number $\veps>0$ determining $W(0)$ by \eqref{eq:self-assignment-flow-a} will do in practice.
\item
The parameter $s\in [0,1]$ of \eqref{eq:self-assignment-flow-b} is chosen depending on the application: As \cref{fig:Seastar} illustrates, and as a consequence of the interpretations of the self-affinity matrix $A_{0}(W)$ (\cref{sec:InterSelfAffinityMatrix}) and the self-influence matrix $A_{1}(W)$ (\cref{sec:InterSelfInfluenceMatrix}), small values $s$ increase the sensitivity of the self-assignment flow to the spatial structure of the partition of the underlying graph $\mc{G}$, whereas large values $s$ make the approach more sensitive with respect to the quantization of the feature space $\mc{F}$ in terms of the prototypes, that are implicitly determined by the self-assignment flow (\cref{sec:Recovery-Prototypes}).
\item
Parameter $\rho$ of the likelihood map \eqref{schnoerr-eq:def-Li} merely normalizes the scale of the input similarity matrix $K_{\mc{F}}$, that can be small or large depending on the particular data under consideration.
\item
The fixed stepsize $h>0$ used in this paper for geometric numerical integration (\cref{sec:Geometrical-Integration}) can be determined automatically if a more advanced numerical scheme with adaptive stepsize control is employed, as worked out by \cite{Zeilmann:2018aa}.
\item
Parameter $c \in \N$ merely specifies an \textit{upper bound} of the number of clusters, whereas the \textit{resulting effective} number of clusters $\hat c \leq c$ does \textit{not} need to be specified beforehand (see \cref{def:hat-c} below).
\end{itemize}

As a result, the only parameter that critically influences the result returned by the self-assignment flow is the \textit{size} $|\mc{N}_{i}|$ of the neighborhoods \eqref{eq:def-Ni}, that determines the \textit{scale} of geometric spatial regularization \eqref{eq:def-Si} and, in turn, the number $\hat c$ of effective clusters. \Cref{sec:Experiments} provides numerous illustrations.
\end{remark}
%

%

\subsection{Self-Assignment Performs Self-Supervision}
\label{sec:Self-Supervision}

We interpret the assignment flow from another point of view that complements the interpretations discussed in \cref{sec:Relaxation-Interpretation}.

In \cref{sec:Recovery-Prototypes}, we showed that running the assignment flow entails learning of latent prototypes that can be explicitly recovered if weighted means in the data space are well-defined and computationally feasible.
Let us temporarily adopt the Euclidean situation \eqref{eq:meanPrototypes-Euclidean}. With these recovered prototypes at hand, we get back to \cref{sec:scatter-matrices} and ask how our approach  relates to the \textit{supervised} situation where the quality of the clustering can be assessed by objectives like \eqref{eq:Sw-1-Sb}. Assuming a labeling $W = W(T) \in \mc{W}_{\ast}^{c}$ has been determined, let the recovered prototypes $f^{\ast}_{j},\,j \in \mc{J}$ play the role of the empirical means $m_{j},\, j \in \mc{J}$. We compute in terms of the data matrix $F$ \eqref{eq:def-F-matrix} the quantities \eqref{eq:def-m-k}
\begin{subequations}
	\begin{align}
	P_{j} &= \frac{1}{n} \la W^{j}, \eins_n \ra =  \frac{1}{n} |\mc{I}_j|,\qquad j \in \mc{J}
	&&(\text{prior probabilities})
	\\
	f^{\ast}_{j} &= F^{\T}\big(W C(W)^{-1}\big)^{j},\qquad j \in \mc{J}
	&&(\text{class-conditional mean vectors})
	\\
	f^{\ast} &=  \frac{1}{n} F^{\T} \eins_{n},
	&&
	(\text{mean vector})
	\end{align}
\end{subequations}
and in turn the scatter matrices \eqref{eq:def-scatter-matrices}
\begin{subequations}
	\begin{align}
	S_{t} &= \frac{1}{n} \sum_{i \in \mc{I}} (f_{i}-f^{\ast})(f_{i}-f^{\ast})^{\T}
	= \frac{1}{n} F^{\T} \big(I - \frac{1}{n} \eins_n \eins_n^{\T} \big)F,
	\\
	S_{w}(W) &= \frac{1}{n}\sum_{j \in \mc{J}}\sum_{i \in \mc{I}_j}(f_{i}-f^{\ast}_{j})(f_{i}-f^{\ast}_{j})^{\T}
	= \frac{1}{n} F^{\T} \big(I - A_{0}(W) \big)F,
	\\
	S_{b}(W) &= \sum_{j \in \mc{J}} P_{j} (f^{\ast}_{j}-f^{\ast})(f^{\ast}_{j}-f^{\ast})^{\T}
	= \frac{1}{n} F^{\T} \big(A_{0}(W) - \frac{1}{n} \eins_n \eins_n^{\T} \big)F.
	\end{align}
\end{subequations}
Regarding the dependency on $W$, we observe that the within-class scatter matrix $S_{w}(W)$ involves the term $F^{\T} A_{0}(W) F$ and the between-class scatter $S_{b}(W)$ the term $-F^{\T} A_{0}(W) F$. Hence, by minimizing the objective \eqref{eq:partitioning-problem}, we \textit{simultaneously} minimize $\tr(S_w)$ and maximize $\tr(S_b)$:
\begin{equation}
\operatornamewithlimits{\arg \min}_{W} \ \tr\big(S_w(W)\big) \quad \Leftrightarrow \quad \operatornamewithlimits{\arg \max}_{W} \ \tr\big(S_b(W)\big) \quad \Leftrightarrow \quad \operatornamewithlimits{\arg \max}_{W} \ \tr \big( A_{0}(W)FF^{\T} \big).
\end{equation}
We conclude that the latent prototypes determined by the self-assignment flow turns a completely unsupervised scenario into a supervised one, in agreement with established measures for class separability like \eqref{eq:Sw-1-Sb}. This interpretation also remains valid when the relaxation with $s=1$ and objective \eqref{eq:trSimLk} is used to compute a labeling $W$, due to \cref{cor:Wt}(iii).

Moreover, since the approach only depends on the inner product matrix $F F^{\T}$, it generalizes to data embeddings into a reproducing kernel Hilbert space and a corresponding data affinity matrix $K_{\mc{F}}$ with entries \eqref{eq:K-mcF-entries}.


\subsection{Geometric Numerical Integration}\label{sec:Geometrical-Integration}
We distinguish the two cases \eqref{eq:self-assignment-flow-b}.

\vspace{0.25cm}
\noindent
\textbf{Case $s=0$.} We directly apply the methods studied by \cite{Zeilmann:2018aa}. To make this paper self-contained, we merely state the simplest scheme, the geometric Euler method. This explicit scheme with fixed step-size $h>0$ reads
\begin{equation}\label{eq:geom-integration}
W_{i}^{(k+1)} = \Exp_{W_{i}^{(k)}}\big(h R_{W_{i}^{(k)}} S(W^{(k)}) \big), \quad i \in \mc{I}.
\end{equation}
It ensures that the self-assignment flow \eqref{eq:self-assignment-flow-a} evolves properly on the assignment manifold $\mc{W}$. See \cite{Zeilmann:2018aa} for more advanced numerical schemes that run `automatically' through adaptive stepsize control. The iteration \eqref{eq:geom-integration} stops when the average entropy of the assignments $W^{(K)}$ drops at some iteration $k=K$ below the predefined threshold $10^{-3}$, which indicates (almost) unique label assignments and hence stationarity of the flow evolution. Then numerical integration is terminated and a labeling $W \in \mc{W}_{\ast}^{\hat c},\, \hat c \leq c$, is determined using $W^{(K)}$ in a trivial postprocessing step by selecting the most likely label for each row $W^{(K)}_{i},\, i \in \mc{I}$ and removing the $c-\hat c$ zero-columns (corresponding to empty clusters) from the resulting labeling $W\in \mc{W}_{\ast}^{\hat c}$.
\begin{definition}[\textbf{Effective Number $\hat c$ of Clusters (Labels)}]\label{def:hat-c}
We call the just described number
\begin{equation}\label{eq:def-hat-c}
\hat c \leq c
\end{equation}
the \textit{effective number of clusters or labels}, respectively. It is determined by the homogeneity of the data $\mc{F}_{n}$ and by the scale
\begin{equation}\label{eq:scale}
|\mc{N}_{i}|,\quad i \in \mc{I} \qquad\qquad(\textbf{scale})
\end{equation}
at which regularization is performed by the assignment flow through the similarity map \eqref{eq:def-Si}. We denote the corresponding index set of labels by
\begin{equation}\label{eq:hat-J}
\hat{\mc{J}} \subset \mc{J},\qquad |\hat{\mc{J}}| = \hat c.
\end{equation}
\end{definition}
\begin{remark}
The assertions of \cref{cor:Wt} as well as the considerations in \cref{sec:Self-Supervision} remain valid after replacing the upper bound $c$ of the number of prototypes (labels) and the corresponding index set $\mc{J}$ by $\hat c$ and $\hat{\mc{J}}$, respectively, according to \cref{def:hat-c}.
\end{remark}

\vspace{0.25cm}
\noindent
\textbf{Case $s=1$.} Integration of the self-assignment flow \eqref{eq:self-assignment-flow-a} restricted to the open submanifold $\mc{W}^{c}$ of full-rank assignments \eqref{eq:def-Wc} is more involved. Corresponding geodesics only locally exist on $\mc{W}$, i.e.~full-rank assignment matrices cannot be guaranteed during the numerical integration process \eqref{eq:geom-integration}. Clearly, if the data affinity matrix $K_{\mc{F}}$ has high rank (induced by heterogeneous data) and if the scale \eqref{eq:scale} for regularization is not chosen too large, a full-rank labeling $W \in \mc{W}^{c}$ may be returned by the self-assignment flow, that is well-defined in view of the relation \eqref{eq:grad-Euclidean-Stiefel}.

In order to handle other cases while still using the numerical scheme \eqref{eq:geom-integration} or more sophisticated ones \cite{Zeilmann:2018aa}, we simply replace the inverse normalizing matrix by its pseudo-inverse,
\begin{equation}\label{eq:WTW-pseudoinverse}
(W^{\T} W)^{-1} \quad\longleftarrow\quad
(W^{\T}W)^{\dagger}.
\end{equation}
Whenever this regularization of the normalizing matrix becomes `active', we extract the effective number $\hat c$ in a postprocessing step, as described above in the case $s=0$.

\section{Related Work and Discussion}\label{sec:Related-Work}
The literature on clustering is vast. We therefore restrict the discussion to few major methodological directions in the literature: Graph cuts and spectral relaxation (\cref{sec:Related-Work:Spectral}), discrete regularized optimal transport (\cref{sec:Discrete-OT}) and combinatorial optimization for graph partitioning (\cref{sec:Combinatorial-Optimization}).
\subsection{Graph Cuts and Spectral Relaxation}\label{sec:Related-Work:Spectral}
Summing up the weights (affinities) of edges that are cut provides a natural quality measure for graph partitioning. To avoid unbalanced partitions, such measures are normalized in various ways, and spectral relaxations of the resulting combinatorial optimization problem renders the computation of good suboptimal solutions feasible. We refer to \cite{Luxburg:2007aa} for a survey.

We focus on two basic balanced cut-criteria that can be expressed by the graph Laplacian
\begin{equation}
L_{\mc{F}} = D_{K,\mc{F}}-K_{\mc{F}},\qquad
D_{K,\mc{F}} = \Diag( K_{\mc{F}} \eins_n )
\end{equation}
and indicator vectors. The \textit{ratio-cut criterion }reads
\begin{equation}\label{eq:ratio-cut}
  \min_{U \in \R^{n \times c}}\, \tr( U^{\T} L_{\mc{F}} U ) \quad \textnormal{subject to} \quad  U \geq 0, \quad U^{\T} U = I_c,
\end{equation}
whereas the \textit{normalized-cut (Ncut) criterion} \cite{Shi2000} additionally uses the degree matrix $D_{K,\mc{F}}$ for normalization,
\begin{equation}\label{eq:normalized-cut}
  \min_{U \in \R^{n \times c}}\, \tr( U^{\T} L_{\mc{F}} U ) \quad \textnormal{subject to} \quad  U \geq 0, \quad U^{\T} D_{K,\mc{F}} U = I_c.
\end{equation}
Due to the conjunction of nonnegativity and orthogonality constraints, both problems \eqref{eq:ratio-cut} and \eqref{eq:normalized-cut} are difficult to optimize globally. \textit{Spectral relaxation} means to drop the element-wise nonnegativity constraint. Then the relaxed problems \eqref{eq:ratio-cut} and \eqref{eq:normalized-cut} amount to solving an eigenvalue problem and a generalized eigenvalue problem, respectively. The price to pay in either case is that the physical interpretation of $U$ as indicator variables is lost and must be recovered by an additional post-processing step, which is usually done by applying the classical k-means algorithm.

A direct relation to the proposed self-assignment flow is apparant in the case $s=1$. Substituting $Y=D_{K,\mc{F}}^{1/2}U$ in the spectral relaxation of \eqref{eq:normalized-cut} results in the problem
\begin{equation}
  \max_{Y \in \R^{n \times c}}\, \tr \big( Y^{\T} \tilde{K}_{\mc{F}} Y \big) \quad \textnormal{subject to} \quad  Y^{\T} Y = I_c,
\end{equation}
that is, the Rayleigh quotient of the \textit{normalized} affinity matrix $\tilde{K}_{\mc{F}} = D_{K,\mc{F}}^{-1/2} K_{\mc{F}} D_{K,\mc{F}}^{-1/2}$ has to be maximized over the compact Stiefel manifold \eqref{eq:Stiefel}. As already discussed for $s=1$ in connection with \eqref{eq:A1-Rayleigh}, assignments $W$ following the self-assignment flow parametrize points $Y(W) \in \mrm{St}(c,n)$ on the compact Stiefel manifold that maximize the Rayleigh quotient: \Cref{eq:relate-Stiefel-c} shows that the driving force of the self-assignment flow (generalized likelihood map) is directly linked to the gradient ascent of the Rayleigh quotient over the compact Stiefel manifold. Finally, when the numerical integration of the self-assignment flow terminates, then the resulting labeling $W \in \mc{W}_{\ast}^{c}$ together with \eqref{eq:CW=WTW} ensures $Y(W) \geq 0$. Hence, after re-substitution, $U(W)=D_{K,\mc{F}}^{-1/2}Y(W)$ is directly feasible for the original problem \eqref{eq:normalized-cut} and hence no `projection' by $k$-means is required as post-processing.

The common way to take into account \textit{spatial regularization} in spectral clustering is to augment given features by spatial coordinates. However, this strategy suffers from a conceptual shortcoming, since augmentation makes the \textit{same} feature vector differ when it is observed at two different spatial locations. In contrast, the self-assignment flow performs unbiased spatial regularization by smooth  geometric averaging and recognizes closeness of features no matter \textit{where} they are observed.

\subsection{Discrete Parametrized and Regularized Optimal Transport}
\label{sec:Discrete-OT}
The theory of optimal transport \cite{Villani:2009aa,Santambrogio:2015aa} has become a major modeling framework for data analysis. Here we focus on discrete optimal transport and computational aspects \cite{Burkhard:1999aa,Peyre:2018aa}.

We consider the case $s=0$ and the self-affinity matrix $A_{0}(W)$. Since $A_{0}(W)$ is doubly stochastic (\cref{prop:propSelfAssignments}), maximizing the objective $E_{0}(W)$ \eqref{eq:def-Es-b} may be interpreted as a discrete optimal transport problem with cost matrix $K_{\mc{F}}$ and uniform marginal measures \eqref{eq:P-uniform}. These marginals correspond to the data $\mc{F}_{n}$ and a copy of the data, respectively, resulting in data self-assignment as discussed in \cref{sec:InterSelfAffinityMatrix}.

For further interpretation, we consider the Euclidean case $K_{\mc{F}}=F F^{\T}$. Inserting the explicit form \eqref{eq:Def_AW} of $A_{0}(W)$ into the objective $E_{0}(W)$ and using \eqref{eq:def-FStar}, we obtain
\begin{equation}
E_{0}(W) = \tr(K_{\mc{F}} W C(W)^{-1} W^{\T})
= \tr(W F^{\ast} F^{\T}).
\end{equation}
Maximizing this objective function reveals what this problem relaxation actually means: A linear assignment problem in terms of the assignment matrix $W$ with \textit{varying} inner product matrix $F^{\ast}(W) F^{\T}$ as costs. Moreover, since $W \in \mc{W}$, we have a fixed marginal $W\eins_{c}=\eins_{n}$ and a the second marginal $W^{\T}\eins_{n}=\diag\big(C(W)\big)$ which is \textit{free}. Alltogether, a quite difficult problem is solved in terms of $W$: latent prototypes $F^{\ast}$ are formed by \textit{transporting} the uniform prior measure to the support of the respective clusters, so as to maximize the correlation $E_{0}(W)$ of the assignments $W$ and the inner product matrix $F^{\ast} F^{\T}$.

We point out a key property of the assignment flow that makes this approach work: It is the spatial regularization performed by the similarity map \eqref{eq:def-Si} that drives the entire process, in addition to the underlying geometry that makes $W(t)$ converge towards hard assignments (labelings). In fact, without spatial regularization, the self-affinity matrix $A_{0}(W)=I_{n}$ would maximize $E_{0}(W)$ assuming the similarity $k_{\mc{F}}(f_{i},f_{k})$ is maximal if $f_{i}=f_{k}$, which means that every given data point $f_{i}$ forms its own cluster. This trivial solution is ruled out, by construction, through the factorization with rank upper bounded by $c$ and  through geometric spatial averaging of the assignments. The corresponding \textit{scale} in terms of the sizes of the neighborhoods \eqref{eq:def-Ni}  determines how coarse or fine the spatial arrangement of the resulting clusters will be.

We informally summarize this discussion: Data self-assignment is defined by uniform marginal measures and a coupling measure parametrized by the assignment flow. Structure in the data is induced by imposing a low-rank constraint (factorization) on the coupling measure (transport plan) and through spatial regularization of the flow of assignments.

\subsection{Combinatorial Optimization}
\label{sec:Combinatorial-Optimization}

Zass and Shashua \cite{Zass:2005aa} studied the formulation of the clustering problem
\begin{subequations}\label{eq:combOptZass}
\begin{align}
\max_{Y \in \R^{n \times c}}&\tr(K_{\mc{F}} Y Y^{\T})
\qquad\text{subject to}
\\
&\mrm{(a)}\; Y \geq 0,\qquad
\mrm{(b)}\; \rank(Y)=c,\qquad
\mrm{(c)}\; Y^{\T} Y=I_{c},\qquad
\mrm{(d)}\; Y Y^{\T} \eins_{n}=\eins_{n}
\end{align}
\end{subequations}
in terms of the completely positive factorization $Y Y^{\T}$ and the constraints (a)--(d).
We notice that the orthogonality constraint (c) with respect to the columns of $Y$ implies (b), and that (a) together with (d) says that $Y Y^{\T}$ is doubly stochastic. The authors show that (a)--(d) imply that $W = YC(Y)^{\tfrac{1}{2}} \in \mc{W}_{\ast}^{c}$ is a labeling. This problem formulation differs from more classical conditions ensuring $W \in \mc{W}_{\ast}^{c}$ \cite[Lemma 2.1]{Rendl:1995aa},
\begin{equation}
W \geq 0,\qquad
W\eins_{c} = \eins_{n},\qquad
W^{\T}\eins_{n} = (n_{1},\dotsc,n_{c})^{\T},\qquad
\tr(W^{\T} W) = n,
\end{equation}
in that the cluster sizes (third constraint) do not have to be specified beforehand.

Regarding relaxation, the authors of \cite{Zass:2005aa} argue that the orthogonality constraint (c) is the weakest one. They propose a two-step procedure after dropping the constraints (b) and (c): approximation of the data similarity matrix $K_{\mc{F}}$ by a doubly stochastic matrix using the Sinkhorn iteration, followed by a gradient ascent iteration with stepsize control so as to respect the remaining constraints. The same set-up was proposed by \cite{Yang2016} except for determining a locally optimal solution by a single iterative process using DC-programming. Likewise, \cite{Kuang2015} explored symmetric nonnegative factorizations but ignored the constraint enforcing that $W W^{\T}$ is doubly-stochastic, which is crucial for cluster normalization.

Our approach uses the factorization $A_{s}(W)$ given by \eqref{eq:def_SelfAssignmentMatrix} instead of $Y Y^{\T}$ in \eqref{eq:combOptZass}. We can relate the two factorizations by identifying the factor
\begin{equation}
 Y(W)=W \gamma_{s}(W)^{-\tfrac{1}{2}},
\end{equation}
that is parametrized by assignments. While the rank constraint (b) and orthogonality constraint (c) are dropped for $s=0$, the constraints (a) and (d) are `built in' by construction of
\begin{equation}
 Y(W)=W C(W)^{-\tfrac{1}{2}} \, \geq \, 0,
\end{equation}
which results in a completely positive and doubly stochastic factorization. 

Conversely, for $s=1$, spectral properties are retained (cf.~\cref{sec:Related-Work:Spectral}). The orthogonality constraint (c) which implies the rank constraint (b), holds for
\begin{equation}
Y(W)=W (W^{\T}W)^{-\tfrac{1}{2}} \, \in \, \mrm{St}(c,n),
\end{equation}
whereas constraints (a) and (d) are ignored. This agrees with the observation that the constraints (a) and (c) cause the combinatorial difficulty of formulation \eqref{eq:combOptZass}, which renders them to be mutually exclusive ((a) ``physical quantity'' vs. (c) ``exclusive decisions''). However, by \cref{def:SelfAssignmentMatrix} of the one-parameter family of self-assignment matrices, we can smoothly interpolate between combinatorial and spectral properties.

Furthermore, optimization is achieved by a single \textit{smooth and continuous} process, the self-assignment flow \eqref{eq:self-assignment-flow}, which enables to apply numerous discrete numerical schemes \cite{Zeilmann:2018aa}, all of which respect the constraints.  Finally, geometric regularization within local neighborhoods of each vertex of the underlying graph through the similarity map \eqref{eq:def-Si} enforces the formation of `natural' clusters, whenever assigning the same label to close vertices is more likely to be correct.

\section{Experiments}\label{sec:Experiments}
In this section, we demonstrate and evaluate the performance of the proposed one-parameter family \eqref{eq:self-assignment-flow} of \textit{self-assignment flows (SAF)} for unsupervised data labeling, using various datasets and feature spaces (\cref{fig:inputImages}).

After describing specific details of the implementation (\cref{sec:Implementation-Details}), we report the study of the two model parameters in \cref{sec:Influence-Parameters}, and the influence of affinity matrix sketching for data reduction in a preprocessing step, to make learning from large data sets computationally feasible. In \cref{sec:Comparison-Other-Methods}, we compare our approach to various methods: basic clustering, normalized spectral cuts with spatial regularization, and partitioning using a variational decomposition of the piecewise constant Mumford-Shah model. We focus on an attractive application of our approach in \cref{sec:Experiments-Patches}: Learning patch dictionaries using the SAF based on a locally invariant distance function. Finally, as a sanity check, we report the application of the SAF to problem data on a graph from a domain that is unrelated to image analysis, to substantiate our claim that our approach applies to any data given on any graph, in principle.

\subsection{Implementation Details}\label{sec:Implementation-Details}
Throughout this paper, the SAF \eqref{eq:self-assignment-flow} was numerically integrated using the geometric explicit Euler scheme \eqref{eq:geom-integration} with step-size $h=0.1$, as described in \cref{sec:Geometrical-Integration}. For parameter values $s \in (0,1]$, we applied \eqref{eq:WTW-pseudoinverse} to avoid numerical problems when the effective number of clusters $\hat c < c$ (\cref{def:hat-c}) actually was smaller than $c$. The SAF with $s=0$ does not encounter any such problems, due to the different normalization involved in \eqref{eq:Def_AW}.
We adopted from \cite{Astrom:2017ac} the numerical renormalization step for the assignments with $\varepsilon=10^{-10}$, to avoid numerical issues for assignments very close to the boundary of the assignment manifold. Numerical integration was terminated when the average entropy of the assignments dropped below the threshold of $10^{-3}$, which indicates that the current iterate is very close to an almost unique assignment (labeling) $W^{(k)} \in \mc{W}_{\ast}^{\hat{c}}$.
\begin{figure}[htpb]
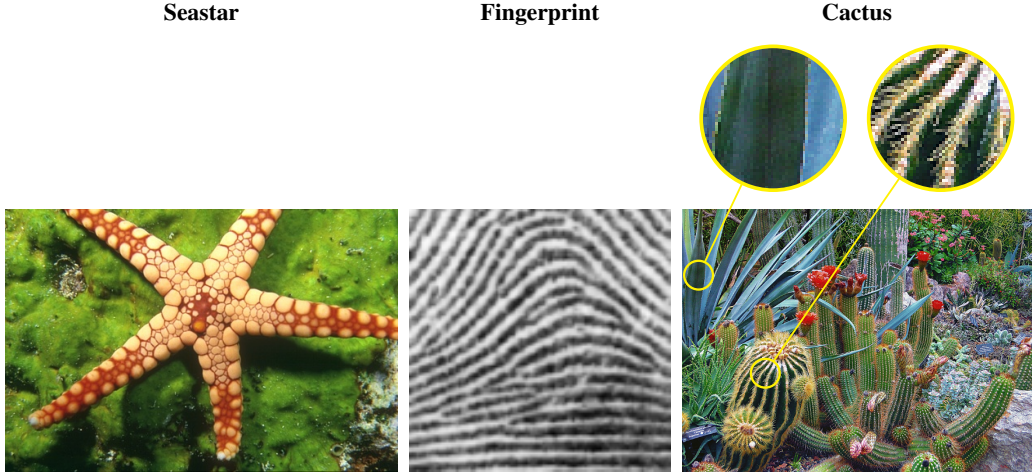

	\begin{center}
		\showExpInput{0.16}{0.2611}
	\end{center}
		\caption{Input image data used in the numerical experiments (\cref{fig:Seastar,fig:CactusI,fig:CactusII,fig:Fingerprint,fig:Fingerprint100}). Close-up views enable to compare the influence of model parameters on local image structure in comparison to alternative approaches from related work. Both the Euclidean RGB-space and locally invariant patch spaces were used as feature spaces. Regarding the latter, additional real image data are processed in \cref{fig:AbandonedTrain,fig:AbandonedEval}. The results of graph network data are depicted by \cref{fig:Football} in order to highlight that our approach more generally applies to data on graphs, beyond image feature data.}
		\label{fig:inputImages}
\end{figure}

Unless specified otherwise, the default value $\rho = 0.1$ (distance normalization in \eqref{schnoerr-eq:def-Li}) and uniform weights $w_{i,k}=1/|\mc{N}_{i}|$ \eqref{eq:weights-Omega-i} for assignment regularization were used in all experiments, with  neighborhoods $\mc{N}_{i}$ of equal size
\begin{equation}\label{eq:def-mcN}
|\mc{N}| := |\mc{N}_{i}|,\quad \forall i \in \hat{\mc{I}},
\end{equation}
for interior pixels $\hat{\mc{I}} \subset \mc{I}$.

Data $\mc{F}_{n}$ were embedded using the standard Gaussian kernel \eqref{eq:K-mcF-Gaussian} with parameter $\sigma=\sqrt{0.1}$, in order to compute the affinity matrix $K_{\mc{F}}$ \eqref{eq:K-mcF-entries}. For larger datasets, a sketch of $K_{\mc{F}}$ was used as described in \cref{sec:sketching}, with parameters $q=1$ and $\ell=100$ random samples drawn without replacement;  see \cref{sec:Influence-Parameters} for a validation. Finally, the initial value  $W(0)$ of \eqref{eq:self-assignment-flow-a} was chosen as small perturbation of the barycenter \eqref{eq:self-assignment-flow-a} with $\varepsilon = 10^{-2}$ and initial distance matrix $D_{\mc{F},0}$, computed with the inexpensive greedy $k$-center clustering algorithm, as explained in \cref{sec:greedy-clustering}.

\subsection{Influence of Model Parameters}\label{sec:Influence-Parameters}
The self-assignment flow (SAF) has three model parameters: The parameter $s$ of the self-assignment matrix $A_{s}(W)$ \eqref{eq:def-As}, the neighborhood size $|\mc{N}|$ controlling the \textit{scale} of regularization, and the upper bound $c$ on the effective number $\hat c$ of labels \eqref{eq:def-hat-c}.
\subsubsection{Influence of $s$, $|\mc{N}|$ and $c$} \label{sec:influence-s-mcN-c}
\Cref{fig:Seastar} shows both labelings and recovered prototypes below each panel, depending on $s$ and $|\mc{N}|$. We set $c=16$ which is sufficiently large, since $\hat c < c$ quickly happens when lowering $s$ even at the smallest scale of $3 \times 3$ pixels. $\hat c$ further drops down with larger scale. Regarding the parameter $s$, we observe:
\begin{description}
\item[Small $s$] Spatial regularization is more aggressively enforced, leading to compact codes in terms of smaller numbers $\hat c$ of prototypes.
\item[Large $s$] Distances in the feature space have more impact. Local image structure is better preserved at the cost of a larger number $\hat c$ of prototypes.
\end{description}
The second observation underlines the relation of the self-assignment flow, for $s=1$, to spatially regularized normalized cuts as worked out in \cref{sec:Related-Work:Spectral}.

\Cref{fig:Seastar} illustrates that depending on the application, the properties of the SAF can be continuously controlled by setting the parameter $s$, thanks to the geodesic interpolation \eqref{eq:def_SelfAssignmentMatrix}.

\begin{figure}[htpb]
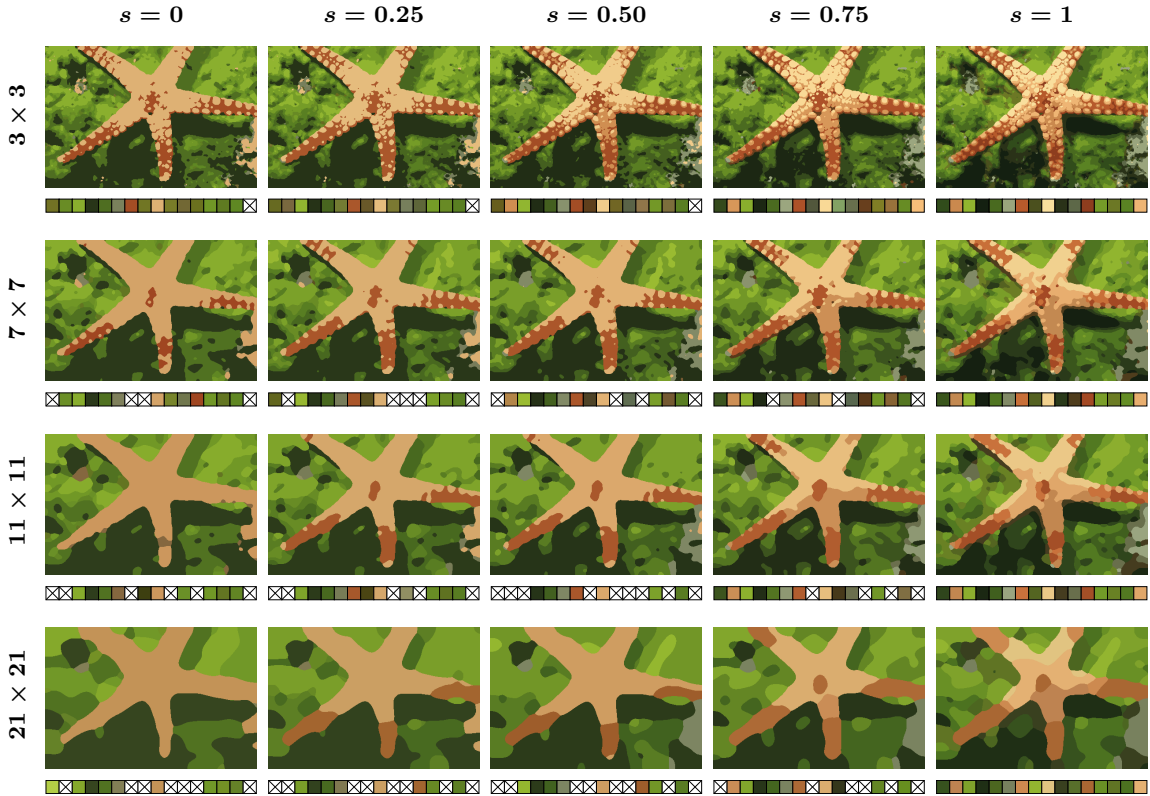

	\begin{center}
		\showExpRGBi{0.180}{12003}
	\end{center}
		\caption{Influence of the model parameters $s \in [0,1]$ parametrizing the SAF in terms of the self-assignment matrix \eqref{eq:def_SelfAssignmentMatrix}, the neighborhood size $|\mc{N}|$ controlling the scale of spatial regularization, and the effective number $\hat c \leq c=16$ of labels. Recovered prototypes are displayed below each labeling and aligned to each other (using linear assignment of the clusters) to ease visual comparison. Prototypes that `died out' are marked by a cross. We observe that due the geodesic interpolation \eqref{eq:def_SelfAssignmentMatrix}, the influence of spatial regularization (small $s$: compact image codes) relative to the influence of distances in the feature space (large $s$: preserving local image structure) can be continuously controlled.}
		\label{fig:Seastar}

\end{figure}
\subsubsection{Evolution of Cluster Sizes, Entropy, and Rank Lower Bound}
\Cref{fig:SeastarPlots} illustrates the \textit{evolution} of the SAF in terms of the following measurements.
\begin{description}
\item[Cluster sizes] For smaller values of $s$, more iterations are required for cluster formation. This conforms with the observation in \cref{sec:influence-s-mcN-c} that the SAF then promotes spatial regularization.
Conversely, larger values of $s$ yield more balanced (uniform) cluster sizes. This is consistent with the observation made in \cref{sec:influence-s-mcN-c} that, in this case, the SAF more carefully explores the feature space and preserves local image structure.
\item[Average entropy] The panels illustrate that the initial assignment is an $\veps$-perturbation of the barycenter on the assignment manifold, and that the termination criterion was reached in all experiments. In agreement with the preceding point, the SAF converges faster for larger values of $s$.
\item[Rank lower bound] The third row of \cref{fig:SeastarPlots} displays the lower bound $\tr \big( B(W^{(k)}) \big)$ of $\rank(W^{(k)})$ due to \cref{prop:BW}(d). After termination of the SAF, this lower bound becomes sharp at $W \in \mc{W}_{\ast}^{\hat{c}}$ and attains the number $\hat{c}$ of effective prototypes.
\end{description}

\begin{figure}[htpb]
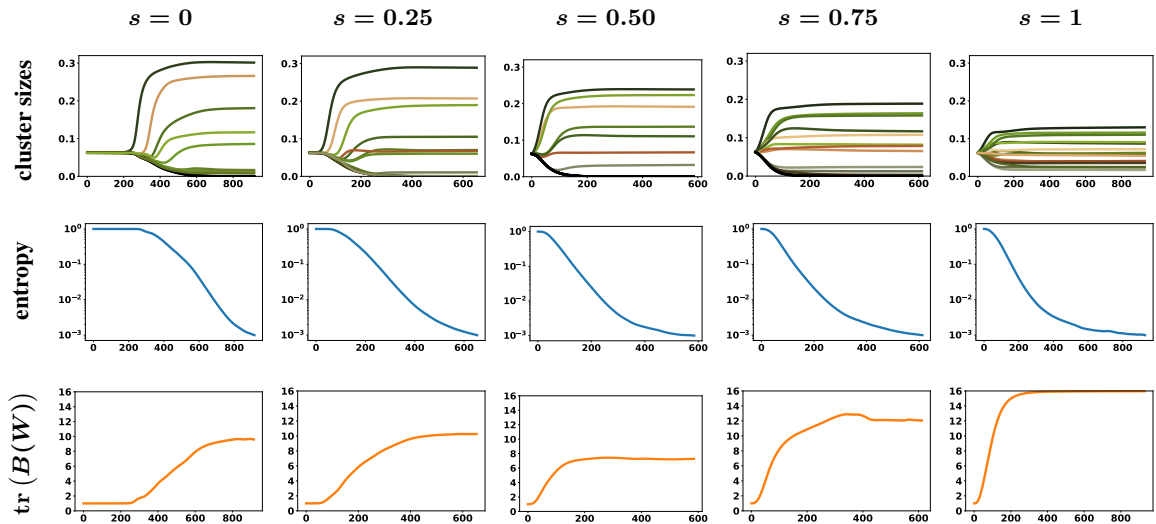

	\begin{center}
		\showExpRGBLabelHist{0.18}{12003}{11}
		\caption{Evolution of relative cluster sizes, average entropy and lower bound of $\rank(W^{(k)})$ as a function of the SAF, depending on the iterations $k$ for the experiment with $|\mc{N}|=11 \times 11$ depicted by \cref{fig:Seastar}. \textsc{top:} Smaller values of $s$ promote spatial regularization. Hence more iterations are required to form clusters. Larger values of $s$ yield more uniform cluster sizes which reflects the stronger influence of feature similarity and the preservation of local image structure. \textsc{center:} The average entropy illustrates the random initialization $\veps$-close to the barycenter and that the termination criterion is reached in all experiments. The entropy decays faster for larger values of $s$.  \textsc{bottom:} The lower rank bound due to \cref{prop:BW}(d) becomes sharp when the SAF terminates at some labeling $W \in \mc{W}_{\ast}^{\hat{c}}$ and attains the number $\hat c$ of effective labels.}
		\label{fig:SeastarPlots}
	\end{center}
\end{figure}
\subsubsection{Influence of Affinity Matrix Sketching}
We evaluate the influence of sketching the data affinity matrix $K_{\mc{F}}$ in a preprocessing step, as described in \cref{sec:sketching}, using the parameter value $q=1$ and varying sample sizes $\ell$.

To this end, we focused on the experiment with $s=0$, $|\mc{N}|=3\times 3$ depicted by \cref{fig:Seastar} and compared the labelings obtained with and without sketching $K_{\mc{F}}$. To handle the latter case where $K_{\mc{F}}$ requires $\approx 177$ GB of memory, we computed on the fly the entries  for every matrix-vector multiplication on GPUs using the software library KeOps\footnote{B. Charlier, J. Feydy, and J.-A. Glaunès, \textit{KeOps Kernel Operations on the GPU}, 2018, \\ https://www.kernel-operations.io/keops/index.html}, rather than holding the matrix in memory.

\Cref{fig:NystroemError} displays the relative error of different label assignments after sketching, depending on the sample size $\ell$, where $100\%$ corresponds to all $n = 321\times 481$ columns of $K_{\mc{F}}$. For each value $\ell$, 100 runs were made using different random seeds. \Cref{fig:NystroemError} displays the \textit{average} error along with the standard deviation. The corresponding curves show that $\ell=100$ samples, i.e.~merely $0.065\%$ of all data points, suffice to eliminate the effect of data reduction by sketching the input affinity matrix.

\begin{figure}[htpb]
	\begin{center}
		\includegraphics[width=0.8\textwidth]{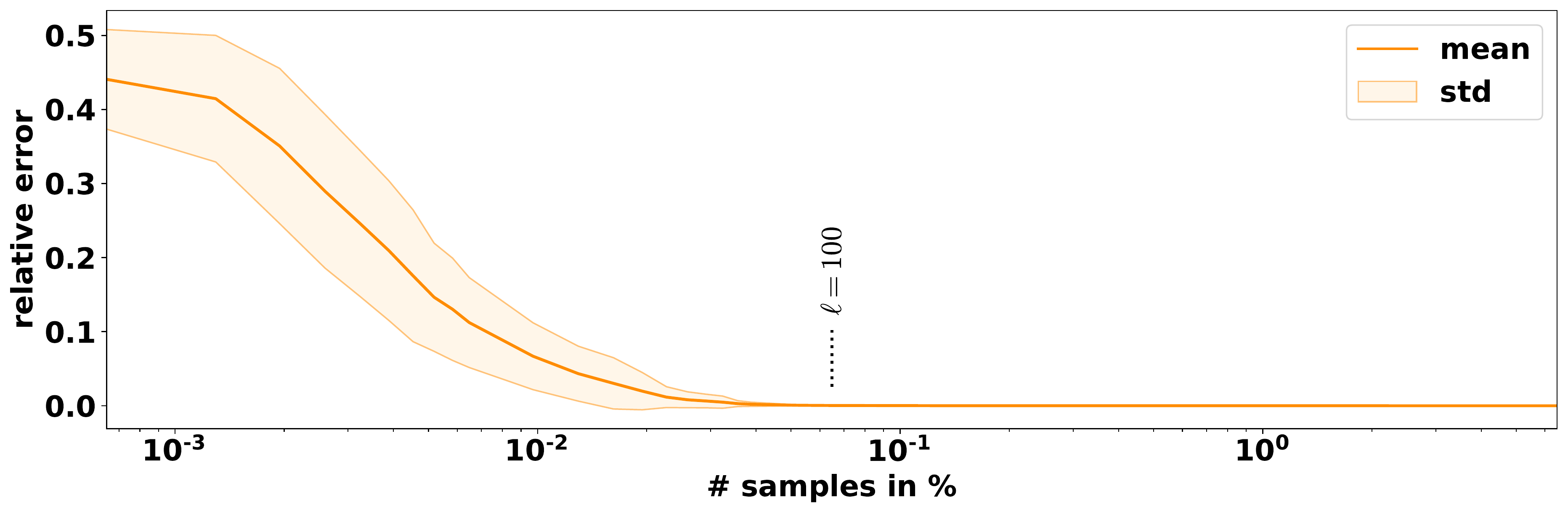}
		\caption{This plot shows the \textit{average relative labeling error} together with the standard deviation, that result from data reduction by sketching the data affinity matrix $K_{\mc{F}}$ in a preprocessing step. for SAF is approximated by the matrix sketching method (see \cref{sec:sketching}) in dependency of the number of sampled pixels $l$ represented in $\%$. The curves show that merely $0.065 \%$ of all data points (corresponding to $\ell=100$ randomly sampled columns of $K_{\mc{F}}$) suffice to eliminate the effect of data reduction. }
		\label{fig:NystroemError}
	\end{center}
\end{figure}

\subsection{Comparison to Other Methods}\label{sec:Comparison-Other-Methods}
We compared the SAF to the following methods:
\begin{description}
\item[Nearest neighbor clustering] \textbf{$k$-means} and \textbf{$k$-center} clustering (no spatial regularization), to show the influence of spatial regularization performed by the SAF on both labeling and prototype formation;
\item[AF] \textit{supervised} assignment flow \cite{Astrom:2017ac} with spatial regularization, using \textit{fixed} prototypes computed beforehand using nearest neighbor clustering, to highlight that the SAF \textit{simultaneously} performs unsupervised label \textit{learning} and label \textit{assignment};
\item[Spectral clustering] We computed partitions using normalized spectral cuts \cite{Shi2000} after augmenting feature vectors by spatial coordinates $x_{i},\,i \in \mc{I}$ for spatial regularization. The resulting data affinity matrix was given by
\begin{equation}\label{eq:k-mcF-spectral}
   {K_{\mc{F}}}_{i,k} = \exp \big(-( \tfrac{1}{\sigma^{2}} \| f_i - f_k \|^{2}_{2}+ \alpha \| x_i - x_k \|^{2}_{2} ) \big),\quad i,k \in \mc{I},
\end{equation}
with parameter $\alpha >0$ controlling the influence of spatial regularization.
\item[Fast partitioning] A variational decomposition of the piecewise-constant Mumford-Shah approach to image partitioning proposed by \cite{Storath2014}, using the publicly available implementation ``Pottslab'' from the authors. The method operates directly on the values in the feature space instead of using a reformulation with labels. Therefore, the number of clusters can be large. For this reason, we  applied an additional $k$-means clustering step to the (over-segmented) results in order to have a direct comparison in terms of labels and prototypes.
\end{description}
Two variants of the SAF were evaluated for comparison: (i) using \textit{uniform} weights for spatial regularization; (ii) using \textit{nonuniform} weights determined in "non-local means fashion" by
\begin{equation} \label{eq:non-local-means-weights}
   w_{i,k} = \frac{\tilde{w}_{i,k}}{\la \tilde{w}_{i}, \eins_{n} \ra} \quad \textnormal{ with } \quad \tilde{w}_{i,k} =
   \begin{cases}
   \exp \big( -\frac{1}{\rho}\| P_i - P_k \|^{2}_{F} \big), &\text{if}\; k \in \mc{N}_{i}, \\
   0, &\text{else},
   \end{cases}
\end{equation}
where $P_i$ denotes the patch centered at pixel $i$. Throughout, the patch size as well as the neighborhood size $|\mc{N}|$ for geometric averaging was chosen to be $5 \times 5$ pixels.

The user parameters of all other methods were manually tuned so as to obtain best comparable results.

\begin{figure}[tpb]
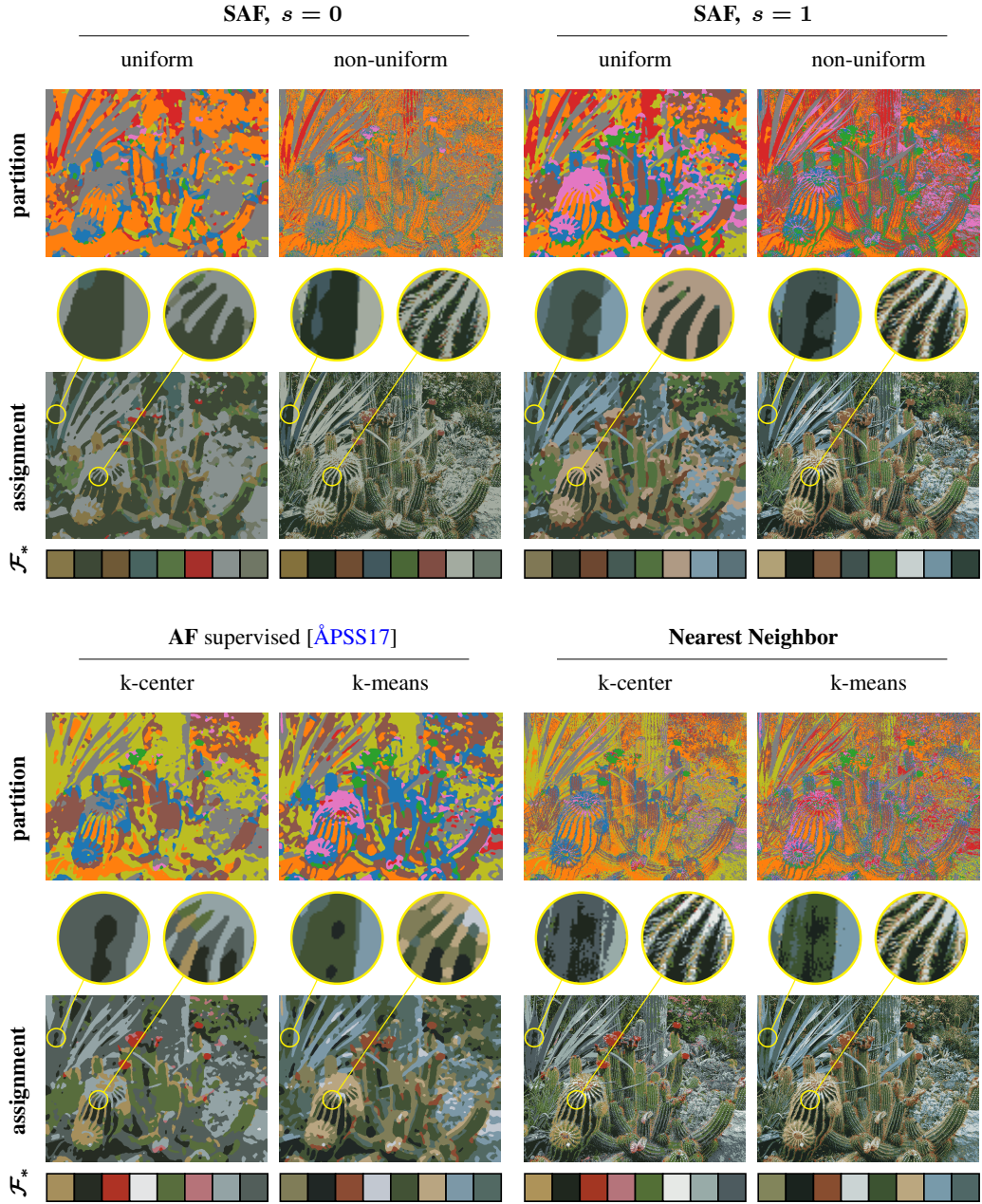

	\begin{center}
		\showExpRGBiiia{0.2}
	\end{center}
		\caption{Comparison of the SAF to nearest neighbor clustering and supervised assignment flow (AF). Inspecting the results and the close-up views shows: Nearest neighbor clustering yields noisy label assignments due to the absence of spatial regularization. The AF returns spatially coherent partitions that may locally look unnatural (see close-up views), since the prototypes are fixed and do not adapt to the spatial components of the resulting partition. The unsupervised SAF learns labels adaptively during label assignment. The resulting partitions have a natural spatial structure with increased details if $s=1$. The latter effect is considerably enhanced, independent of $s$, when nonuniform weights are used.}
		\label{fig:CactusI}

\end{figure}

\subsubsection{Nearest Neighbor Clustering, Supervised Assignment Flow}
\Cref{fig:CactusI} displays the results obtained using the SAF, the supervised assignment flow (AF) and nearest neighbor clustering, respectively. The close-up view of the results of nearest neighbor clustering shows noisy label assignments even in homogeneous regions, due to the absence of spatial regularization. By contrast, the AF returns spatially coherent labelings. However, since the labels (prototypes) are fixed beforehand, their assignments yield partitions that may locally look unnatural (see close-up views). Note that the prototypes displayed for the AF were recomputed after convergence from the resulting partition and, therefore, differ from the nearest neighbor prototypes that were used as input labels for computing the AF.

In comparison with these methods, the SAF yields more natural partitions due to forming the labels \textit{during} label assignment and preserves fine structure for $s=1$, in agreement with the experiments discussed in \cref{sec:Influence-Parameters}. This latter effect is considerably enhanced when nonuniform weights are used, independently of $s$, without compromising the quality of the spatial structure of the resulting partitions.

\begin{figure}[tpb]
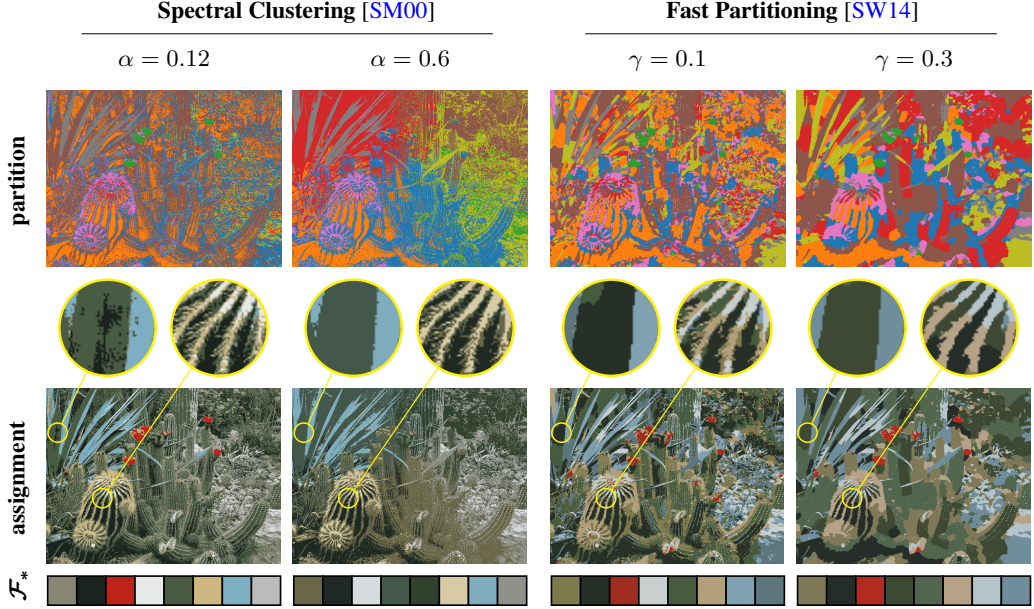

	\begin{center}
		\showExpRGBiiib{0.2}
		\caption{Comparison of the SAF to spectral clustering using feature vectors augmented by spatial coordinates and normalized cuts, and to fast partitioning that approximates the piecewise constant Mumford-Shah model. Spatial regularization as performed by spectral clustering is clearly suboptimal, since weak regularization returns noisy partitions where strong regularization yields biased clusters (e.g.~red cluster). See the last paragraph of \cref{sec:Related-Work:Spectral} for an explanation. Fast partitioning yields good labelings but does not consistently enforce the scale of spatial regularization through the choice of $\gamma$ -- see, e.g.~the small red clusters in the panel on the right-hand side. This reflects that fast partitioning directly operates on the feature space rather then separating data representation from inference, as does the SAF.}
		\label{fig:CactusII}
	\end{center}
\end{figure}
\subsubsection{Spatial Feature Augmentation and Normalized Spectral Cuts}
\Cref{fig:CactusII} displays the corresponding results for spectral clustering and fast partitioning, respectively, using two parameter values enforcing weak and strong spatial regularization in either case.

We observe that spectral clustering is highly sensitive to the value of $\alpha$. Small values yield noisy partitions, whereas larger values yield biased partitions (e.g.~red cluster). We attribute this strange behavior to the conceptual deficiency of spatial regularization performed by feature augmentation, as discussed in the last paragraph of \cref{sec:Related-Work:Spectral}.

Fast partitioning returned the closest labelings to those computed by the SAF. The scale of spatial regularization is not consistently enforced everywhere, however, as e.g.~the small red dots on the cactus arms reveal. We attribute this to the above-mentioned fact that fast partitioning directly operates on the feature space, rather than separating data representation from inference using labels and label assignments. In addition, the variational decomposition may be susceptible to getting stuck in suboptimal minima.

\subsection{Unsupervised Learning and Assignment of Locally Invariant Patch Dictionaries}\label{sec:Experiments-Patches}
In this section, we base the self-assignment flow (SAF) on more advanced features, viz.~feature \textit{patches}, and a corresponding locally invariant distance function.

\subsubsection{Locally Invariant Patch Distances}\label{sec:locally-invariant-distances}
Let
\begin{equation}\label{eq:def-mcN-mcP}
\mc{N}_{\mc{P},i},\quad i \in \hat{\mc{I}},\qquad\qquad
n_{\mc{P}} := |\mc{N}_{\mc{P},i}|,\quad \forall i
\end{equation}
denote quadratic sections centered at pixel (vertex) $i$ of the underlying image grid graph, with uniform size $n_{\mc{P}}=2 k+1$ for some $k \in \N$, for every $i$. We only consider region centers at interior grid points $i \in \hat{\mc{I}} \subset \mc{I}$ such that no section $\mc{N}_{\mc{P},i}$ extends beyond the boundary of the graph, which implies
\begin{equation}
\mc{N}_{\mc{P},i} \subset \mc{I},\quad \forall i \in \hat{\mc{I}}.
\end{equation}

\begin{figure}[htpb]
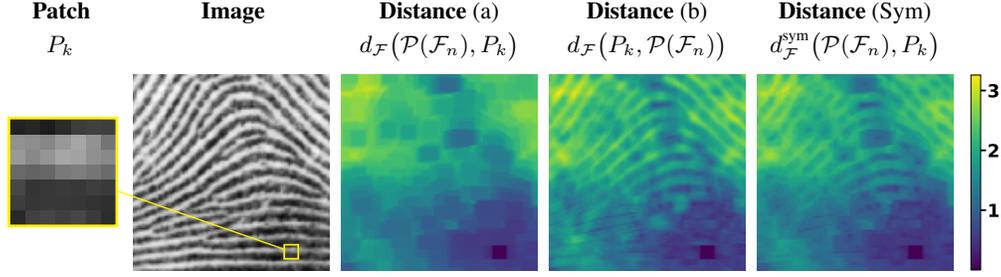

	\begin{center}
		\showPatchDist{0.12}
	\end{center}
		\caption{Visualization of the distance functions \eqref{eq:patch-dist-asym} and \eqref{eq:patch-dist-sym}) evaluated for a single patch $P_{k}$ and all patches $\mc{P}(\mc{F}_{n})$ of size $n_{\mc{P}} = 7 \times 7$ of the depicted image.  The evaluation of distance (a) amounts to determine the minimal distance of $P_{k}$ to \textit{all} equivalence classes of patches generated by the patches of the entire image. As a consequence, equivalence classes close to $P_{k}$ generate the `blocky' graph of the distance function. Conversely, evaluation of distance (b) amounts to compare the \textit{single} equivalence class generated by $P_{k}$ to all image patches. As a consequence, the graph of the distance function reflects the original image structure in more detail. The symmetric distance (rightmost panel) is the pointwise minimum of distance (a) and (b). It is apparent that neither distance (a) nor (b) dominates the other distance.}
		\label{fig:PatchDist}
\end{figure}

We define a \textit{patch centered at pixel} $i$ as the ordered tuple of data points
\begin{equation}\label{eq:def-patch-Pi}
	P_{i}=\big(f_{k_1}, \dots, f_{i}, \dots, f_{k_{n_{\mc{P}}}} \big), \qquad k_{1},\dots, k_{n_{\mc{P}}} \in \mc{N}_{\mc{P},i},\quad i \in \hat{\mc{I}},
\end{equation}
where the particular chosen order does not matter, but should be fixed for all patches. The individual patch features are denoted by
\begin{equation}
P_{i;m} = f_{m},\quad m \in \mc{N}_{\mc{P},i}
\end{equation}
and the collection of all patches induced by the data $\mc{F}_{n}$ is denoted by
\begin{equation}\label{eq:all-image-patches}
	\mc{P}(\mc{F}_{n})=\big\{ P_{i} \in \mc{F}_{n}^{n_{\mc{P}}} \colon \, i \in \hat{\mc{I}} \big\}.
\end{equation}

In order to define invariant distance functions,
we consider the dihedral group
\begin{equation}
	 \mc{D}_{4} =\Big\{
	 \begin{psmallmatrix} 1 & 0  \\ 0 & 1  \\ \end{psmallmatrix},
	 \begin{psmallmatrix} 0 & -1  \\ 1 & 0  \\ \end{psmallmatrix},
	 \begin{psmallmatrix} -1 & 0  \\ 0 & -1  \\ \end{psmallmatrix},
	 \begin{psmallmatrix} 0 & 1  \\ -1 & 0  \\ \end{psmallmatrix},
	 \begin{psmallmatrix} -1 & 0  \\ 0 & 1  \\ \end{psmallmatrix},
	 \begin{psmallmatrix} 1 & 0  \\ 0 & -1  \\ \end{psmallmatrix},
	 \begin{psmallmatrix} 0 & 1  \\ 1 & 0  \\ \end{psmallmatrix},
	 \begin{psmallmatrix} 0 & -1  \\ -1 & 0  \\ \end{psmallmatrix}
	 \Big\} \subset \mc{O}(2)
\end{equation}
generated by the following elements of the two-dimensional orthogonal group $\mc{O}(2)$: four two-dim\-en\-sional rotations by $\{0^{\circ},90^{\circ},180^{\circ},270^{\circ}\}$ and the two reflections with respect to the local coordinate axes, using the center pixel as origin. Since local grid coordinates are mapped onto each other,
we can identify each transformation of the group $\mc{D}_{4}$ with a corresponding permutation $\sigma$ of the pixel locations within the patch domain. Accordingly, writing with abuse of notation $\sigma \in \mc{D}_{4}$, the corresponding transformed patch \eqref{eq:def-patch-Pi} is given and denoted by
\begin{equation}\label{eq:Tsigma-D4}
T_{\sigma} P_{i}=\big(f_{\sigma(k_1)}, \dots, f_{i}, \dots, f_{\sigma(k_{n_{\mc{P}}})} \big) \qquad k_{1},\dots, k_{n_{\mc{P}}} \in {\mc{N}}_{\mc{P},i},\qquad \sigma \in \mc{D}_{4}.
\end{equation}
We point out that no interpolation is required to compute these patch transformations.

In addition to the transformations \eqref{eq:Tsigma-D4}, we consider all translations $P_{i} \mapsto P_{k},\, k \in \hat{\mc{N}}_{\mc{P},i}$ of patch $P_{i}$ mapping the center location $i$ to some grid location $k$ within its own region $\hat{\mc{N}}_{\mc{P},i} := \mc{N}_{\mc{P},i} \cap \hat{\mc{I}}$  restricted to interior pixels.
We factor out these $|\mc{D}_{4}| \cdot n_{\mc{P}} = 8 \cdot n_{\mc{P}}$ degrees of freedom by considering all corresponding transformations of patch $P_{i}$ as \textit{equivalent}. These equivalence classes of patches provide the basis for invariant patch distances as defined next.

We define the \textit{asymmetric patch distance} between two patches centered at pixel $i \in \hat{\mc{I}}$ and $k \in \hat{\mc{I}}$ by
\begin{equation}\label{eq:patch-dist-asym}
	d_{\mc{F}}(P_{i}, P_{k}) = \min_{\substack{ \sigma \in \mc{D}_{4} \\ j \in \hat{\mc{N}}_{\mc{P},i}   }} \, \sum_{m \in [n_{\mc{P}}]} d_{\mc{F}}\big((T_{\sigma} P_{j})_{m}, P_{k;m}\big)
\end{equation}
and the \textit{symmetric patch distance} by
\begin{equation}\label{eq:patch-dist-sym}
	d^{\textnormal{sym}}_{\mc{F}}(P_{i}, P_{k}) = \min \big\{d_{\mc{F}}(P_{i}, P_{k}), d_{\mc{F}}(P_{k}, P_{i}) \big\}.
\end{equation}
\Cref{fig:PatchDist} illustrates these locally invariant distance functions.

\subsubsection{Recovery of Patch Prototypes and Images}
Distance \eqref{eq:patch-dist-sym} defines the affinity matrix \eqref{eq:K-mcF-entries} by \eqref{eq:K-mcF-Gaussian} and in turn the likelihood map \eqref{eq:def-Lsi} and the similarity map \eqref{eq:def-Si}. As a consequence, the self-assignment flow can be integrated to obtain the assignment $W(t)$. We focus in this section on the recovery of prototypical patches and on `explanations' of input images by assigning these prototypical patches. The corresponding results are illustrated by numerical examples in the subsequent \cref{sec:patch-based-SAF,sec:patch-assignment-novel}.

According to \cref{sec:Recovery-Prototypes}, prototypical patches representing each cluster are determined as weighted averages
\begin{equation}\label{eq:prototypical-patches}
P_{j}^{\ast} = \operatornamewithlimits{\arg \min}_{P \in \mc{P}(\mc{F})} \sum_{i \in \mc{\hat{I}}} \big(C(W)^{-1} W^{\T} \big)_{j,i} d^{2}_{\mc{F}}(P_{i}, P),\qquad j \in \mc{J},
\end{equation}
with respect to the \textit{asymmetric} patch distance \eqref{eq:patch-dist-asym}, since the prototypical patch $P \in \mc{P}(\mc{F})$ is not contained in the set of all image patches $\mc{P}(\mc{F}_{n})$ \eqref{eq:all-image-patches}.

Using these prototypes, the corresponding image is computed as follows. For each prototypical patch $P_{j}^{\ast}$, the optimal transformation for the assignment to pixel $i$ is determined as
\begin{equation}\label{eq:optimal-transformations}
(\sigma_{i,j}^{\ast}, l_{i,j}^{\ast}) = \operatornamewithlimits{\arg \min}_{\substack{ \sigma \in \mc{D}_{4} \\ l \in \hat{\mc{N}}_{\mc{P},i}  }} \, \sum_{m \in [n_{\mc{P}}]} d_{\mc{F}}\big((T_{\sigma} P_{l})_{m}, P^{\ast}_{j;m}\big).
\end{equation}
Using these transformations, a prototypical patch is assigned to every pixel $i \in \hat{\mc{I}}$. This implies that, for each pixel $i$, patches assigned to pixels $j \in \mc{N}_{\mc{P};i}$ may assign a corresponding patch entry to pixel $i$. Averaging these entries, normalized by the number of values contributed to pixel $i$, defines the restored image value at pixel $i$.

\begin{figure}[tpb]
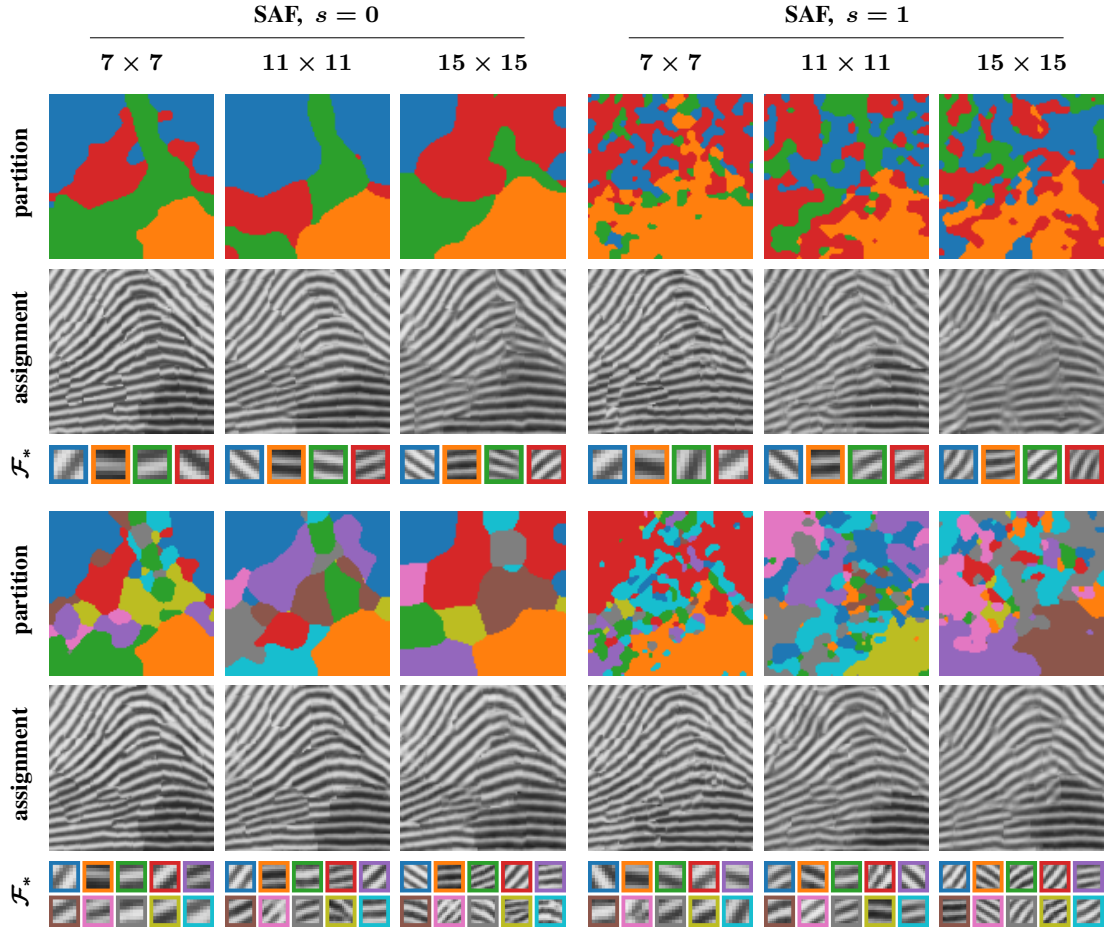

	\begin{center}
		\showExpPatchesII{0.14}{fingerprint2}{3}
	\end{center}
		\caption{Determination of locally invariant patch prototypes, their assignment to the original image data and the corresponding partitions (depicted with pseudo-colors), using the SAF ($s=0$ and $s=1$), different patch sizes ($7\times 7$, $11\times 11$, $15\times 15$) and numbers of prototypes ($c=4$ and $c=10$). The underlying transformation group enables accurate image representations even with $c=4$ patches only, provided the patch size is close to the spatial scale of local image structure (here: $7 \times 7$ pixels). This performance deteriorates for larger patch sizes. The SAF with $s=0$ yields partitions that are spatially more regular than the partitions computed with $s=1$, since the latter tend to cover the feature space more uniformly, in agreement with the result depicted by \cref{fig:Seastar}.}
		\label{fig:Fingerprint}
\end{figure}
\begin{figure}[tpb]
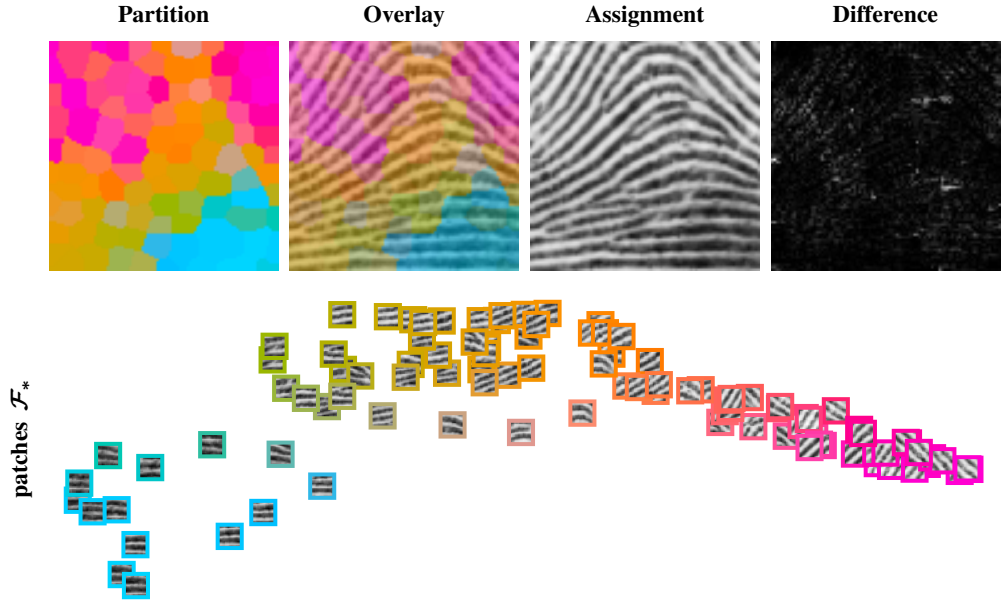

	\begin{center}
		\showPatchFingerprintVisu{0.14}{0.786}
	\end{center}
		\caption{Experiment of \cref{fig:Fingerprint} repeated with a larger patch dictionary leads to a detailed representation of local image structure. Although overlapping regions of assigned prototypical patches are averaged at each pixel in order to restore an image, the result `Assignment' is quite close to the input data `Image' of \cref{fig:PatchDist}, due to using the locally invariant patch distance. Panel `Difference' shows the difference as grayvalue plot (range $[0,0.3]$). The lower panel displays a 2D embedding of the learned prototypical patches. The corresponding colors indicate their assignment in `Partition' and `Overlay'. Clusters in the lower panel, e.g.~those colored pink and blue, illustrate the invariance under discrete rotations and reflections.
		}
		\label{fig:Fingerprint100}
\end{figure}
\begin{figure}[tpb]
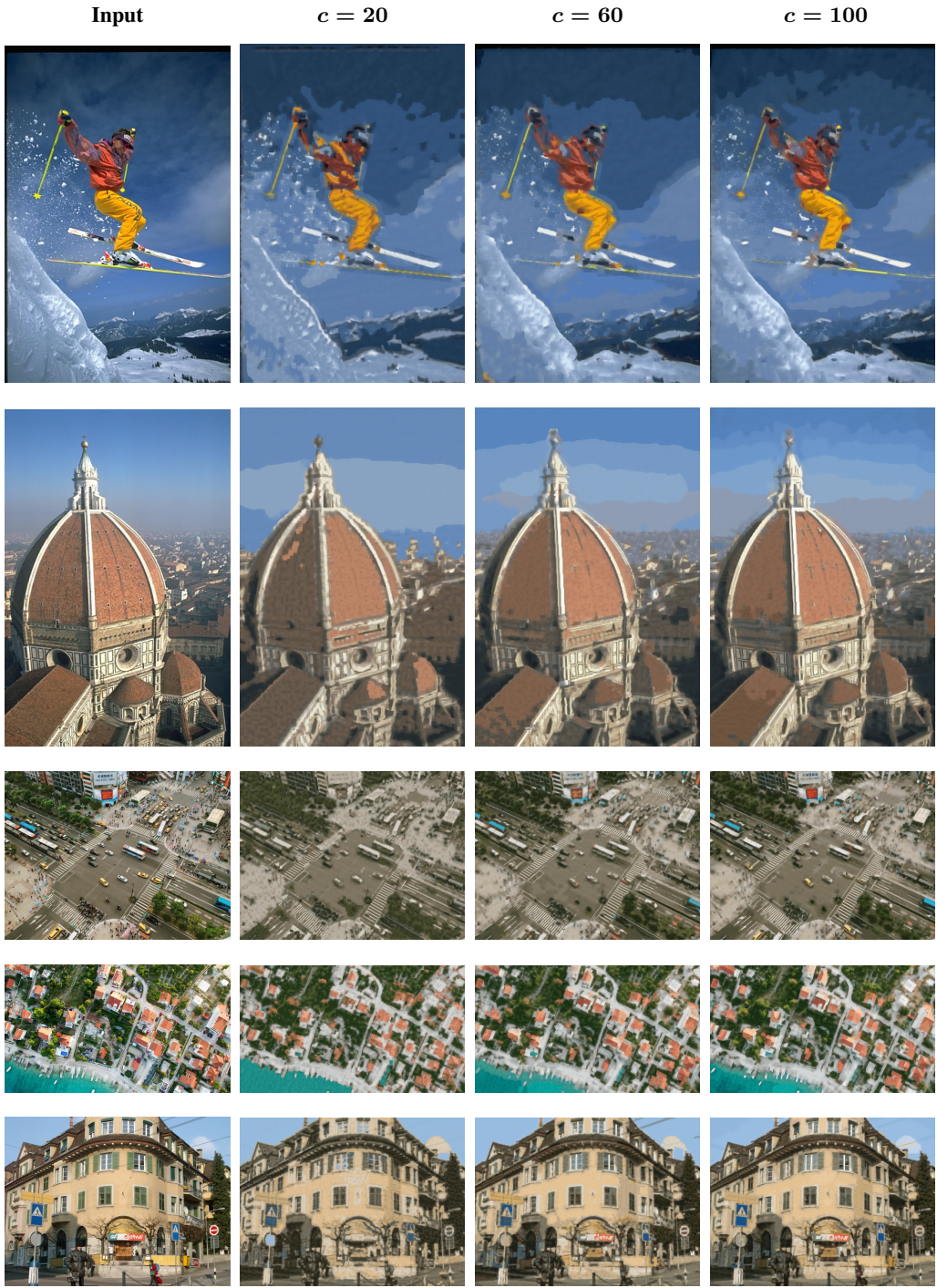

	\begin{center}
		\showNewPatchExp{0.21}
	\end{center}
		\caption{'Input' images (left-most column) are represented in a compact way by unsupervised patch learning and assignment using the SAF with $s=1$, $|\mc{N}|=3 \times 3$ for geometric regularization, and with increasing dictionary sizes $c \in \{20,60,100\}$ of locally invariant patches of size $7 \times 7$, as described in \cref{sec:Experiments-Patches}. The recovered images are shown in the remaining three columns. We observe that for more complex real-world scenarios, a larger number of patches is required for representing all local details (e.g., see the arcs of the dome in the second row). This suggests to extend the local patch invariance towards affine transformations with arbitrary rotations and scalings, which requires more expensive interpolation of the pixel-grid, however.}
		\label{fig:PatchLabeling}
\end{figure}
\begin{figure}[tpb]
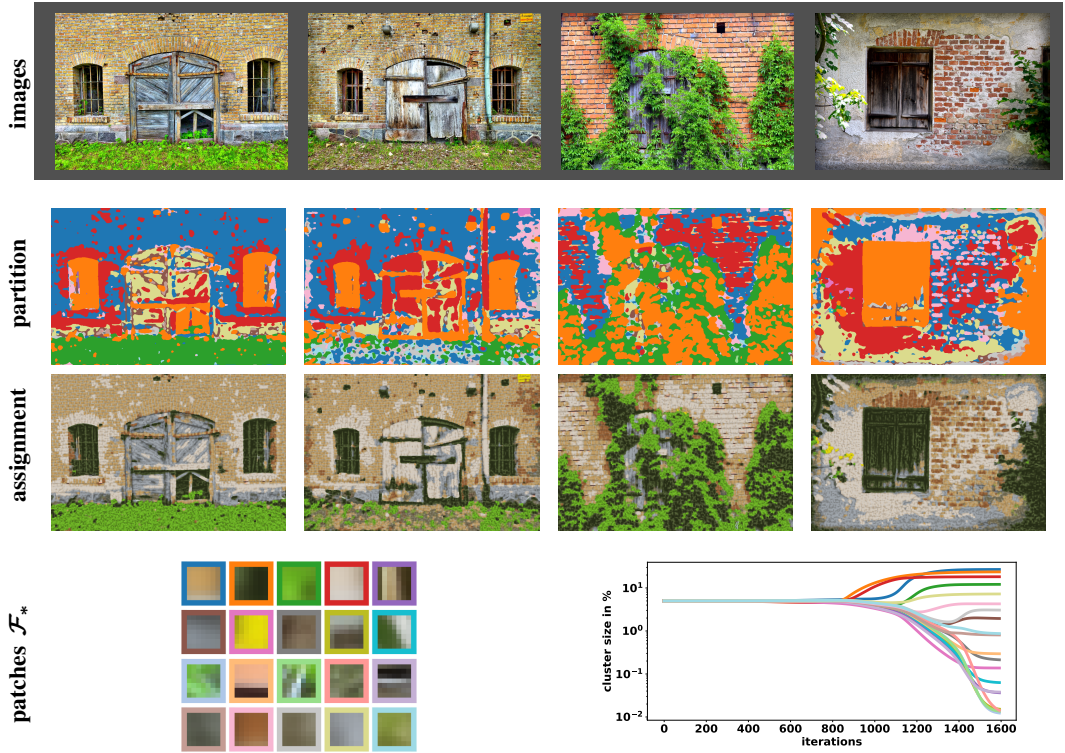

	\begin{center}
		\showExpPatchTransferTrain{0.2}{0.36}
	\end{center}
		\caption{The bottom row shows a dictionary of $c=20$ locally invariant patches of size $7 \times 7$ pixels, learned from the four images shown in the top row using the SAF with $s=0$ and $|\mc{N}|=3\times 3$ pixels. The second and third row illustrate the patch assignments with pseudo-colors and the recovered image data, respectively. Closeness of the restored images to the input data, despite the small size of the patch dictionary, demonstrates the effectiveness of the underlying discrete transformation group. The evolution of cluster sizes (bottom row, right panel) illustrates the ability of the SAF to resolve `conflicting' assignments due to mutually overlapping patches successfully, along with the formation of invariant patch prototypes, in a completely unsupervised way.}
		\label{fig:AbandonedTrain}
\end{figure}

\subsubsection{Patch-Based Self-Assignment Flow} \label{sec:patch-based-SAF}
\Cref{fig:Fingerprint} illustrates image partitions, the corresponding $c=4$ and $c=10$ prototypical patches of sizes $n_{\mc{P}} \in \{7 \times 7, 11 \times 11, 15 \times 15\}$,  their assignment to the input image data as described in the preceding section, based on integrating the SAF with $s=0$ and $s=1$ and spatial regularization parameter $|\mc{N}|=3 \times 3$.

In agreement with the discussion of the results depicted by \cref{fig:Seastar}, we observe that the SAF with $s=0$ returns partitions with a more regular spatial structure, whereas the SAF with $s=1$ tends to cover the feature space more uniformly which is achieved with partitions that have a irregular spatial structure.

The image recovered by assigning the prototypical patches exhibits relatively sharp spatial structures, despite the small number of prototypes ($c \in \{4,10\}$) and the pixel-wise averaging of grayvalues assigned by multiple patches. This illustrates that the small transformation group defined in \cref{sec:locally-invariant-distances} that does not even require image interpolation, actually is quite powerful. For example, the large blue region of the partition shown in \cref{fig:Fingerprint} that results from the SAF with $s=0$ and $7\times 7$ patches, indicates the optimal assignment of patches from a \textit{single} equivalence class only. These patches fit quite accurately to image structures with different orientations and local edge profiles. This effect deteriorates when using patch sizes that are much larger than the typical variations of local image structure, as a comparison of the results for the patch size $15 \times 15$ with $c=4$ and $c=10$ shows.

For comparison, \cref{fig:Fingerprint100} shows the result for a larger number $c=100$ of prototypes, which leads to a detailed representation of local image structure. The lower panel displays a two-dimensional embedding of the weighted graph with prototypes as patches and the similarities \eqref{eq:K-mcF-Gaussian} as weights. Representatives of equivalence classes of patches that are close to each other, are grouped together. Factoring out the group of transformations effectively copes with different edge profiles and orientations. Panel `Difference' shows the absolute difference between the input image and labeling, ranging from $0$ (black) to $0.3$ (white).

We additionally evaluated the unsupervised patch-based SAF approach using various real-world images. \Cref{fig:PatchLabeling} depicts the input data as well as the resulting patch assignments for an increasing number of labels $c \in \{20,60,100\}$.

\subsubsection{Patch Assignment to Novel Data} \label{sec:patch-assignment-novel}
We repeated the experiment illustrated by \cref{fig:Fingerprint} using the data shown in \cref{fig:AbandonedTrain}. $c=20$ locally invariant prototypical patches of size $7 \times 7$ pixels were learned from $4$ images using the SAF with $s=0$ and $|\mc{N}|=3\times 3$ pixels. The restored images shown in the third row are remarkably close to the input data (first row), despite the small size $c=20$ of the patch dictionary. This demonstrates again the effectiveness of the underlying discrete transformation group.

\Cref{fig:AbandonedEval} shows in the top row \textit{novel} image data. These four images that are semantically similar to the training images of \cref{fig:AbandonedTrain} regarding the local image structure and texture (brick/stone, door/window, grass/ivy). The corresponding partitions and recovered images solely resulted from assigning the patch dictionary depicted by \cref{fig:AbandonedTrain} to the data by the supervised assignment flow. Again, the quality of image represention using this small dictionary is remarkable, except for the stone wall texture shown in column (c) of \cref{fig:AbandonedEval}, that is not present in the training data depicted by \cref{fig:AbandonedTrain}.

\begin{figure}[tpb]
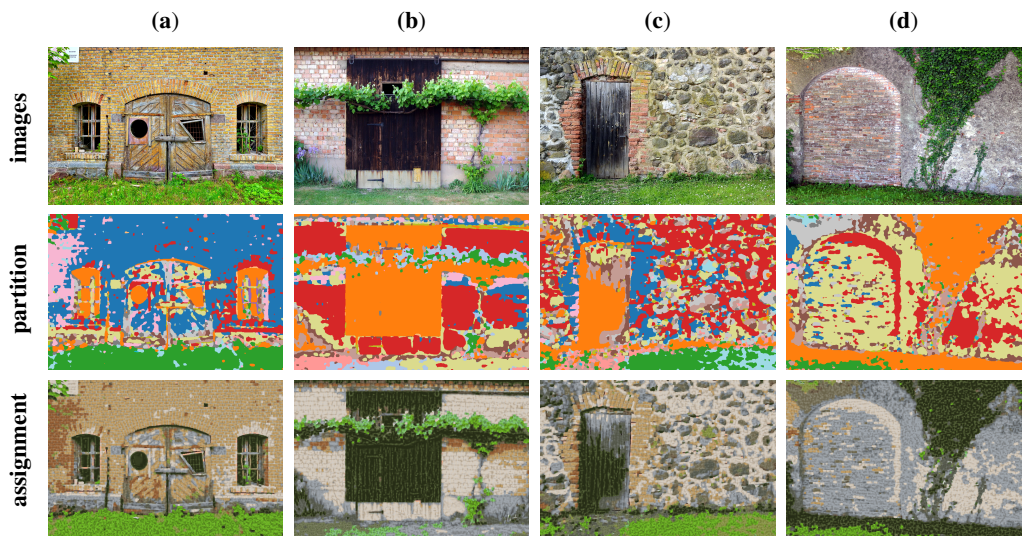

	\begin{center}
		\showExpPatchTransferEval{0.2}{0.8}
	\end{center}
		\caption{Supervised regularized assignment of the locally invariant patch dictionary from \cref{fig:AbandonedTrain} using the AF, to four \textit{novel} images (top row). Since these images are semantically similar to the training data from \cref{fig:AbandonedTrain}, the restored images are close to the input data, except for image (c) whose stone wall texture is not present in the training data. }
		\label{fig:AbandonedEval}
\end{figure}

\subsection{Regularized Clustering of Weighted Graph Data}\label{sec:Graph-Clustering}
Our approach can be applied to any data given on any undirected weighted graph. For illustration, we included an additional experiment using data not related to image analysis.

\Cref{fig:Football} shows data in terms of a weighted graph $(\mc{I}, \mc{E}, K_{\mc{E}})$ adopted from \cite{Girvan2002}. It represents the network of American football games between Division IA colleges during the regular season fall 2000. Teams are subdivided into 12 conferences, mainly based on the geographical distance, that primarily play against each other in a first period. Afterwards, the conference champions play against each other in the final games. Each node of the network represents a team. Edge weights ${K_{\mc{E}}}_{i,k}$ represent the number of games played between two teams. Labels for each vertex indicate the conference to which a team belongs, displayed by a corresponding color in \cref{fig:Football} (ground truth). We considered this labeling as ground truth for the task to partition the graph into $c=12$ classes.
The initial perturbation of the barycenter \eqref{eq:self-assignment-flow-a} in terms of a distance matrix $D_{\mc{F},0}$ was computed by assigning feature vectors to each node based on the $c$ dominant eigenvectors of $K_{\mc{E}}$, followed by greedy $k$-center clustering (\cref{sec:greedy-clustering}). Markers $\PlusMarker$ indicate nodes that were assigned to a conference different from ground truth. Weights were defined as
\begin{equation}\label{eq:graph-regularization-weights}
   w_{i,k} = \frac{\tilde{w}_{i,k}}{\la \tilde{w}_{i}, \eins_{n} \ra} \quad \textnormal{ with } \quad \tilde{w}_{i,k} = {K_{\mc{E}}}_{i,k} + \Diag(K_{\mc{E}} \eins_{n}),
\end{equation}
i.e.~by adding the total number of games played by each team to the diagonal.

The nearest neighbor assignment of the initial distance matrix contains many erroneous assignments (\cref{fig:Football}, initialization).
The results of the SAF with $s=1$ reproduces almost the ground-truth labeling and is also close to the result of applying spectral clustering \cite{Shi2000} directly to $K_{\mc{E}}$. The SAF with $s=0$ enforces assignments with a more regular spatial structure. Both findings agree with observations made in preceding experiments; see e.g. \cref{fig:Seastar}.

\begin{figure}[tpb]
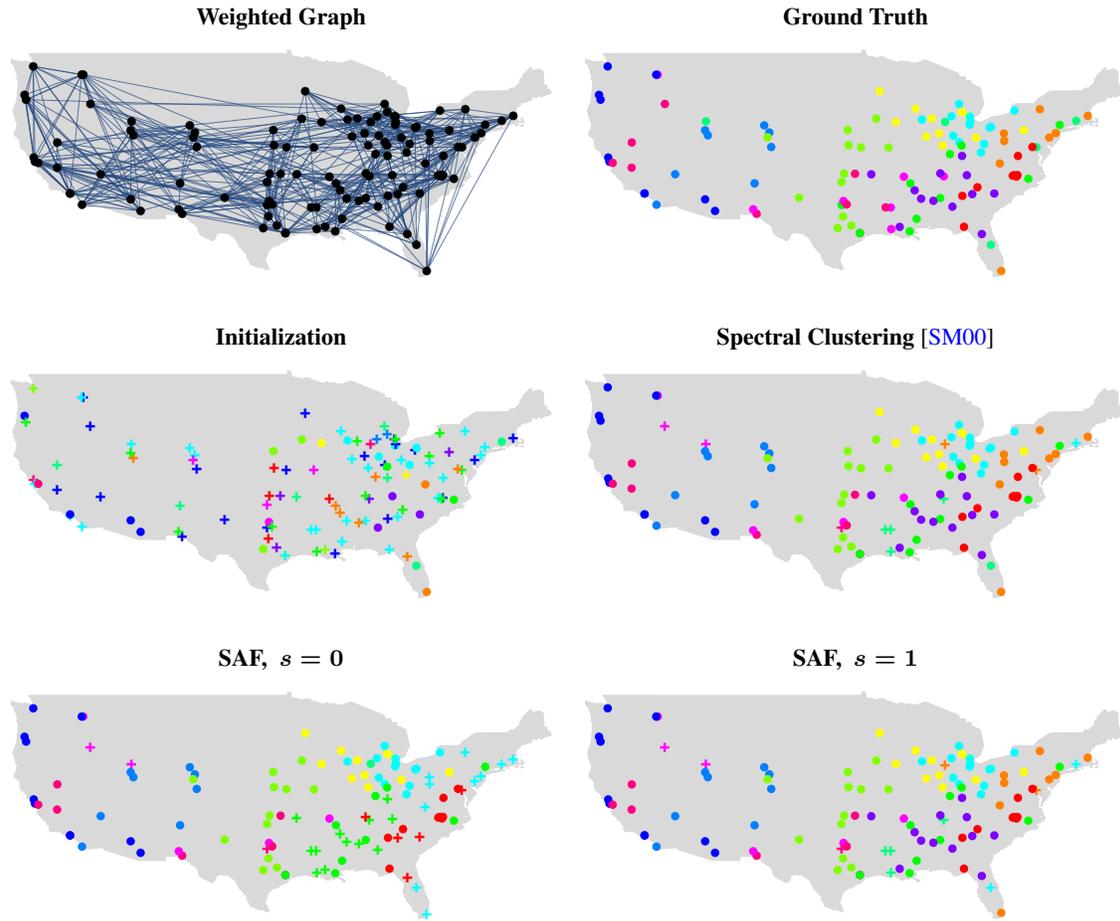

	\begin{center}
		\showExpGraph{0.48}
	\end{center}
		\caption{Weighted graph data of American football games between Division IA colleges during the regular season fall 2000 are clustered. Each node represents a team and edge weights indicate the number of games played between two teams. The colored nodes in `Ground Truth' show the subdivision of the teams into 12 conferences (clusters), that primarily play against each other in a first period. Graph partitioning with $c=12$ was performed using the SAF with $s=0$ and $s=1$, and with weights defined by \eqref{eq:graph-regularization-weights}. Markers $\PlusMarker$ indicate labels assigned to nodes that differ from ground truth. Starting from the initialization (2nd row, left panel) which is noisy, the SAF with $s=1$ returns almost the ground-truth labeling and is also close to the result of directly applying spectral clustering to $K_{\mc{E}}$.  The SAF with $s=0$ enforces label assignments that are spatially more regular, and with empty clusters orange and purple.
		}
		\label{fig:Football}
\end{figure}
%

\section{Conclusion}

We extended the assignment flow approach to supervised image labeling introduced by \cite{Astrom:2017ac} to unsupervised scenarios where no labels are available. The resulting self-assignment flow takes a pairwise affinity matrix as input data and maximizes the correlation (inner product) with a low-rank self-assignment matrix, corresponding to a factorization determined by the variables of the assignment flow. A single parameter $s \in [0,1]$ determines the self-assignment matrix as smooth geodesic interpolation of the self-affinity matrix ($s=0$) and the self-influence matrix ($s=1$), which enables to control the relative influence of spatial regularization and the preservation of feature-induced local image structure, respectively. A second parameter, the size $|\mc{N}|$ of local neighborhoods for geometric averaging of assignments, controls the scale of the resulting image partition, as in the supervised case, and the resulting number of clusters.

The compositional design of the approach, informally expressed as `regularization $\circ$ data likelihood' as opposed to `regularization $+$ data likelihood' as in traditional variational approaches, merely required to generalize the likelihood map (cf.~\eqref{eq:def-Lsi}) in order to extend the approach to the unsupervised case. In particular, numerical techniques developed by \cite{Zeilmann:2018aa} for integrating the assignment flow still apply. Learning patch dictionaries with a locally invariant patch distance function demonstrated exemplarily, together with a range of further numerical experiments, that our approach can flexibly cope with all common feature representations, including RKHS embeddings.

We characterized mathematically our approach from different relevant viewpoints, depending on the parameter $s$: As rank-constrained discrete optimal transport and as normalized spectral cuts that are spatially regularized in an unbiased way (rather than adding spatial coordinates as `features'). Additionally, we showed that the formation of prototypes automatically optimizes a basic class separability measure. Finally, from the viewpoint of combinatorial optimization, our approach successfully handles completely positive factorizations of self-assignments in large-scale scenarios, subject to spatial regularization.

Promising directions of further research include application-dependent extensions of the invariance group in order to learn compact patch dictionaries using the self-assignment flow in various scenarios. An open challenging problem concerns the extension of weight parameter estimation for application-specific adaptive regularization
\cite{Huhnerbein:2019ab} to the unsupervised self-assignment flow approach.


\appendix
\section{Scatter Matrices}\label{app:scatter-matrices}

Let $\mc{F}_{n} = \{f_{i}\in\mc{F} \colon i \in \mc{I}\}$, $n=|\mc{I}|$, denote given data. Consider a partition $\mc{I}=\dot\cup_{j\in\mc{J}}\mc{I}_{j}$ with $n_{j}=|\mc{I}_{j}|$ and $\sum_{j \in \mc{J}} n_{j}=n$.
We define the empirical quantities
\begin{subequations}\label{eq:def-m-k}
\begin{align}
P_{j} &= \frac{n_{j}}{n},\qquad j \in [c]
&&(\text{prior probabilities})
\\
m_{j}
&= \frac{1}{n_{j}} \sum_{i \in \mc{I}_{j}} f_{i},\qquad j \in [c]
&&(\text{class-conditional mean vectors})
\\
m &= \sum_{j \in [c]} P_{j} m_{j}
= \frac{1}{n} \sum_{i \in [n]} f_{i}
&&
(\text{mean vector})
\end{align}
\end{subequations}
and the \textit{scatter matrices} (empirical covariance matrices)
\begin{subequations}\label{eq:def-scatter-matrices}
\begin{align}
S_{t} &= \frac{1}{n} \sum_{i \in [n]} (f_{i}-m)(f_{i}-m)^{\T},
\\
S_{w} &= \sum_{j \in [c]} P_{j} \cdot \frac{1}{n_{j}} \sum_{i \in \mc{I}_{j}}(f_{i}-m_{j})(f_{i}-m_{j})^{\T}
= \frac{1}{n}\sum_{j \in [c]}\sum_{i \in \mc{I}_{j}}(f_{i}-m_{j})(f_{i}-m_{j})^{\T},
\\
S_{b} &= \sum_{j \in [c]} P_{j} (m_{j}-m)(m_{j}-m)^{\T}.
\end{align}
\end{subequations}
$S_{w}$ is called the \textit{within-class} scatter matrix, whereas $S_{b}$ is called the \textit{between-class} scatter matrix. $S_{t}$ is called the \textit{total} scatter matrix due to the decomposition \eqref{eq:St-decomposition}, that can be shown by an elementary calculation.

\vspace{1cm}
\centerline{\textbf{Acknowledgements}}
Financial support by the German Science Foundation (DFG), grant GRK 1653, is gratefully acknowledged. This work has also been stimulated by the Heidelberg Excellence Cluster STRUCTURES, funded by the DFG under Germanys Excellence Strategy EXC-2181/1 - 390900948.

\bibliographystyle{amsalpha}
\bibliography{TexInput/bibliography}
\clearpage
\end{document}